\documentclass[preprint,12pt]{elsarticle}
\biboptions{sort&compress}

\usepackage[margin=1in]{geometry}
\usepackage[x11names]{xcolor}
\usepackage{array}
\usepackage{tabularx}
\usepackage{booktabs}
\usepackage{tcolorbox}
\usepackage{hyperref}
\usepackage{amsfonts} 
\usepackage{multirow}
\usepackage{enumitem}
\setlistdepth{9}
\renewlist{itemize}{itemize}{9}
\setlist[itemize]{label=$\bullet$}
\usepackage{pdflscape}
\usepackage{float}
\usepackage{caption}
\usepackage[dvipsnames]{xcolor}
\usepackage{bm}
\definecolor{darkred}{rgb}{0.9,0.1,0.1}
\usepackage{subcaption}
\usepackage{amsthm}
\newtheorem{theorem}{Theorem}
\newtheorem{proposition}[theorem]{Proposition}
\theoremstyle{remark}
\newtheorem{remark}[theorem]{Remark}
\theoremstyle{definition}
\newtheorem{definition}[theorem]{Definition}
\newtheorem{assumption}[theorem]{Assumption}
\theoremstyle{plain}
\usepackage{mathtools}
\usepackage{algorithmic}

\usepackage[ruled,vlined]{algorithm2e}

\hypersetup{
 colorlinks,
 citecolor=black,
 filecolor=black,
 linkcolor=black,
 urlcolor=black
}
\usepackage{amssymb}
\usepackage{mathrsfs}
\usepackage{textcomp}
\usepackage{listings}
\usepackage{multicol}

\journal{arXiv}

\begin{document}

\begin{frontmatter}

\title{GRAFT-ATHENA: Self-Improving Agentic Teams for Autonomous Discovery and Evolutionary Numerical Algorithms}

\author[inst1]{Juan Diego Toscano}
\author[inst1]{Zhaojie Chai}
\author[inst1,label2]{George Em Karniadakis}
\affiliation[inst1]{organization={Division of Applied Mathematics, Brown University},%
   city={Providence},
   postcode={02912}, 
   state={RI},
   country={USA}}

\fntext[label2]{Corresponding author: george\_karniadakis@brown.edu}
\begin{abstract}
Language-model agents plan, execute, and evaluate scientific workflows. However, the problem--method relationships guiding their choices remain implicit, so plausible outputs can violate the assumptions of the scientific problem. We introduce Graph Reduction to Adaptive Factored Trees (GRAFT), which makes this problem-to-method map explicit as an expandable probabilistic structure of admissible problems, methods, and dependencies. Because this structure follows scientific semantics, knowledge from documentation, papers, and domain experts can be encoded, while factorization keeps the structure tractable. Semantic fingerprints of encoded problems measure their similarity, so expertise and experience guide related problems. GRAFT-guided workflows exceeded published expert baselines in physics-informed learning, recovered blood-rheology trends, predicted Apollo Command Module hypersonic loading within $1.8\%$ of experiment, and established machine-checked error bounds for regularized brain-flow reconstruction and a spectral physics-informed architecture. The representation retains evaluated outcomes and incorporates new methods, allowing scientific knowledge and experience to accumulate.
\end{abstract}


\end{frontmatter}

\section{Introduction}

Scientific methods are rarely developed for isolated problems. Scientists instead organize problems into families and design methods around the assumptions shared within each class. This organization gives scientific expertise a structure: problem characteristics constrain which methods are admissible, dependencies among method choices determine how they combine, and shared assumptions let methods transfer across classes. Agentic systems built on large language models (LLMs) are increasingly used to automate these selections by formulating plans, implementing and executing methods, and refining solutions~\cite{ghafarollahi2025sciagents,toscano2025athena,jiang2026agenticsciml,gottweis2025towards,aygun2025ai,lu2024ai,li2025codepde,liu2025pde,du2026autonumerics,loo2026agentic,lupoiu2025multi,chen2026emseek,kaiser2025large}. However, the models guiding these selections are generally trained to predict text and optimized to produce responses aligned with preferences supplied by human or model feedback~\cite{ouyang2022training,bai2022constitutional}. Therefore, sequences favored by the training objectives can still violate the assumptions or dependencies a scientific problem requires, and the problem-to-method map that their selections define remains largely implicit.

Without an explicit action space, method selection relies largely on repeated generation and ranking, a strategy most effective when candidates are easy to evaluate~\cite{romera2024mathematical,novikov2025alphaevolve,georgiev2025mathematical,aygun2025ai,jiang2025aide,real2023autonumerics}. Successful runs can also omit why a choice worked, limiting what experience can guide subsequent problems. Prompting, retrieval, and learned graphs can supply expert knowledge or expose relationships, but do not generally integrate scientific similarity, dependencies, and action probabilities into a measurable space~\cite{buehler2025agentic,stewart2026higher,stewart2026graphagents,chen2026emseek}. These limitations suggest that what is missing is a tractable scientific substrate in which problem--method relationships are explicit, dependencies are encoded, and expert knowledge and evaluated experience can guide new problems.

Scientific semantics provides this substrate: problems and methods are described through hierarchical characteristics and actions, with dependencies among them constraining subsequent choices. Therefore, we introduce Graph Reduction to Adaptive Factored Trees (GRAFT), which makes this organization explicit as problem and method graphs. Because the number of joint configurations grows exponentially with the number of coupled choice families, GRAFT decomposes the policy into local decision chains while retaining cross-chain dependencies as explicit rules.

The resulting dependency graph is an $I$-map of the policy, making the action distribution tractable and learnable without discarding the encoded couplings~\cite{pearl1988probabilistic,vermapearl1988causal}. To compare the problems and methods represented by these factored graphs, recursive projection turns the nodes activated by each problem or method into a unique semantic fingerprint, making scientific distance measurable. A solved-case repository then links these fingerprints to evaluated methods, outcomes, and rewards, creating an explicit problem-to-method map in which related cases inform future priors.

Since GRAFT follows scientific semantics, knowledge from solver documentation, methods papers, and domain expertise can be encoded into its graphs as problem characteristics, method actions, and dependency rules. With this map, evaluated methods and outcomes guide later problems, allowing GRAFT-guided agents to build on accumulated scientific practice and encoded expertise rather than search from scratch. As a result our workflows exceed published expert baselines in physics-informed learning while exposing relationships between structurally related problems and methods. When mathematical analysis produces new actions or formulations, the expanding graphs incorporate them and retain their outcomes, making each discovery available to later problems. The same loop spans physics-informed learning, numerical solvers, particle simulations, and formal mathematics, carrying workflows from formulation and selection through execution and evaluation and, for mathematical claims, to machine-checked verification.

\section{Results}

\subsection{GRAFT makes scientific problem and method structure explicit and tractable}

An agentic system selects a method $m_i\in\mathcal{M}$ to solve a scientific problem $p_i\in\mathcal{P}$. Since several methods may be admissible for a given problem, we model this selection through a stochastic policy $\pi:\mathcal{P}\to\Delta(\mathcal{M})$, which assigns probabilities to candidate methods (Methods). To make the scientific structure of these spaces explicit, we organize problem characteristics and method actions into families. Each problem or method is then represented by its selections within these families.

These selections are coupled rather than independent: several feature families may apply jointly, and one choice can restrict alternatives. For example, in Fig.~\ref{fig:graft_semantic_geometry}A, selecting a network with at least $10^5$ parameters excludes SSBroyden because of its computational cost. We encode these relations in problem and action graphs $\mathcal{G}_P$ and $\mathcal{G}_A$. A dense conditional policy must store probabilities for every configuration on which each decision depends, so its parameter footprint grows exponentially with interacting families (Methods). At this scale the conditional probabilities exceed what domain expertise can specify or solved cases can estimate, so the problem-to-method map generally remains implicit.

GRAFT makes this policy tractable by factorizing the problem and action graphs $\mathcal{G}_P$ and $\mathcal{G}_A$ into adaptive trees $\mathcal{T}_P$ and $\mathcal{T}_A$, while retaining cross-family dependencies as explicit rules $\mathcal{R}$. Independently resolvable feature families form separate decision chains, ordered so that constraining choices are resolved first; the rules then restrict subsequent choices to admissible actions (Fig.~\ref{fig:graft_semantic_geometry}A). The resulting policy has a learnable parameter footprint that grows linearly with the available actions and encoded dependencies rather than exponentially with coupled configurations (Methods). We prove that the factorized dependency graph is an $I$-map of the policy~\cite{pearl1988probabilistic}, so the factorization preserves the couplings encoded in the original graphs (Proposition~\ref{prop:imap}; proof in the Supplementary Information (SI)).

Since GRAFT follows scientific semantics, the system integrates knowledge from solver documentation, scientific literature, and domain expertise into explicit problem and method spaces spanning PIML, numerical PDEs, DPD simulations, and formal reasoning (Supplementary Table~\ref{tab:graft_space_inventory}). For instance, solver choices and compatibility constraints from the Nektar++~\cite{cantwell2015nektar} and Trixi.jl~\cite{ranocha2021adaptive} documentation become explicit actions and dependency rules in the numerical space (Supplementary Fig.~\ref{fig:supp_full_space_pde_piml}; construction details in the Supplementary Information). Factorization keeps the resulting policies compact despite the coupled choices represented across these families (Fig.~\ref{fig:graft_semantic_geometry}C). In PIML, a dense conditional table for a single decision chain requires $7.12\times10^{13}$ entries, whereas GRAFT stores 512 learnable categorical entries plus 17 symbolic rules (Supplementary Table~\ref{tab:policy_footprint_anatomy}).

\subsection{Semantic fingerprints make scientific similarity measurable}

To compare problems and methods through their encoded scientific characteristics, a recursive projection assigns every node of each factored tree a position in the unit square (Supplementary Video 1; Algorithm~\ref{alg:partition_layout} in the SI). The resulting layout is injective, so distinct tree nodes retain distinct positions (Proposition~\ref{prop:phi_injective}). The nodes activated by a particular problem or method then form its semantic fingerprint; the complete projected spaces are shown in Supplementary Figs.~\ref{fig:supp_full_space_pde_piml} and \ref{fig:supp_full_space_formal_dpd}. These fingerprints localize both shared and distinct scientific characteristics. Rayleigh--B\'enard and Rayleigh--Taylor, for example, have lexically similar names, but the evaluated cases describe incompressible Boussinesq convection and a compressible Euler instability, respectively; their problem fingerprints expose the corresponding differences in equation structure, regularity, and forcing (Fig.~\ref{fig:graft_semantic_geometry}D). These differences propagate to their method fingerprints, which overlap only in parts of the mesh specification and otherwise contain nearly disjoint solver selections. In the name-only retrieval comparison over the solved-case source codes, embedding-based RAG ranks a lexical name-twin first in both directions, while BM25 does so for Rayleigh--B\'enard; fingerprint retrieval instead ranks solver-compatible cases on top in both directions (Supplementary Fig.~\ref{fig:retrieval_baselines}; a separate benchmark using complete formalized requests is reported in Table~\ref{tab:topk_benchmark}).

\begin{figure}[H]
\centering
\includegraphics[width=0.9\textwidth]{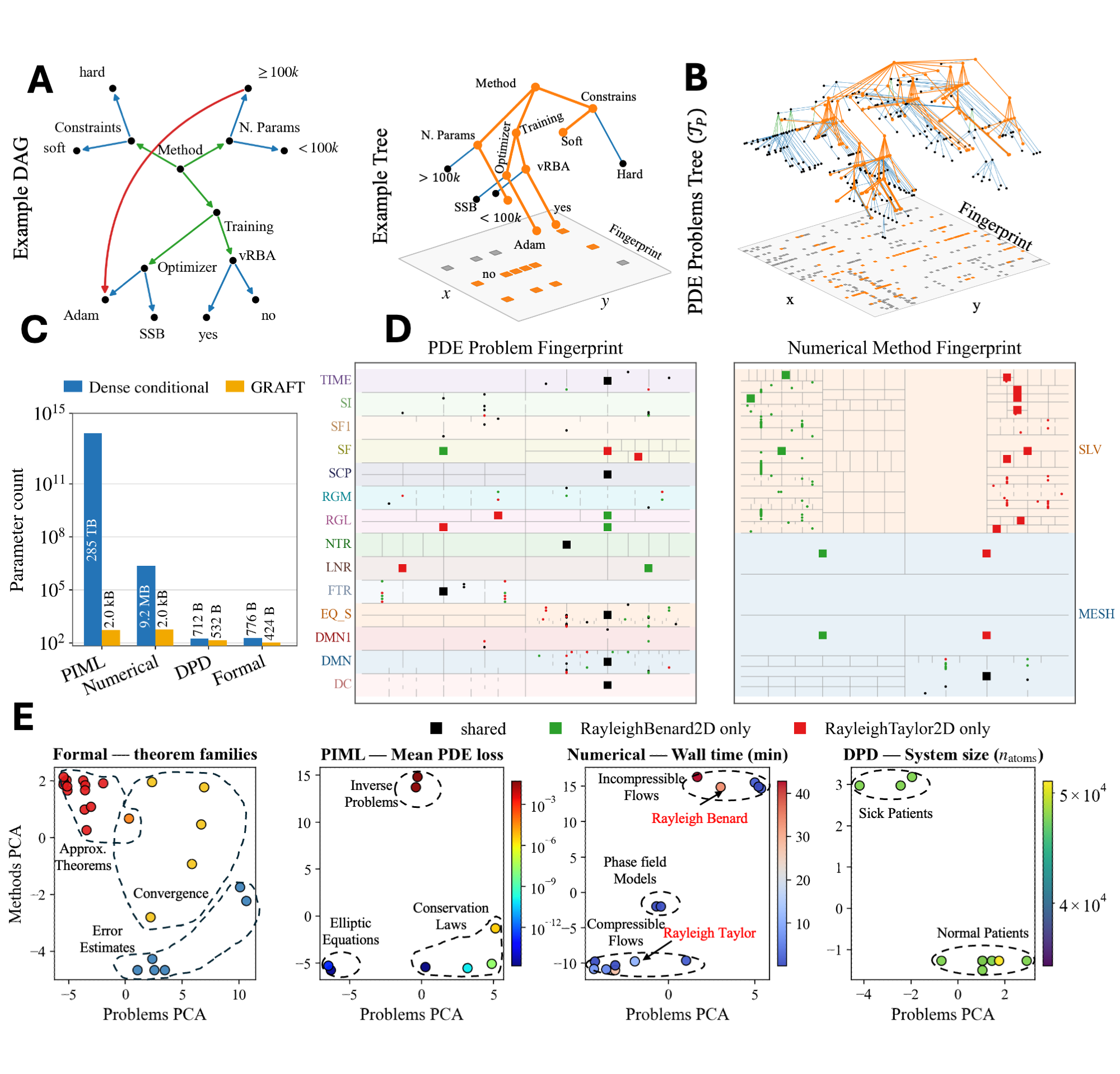}
\linespread{1}\selectfont
\caption{
\textbf{GRAFT converts coupled scientific graphs into compact, navigable semantic spaces.}
(\textbf{A}) A method DAG is reduced to an adaptive factored tree and projected injectively into two dimensions. Locally resolvable choices are separated; explicit cross-chain rules preserve dependencies. In the illustrated two-optimizer choice, the red rule selects Adam on the high-parameter branch, excluding SSBroyden. Orange cells mark the highlighted selections; grey cells mark the remaining projected nodes.
(\textbf{B}) The complete PDE problem tree and fingerprint projection extend this construction to the complete PDE problem vocabulary.
(\textbf{C}) Dense conditional-policy and GRAFT parameter footprints across PIML, Numerical, DPD, and Formal method spaces. Labels report float32 memory for learnable categorical entries. Dense conditioning grows rapidly with coupling; GRAFT remains compact.
(\textbf{D}) Problem and numerical-method fingerprints for Rayleigh--B\'enard 2D and Rayleigh--Taylor 2D. Black cells are shared; green and red are case-specific. Similar names conceal differences in equation structure, forcing, and regularity that produce largely distinct solver selections.
(\textbf{E}) Leading-PCA projections of evaluated problem--method pairs, colored by theorem family, mean PDE loss, wall time, or system size. Neighborhoods reflect scientific structure and link nearby problems to evaluated outcomes. Retrieval uses full-fingerprint similarity, not distance in this projection.
}
\label{fig:graft_semantic_geometry}
\end{figure}

Since fingerprints encode scientific characteristics, their overlap makes similarity between problems or between methods measurable. We quantify this overlap using Jaccard similarity $J_K$, the number of active cells shared by two fingerprints divided by the number present in either, or intersection over union. The complementary quantity, $1-J_K$, expresses their separation as a distance and defines a metric on nonempty fingerprints (SI Proposition~\ref{prop:fingerprint_metric}). The scientific organization represented by these fingerprints is visible in the two-dimensional PCA landscape of the evaluated problem--method pairs (Fig.~\ref{fig:graft_semantic_geometry}E): formal results separate by theorem class, PIML problems by equation class and attainable loss, numerical cases by physical regime, and DPD cases by system size. The resulting neighborhoods expose similarities and differences grounded in scientific characteristics that names alone do not convey.

Because each evaluated problem is paired with an accepted method, we can test whether similarity between problems retrieves cases with relevant method choices. We held out each of 13 solved numerical cases in turn and compared the methods retrieved from the remaining 12 with its accepted method. GRAFT ranked cases using $J_K$ in the full problem-fingerprint space, while embedding-based RAG and BM25 lexical search queried the stored solution codes using the complete formalized request. Top-1 overlap measures how closely the first retrieved method matches the accepted method, while Coverage@3 measures the fraction of its encoded method nodes represented collectively by the first three cases. NDCG@3 evaluates their ordering by method overlap, with 1 corresponding to an oracle ranking that uses the known accepted method to place the closest matches first. GRAFT exceeded both baselines on all three measures and approached the oracle, reaching 94.0\% coverage compared with 95.5\% for the oracle and an NDCG@3 of 0.916 compared with 1.000 (Table~\ref{tab:topk_benchmark}; full analysis in the SI). Thus, the encoded scientific knowledge and retrieved experience provide the basis for agentic workflows in which dependencies constrain method choices and evaluated outcomes inform their selection and assessment (Fig.~\ref{fig:r1_system_overview}).

\begin{table}[H]
\centering
\caption{
\textbf{Scientific fingerprint retrieval approaches the methodological coverage of an oracle ranking.}
The benchmark starts with 13 solved Numerical problems for which the accepted method is already known. In each test, one solved problem is removed from the repository, while its accepted method is kept separately as the answer used for evaluation. Its complete formalization-stage request is then submitted as a new query. The retrieval system ranks the other 12 solved cases, and their methods are compared with the known method of the held-out problem. This procedure is repeated once for each of the 13 problems, and the table reports the average results. Top-1 overlap asks how closely the first retrieved method matches the held-out method. Coverage@3 asks what fraction of the held-out method's choices appears among the first three retrieved cases. NDCG@3 asks whether the cases with the strongest method overlap are placed first. All retrieval arms are evaluated by overlap between accepted methods encoded in the same action taxonomy. These scores do not measure downstream solver accuracy or attempts to acceptance. The oracle ranks cases using the known held-out method and gives the repository-dependent ceiling, whereas random ranking gives the expected floor. GRAFT exceeds embedding RAG and BM25 on all three measures and reaches 0.940 coverage against an oracle ceiling of 0.955, showing that the remaining solved cases contain nearly all of the method information needed for the held-out problems and that the fingerprints retrieve it.}
\label{tab:topk_benchmark}
\begin{tabular}{@{}lccc@{}}
\toprule
Arm & Top-1 overlap & Coverage@3 & NDCG@3 \\
\midrule
Oracle (ceiling)      & 0.852 & 0.955 & 1.000 \\
GRAFT fingerprints    & 0.742 & 0.940 & 0.916 \\
Embedding RAG         & 0.638 & 0.821 & 0.744 \\
BM25 lexical search   & 0.381 & 0.593 & 0.500 \\
Random (floor)        & 0.298 & 0.696 & 0.381 \\
\bottomrule
\end{tabular}
\end{table}

\subsection{Embedding scientific knowledge in self-improving agentic workflows}
\label{sec:R1}

\begin{figure}[H]
\centering
\includegraphics[width=0.92\textwidth]{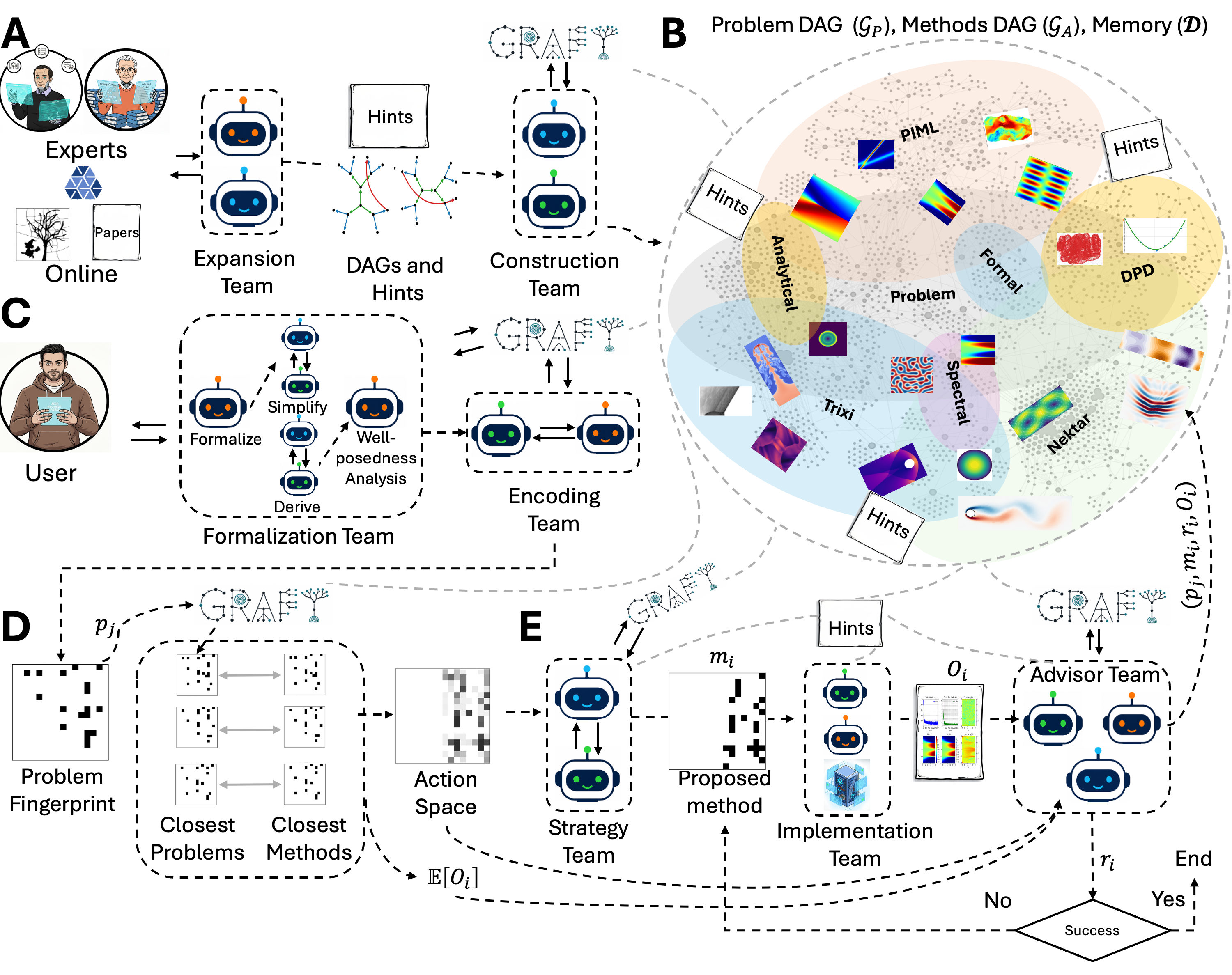}
\par\medskip
\caption{
\textbf{GRAFT integrates scientific knowledge and evaluated experience into agentic workflows.}
(\textbf{A}) The Expansion team extracts problem characteristics, method actions, dependencies, and hints from documentation, papers, and domain expertise. The Construction team incorporates additions that preserve valid GRAFT factorization.
(\textbf{B}) The problem graph $\mathcal{G}_P$, action graph $\mathcal{G}_A$, and solved-case memory $\mathcal{D}=\{(p_i,m_i,O_i,r_i)\}$ persist across problems. New nodes extend the vocabulary without invalidating recorded problems, methods, outcomes, or rewards.
(\textbf{C}) Formalization tests simplifications and well-posedness; Encoding maps the selected formulation onto $\mathcal{T}_P$ to obtain its fingerprint.
(\textbf{D}) Neighboring methods define a prior weighted by fingerprint similarity and reward and blended with a uniform component. Recorded observables provide performance targets.
(\textbf{E}) This prior guides task-specific proposal, implementation, execution, and evaluation. Every attempted method and outcome returns to $\mathcal{D}$, allowing accumulated experience to guide subsequent structurally related problems.}
\label{fig:r1_system_overview}
\end{figure}

Building on this representation, we organize an agentic workflow that uses encoded scientific knowledge and evaluated experience to guide problem formulation, method selection, execution, and assessment (Fig.~\ref{fig:r1_system_overview}). The same representation supports workflows for PIML, numerical PDEs, DPD simulations, and formal reasoning, with stages and evaluation procedures adapted to each scientific family (Supplementary Table~\ref{tab:graft_athena_family_stages}). These workflows maintain two persistent components: the semantic spaces that encode scientific knowledge and the repository that links evaluated problems to their methods and outcomes. An Expansion team constructs and extends the spaces by extracting problem characteristics, method actions, and dependencies from solver documentation, methods papers, and domain expertise. A Construction team incorporates additions that preserve a valid factored representation, while a domain expert verifies their scientific meaning (Fig.~\ref{fig:r1_system_overview}A). The repository is initialized with expert-curated examples from scientific software, published studies, and established mathematics (Supplementary Tables~\ref{tab:memory_seeds_piml}--\ref{tab:memory_seeds_formal_b}), and each evaluated problem--method pair adds its observables and reward to this shared experience (Fig.~\ref{fig:r1_system_overview}B). The semantic spaces can also expand through additive extensions that preserve similarities between unchanged encodings, allowing new scientific knowledge to enter while earlier cases remain available (Supplementary Fig.~\ref{fig:jaccard_invariance}).

For each new problem, the Formalization team converts the free-form statement into a structured request, tests candidate simplifications, and audits well-posedness. The Encoding team then maps the selected formulation onto the problem tree to obtain its fingerprint (Fig.~\ref{fig:r1_system_overview}C; Methods). This fingerprint identifies the closest solved cases, whose methods contribute to a prior weighted by similarity and reward and blended with a uniform prior. The resulting distribution guides action selection, while neighboring observables supply performance targets (Fig.~\ref{fig:r1_system_overview}D). A proposer considers this prior together with the current problem's characteristics and expert hints; encoded rules restrict subsequent choices, and a critic audits the proposals (Methods). A support check flags weakly supported cases and blocks their execution in autonomous mode, while unrepresentable cases require expansion and review before execution. Each run is assessed using a family-specific informed reward that combines these performance targets with deterministic execution or verification signals (Methods). Each evaluated attempt is retained in the repository, so the workflow builds on both its initial scientific knowledge and experience gained during execution.

\subsection{From encoded expertise to expert-level scientific results}
\label{sec:semantic_memory_results}

With expert knowledge encoded in GRAFT's problem and method spaces, the guided PIML workflow attained lower relative errors than the reported human-designed methods and ATHENA without GRAFT across four canonical benchmarks: viscous Burgers, KdV, Helmholtz, and Poisson. Residual losses approached machine precision, while comparison with reference solutions established the corresponding field accuracy (Fig.~\ref{fig:semantic_memory_learning}B; Supplementary Fig.~\ref{fig:priors_and_results}A; Table~\ref{tab:sota_comparison_graft}).

For viscous Burgers, we extended the comparison to two commercial agentic systems, Opus 4.8 and Fable 5, executed through Claude Code using the same problem request (SI). Both constructed competent PINN solvers and performed meaningful validation checks, but their accepted solutions remained substantially less accurate than specialized methods. Both used Adam followed by L-BFGS, whereas ATHENA and GRAFT+ATHENA used repeated SSBroyden cycles with resampling to refine their solutions beyond individual optimizer termination. GRAFT additionally recovered a Reynolds-number continuation strategy from the displayed closest neighbor and adapted it to a three-stage schedule (Fig.~\ref{fig:semantic_memory_learning}A). Error against the target-viscosity reference remained high while the intermediate equations were solved, then dropped sharply when the schedule reached the requested viscosity, yielding the lowest final error among the compared systems (Fig.~\ref{fig:semantic_memory_learning}B). Across the four PIML benchmarks, GRAFT+ATHENA reached acceptance in one to three attempts while retaining evaluated experience between problems (Fig.~\ref{fig:loop_efficiency}). These results suggest that encoded expertise and evaluated experience inform both the methods used to refine a solution and the performance standards used to assess it.

\begin{figure}[H]
\centering
\includegraphics[width=0.9\textwidth,keepaspectratio]{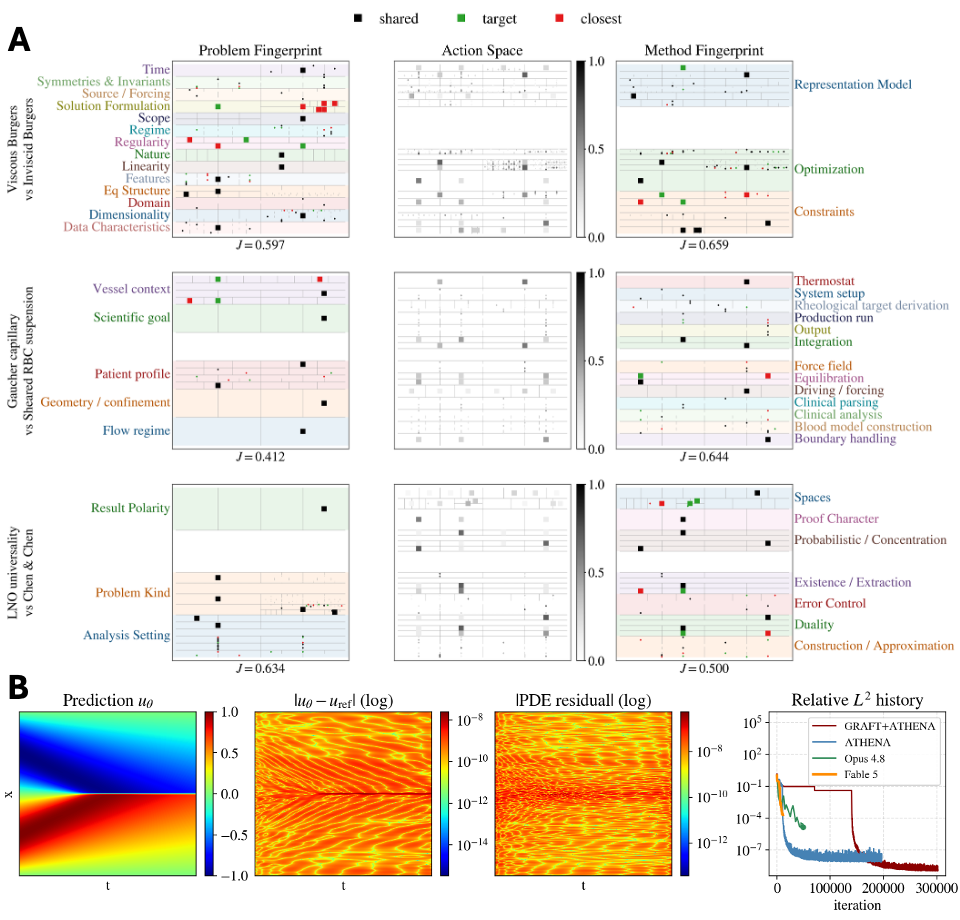}
\linespread{1}\selectfont
\par\medskip
\caption{
\textbf{Scientific neighborhoods transfer successful choices between related problems.}
(\textbf{A}) Three examples follow transfer from a retrieved neighborhood to the selected method. Each row shows target and nearest-neighbor problem fingerprints, the similarity- and reward-weighted action prior, and target and neighbor method fingerprints. Only the nearest neighbor is displayed; all retrieved cases contribute to the prior. Black entries are shared, green are target-specific, and red are neighbor-specific; action-space darkness denotes probability, and $J$ denotes fingerprint similarity. Viscous Burgers inherits Reynolds-number continuation from inviscid Burgers (top). For the Gaucher capillary case versus a sheared RBC suspension (middle), the broader DPD neighborhood supplies validated particle-simulation choices adapted to the target. For LNO universality versus the Chen--Chen theorem (bottom), related formal results supply proof ingredients for the new theorem.
(\textbf{B}) GRAFT+ATHENA prediction, absolute error, PDE residual, and relative $L^2$ history for viscous Burgers, compared with ATHENA, Opus 4.8, and Fable 5. Error remains high through intermediate continuation viscosities, then falls sharply at the requested viscosity. The inherited strategy remains effective in the target regime, yielding the lowest final error among the compared systems.
}
\label{fig:semantic_memory_learning}
\end{figure}

\begin{figure}[H]
\centering
\includegraphics[width=1\textwidth]{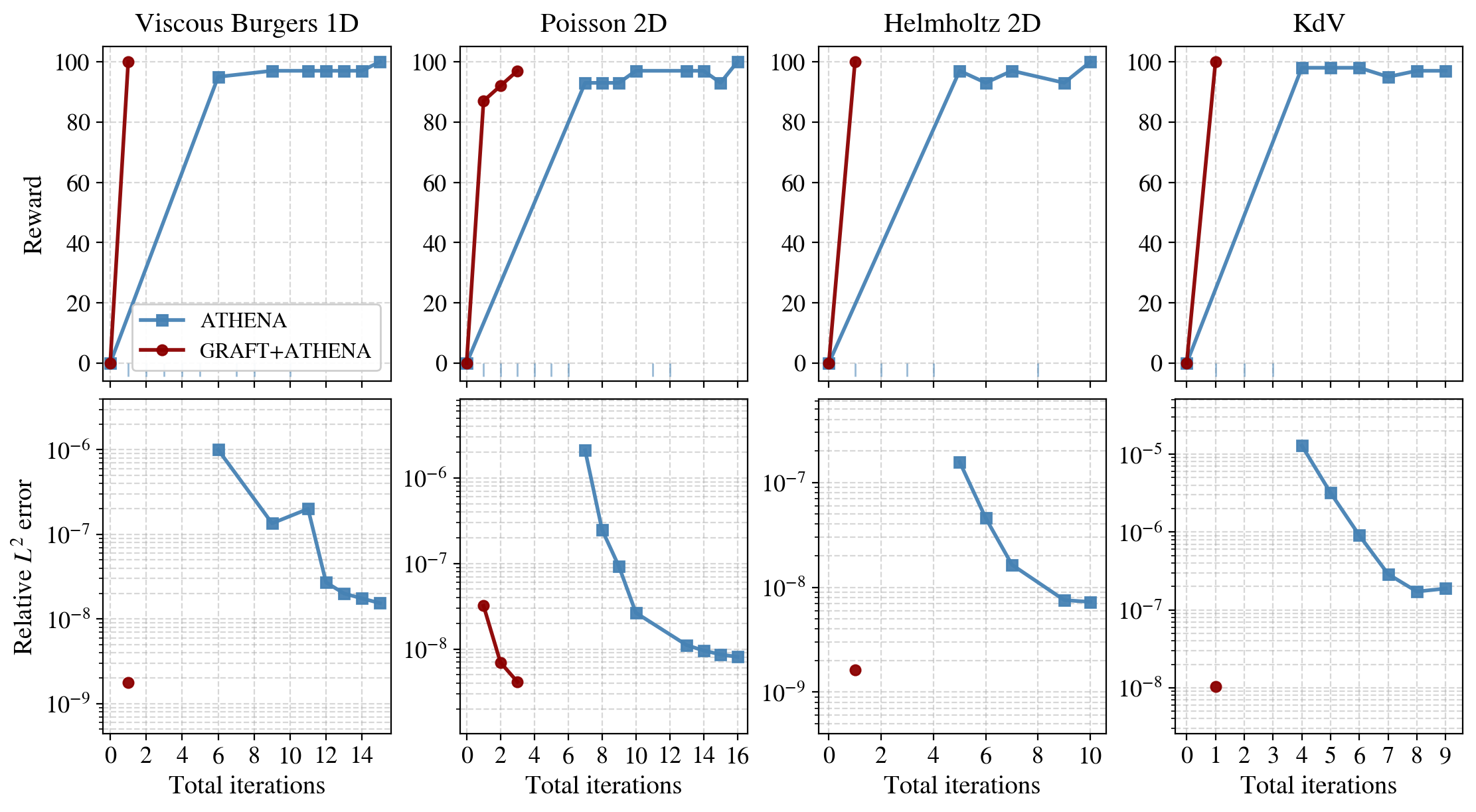}
\caption{
\textbf{Persistent solved-case memory accelerates refinement across related PIML problems.}
Reward (top) and relative $L^2$ error (bottom) across outer-loop attempts for viscous Burgers, Poisson, Helmholtz, and KdV. GRAFT+ATHENA (red) uses the methods and outcomes of fingerprint-similar solved cases to initialize the action prior; ATHENA without GRAFT (blue) retains the same within-problem refinement loop but resets before each new problem. Short ticks mark attempts that did not return a scored run. Reward is each system's internal acceptance signal and is therefore interpreted within a system, whereas relative $L^2$ error is evaluated against the same external reference for both. GRAFT+ATHENA reaches acceptance within one to three attempts and attains a lower final error on all four problems, while ATHENA requires a longer sequence of execution and repair. The Poisson--Helmholtz sequence shows how this advantage accumulates: Poisson enters without a closely matched solved case and requires three GRAFT+ATHENA attempts, but once stored it becomes the highest-ranked neighbor for Helmholtz, which is accepted on the first attempt. ATHENA's memory resets between problems, so its experience on Poisson does not improve the starting point for Helmholtz; it requires 16 attempts for Poisson and another 10 for Helmholtz. Persistent memory therefore allows a solved problem to reduce the search cost of the next structurally related problem.}
\label{fig:loop_efficiency}
\end{figure}

\begin{table}[!htbp]
\centering
\caption{
\textbf{GRAFT+ATHENA reaches the lowest solution error across four canonical PIML benchmarks.}
Comparison with recent expert-designed methods (Human) and agentic frameworks on viscous Burgers, KdV, Helmholtz, and Poisson, four standard benchmarks in physics-informed machine learning. For each problem, the table places the best result reported by each cited method alongside the result from the final accepted GRAFT+ATHENA solution. For GRAFT+ATHENA, residual mean-squared error (MSE) is the PDE-residual loss minimized during training. Values near machine precision therefore show that this training objective has been reduced to the numerical limit of the computation. Relative $L^2$ error (Eq.~\eqref{eq:RL2}) measures agreement between the predicted field and the benchmark solution. The cited studies generally report this error on dense evaluation grids, making it the standard field-accuracy measure for these benchmarks in the PIML community. For GRAFT+ATHENA, relative $L^2$ error is computed over the full domain on a fixed grid of approximately $2\times10^5$ points, using the exact Cole--Hopf solution evaluated numerically for viscous Burgers and analytical solutions for KdV, Helmholtz, and Poisson. Dashes denote quantities not reported by the cited source.}
\label{tab:sota_comparison_graft}
{\small\linespread{1}\selectfont
\begin{tabular}{@{}lllcc@{}}
\toprule
\textbf{Problem} & \textbf{Reference} & \textbf{Benchmark Type} & \textbf{Rel.\ $L^2$ Error} & \textbf{MSE}\\
\midrule
\multirow{10}{*}{Burgers ($\nu=\frac{1}{100\pi}$)}
 & \cite{wang2025gradient} & Human & $4.03 \times 10^{-5}$ & -- \\
 & \cite{wuwu2025pinnsagent} & Agents & -- & $6.51 \times 10^{-5}$ \\
 & \cite{he2025lang} & Agents & -- & $6.48 \times 10^{-5}$ \\
 & \cite{urban2024unveiling} & Human & $2.90 \times 10^{-6}$ & -- \\
 & \cite{kiyani2025optimizing} & Human & $1.62 \times 10^{-8}$ & -- \\
 & \cite{toscano2026variational} & Human & $8.25 \times 10^{-9}$ & -- \\
 & Opus 4.8 & Agents & $1.50 \times 10^{-5}$ & -- \\
 & Fable 5 & Agents & $2.03 \times 10^{-4}$ & -- \\
 & \cite{toscano2025athena}  & Agents & $3.79 \times 10^{-8}$ & $4.33 \times 10^{-14}$ \\
 & Ours & Agents & $\mathbf{1.78 \times 10^{-9}}$ & $\mathbf{5.21 \times 10^{-16}}$ \\
\midrule
\multirow{3}{*}{KdV}
 & \cite{urban2024unveiling} & Human & $6.00 \times 10^{-6}$ & -- \\
 & \cite{toscano2025athena}  & Agents & $1.54 \times 10^{-7}$ & $1.22 \times 10^{-13}$ \\
 & Ours & Agents & $\mathbf{1.04 \times 10^{-8}}$ & $\mathbf{1.24 \times 10^{-15}}$ \\
\midrule
\multirow{5}{*}{Helmholtz}
 & \cite{urban2024unveiling} & Human & $3.60 \times 10^{-7}$ & -- \\
 & \cite{chen2025self} & Human & $4.86 \times 10^{-5}$ & -- \\
 & \cite{jnini2026curvature} & Human & $2.0 \times 10^{-9}$ & -- \\
 & \cite{toscano2025athena}  & Agents & $6.23 \times 10^{-8}$ & $7.25 \times 10^{-13}$ \\
 & Ours & Agents & $\mathbf{1.60 \times 10^{-9}}$ & $\mathbf{3.13 \times 10^{-15}}$ \\
\midrule
\multirow{4}{*}{Poisson}
 & \cite{wang2025gradient} & Human & $2.99 \times 10^{-7}$ & -- \\
 & \cite{muller2023achieving} & Human & $1.0 \times 10^{-7}$ & -- \\
 & \cite{toscano2025athena}  & Agents & $8.46 \times 10^{-9}$ & $\mathbf{3.95 \times 10^{-13}}$ \\
 & Ours & Agents & $\mathbf{3.13 \times 10^{-9}}$ & $7.41 \times 10^{-13}$ \\
\bottomrule
\end{tabular}
}
\end{table}

In particle biophysics, the GRAFT-guided DPD workflow adapted established particle-simulation methods to different blood-flow settings. For a Gaucher-disease capillary target, the retrieved DPD neighborhood informed the action prior despite differences in vessel context and patient profile between the target and its displayed nearest neighbor, a sheared RBC suspension (Fig.~\ref{fig:semantic_memory_learning}A). Applying this knowledge also requires preserving the interactions that define the selected biophysical model. In a separate controlled hematocrit study, selecting an established multiscale RBC formulation~\cite{pivkin2008accurate,fedosov2010multiscale} constrained the membrane interactions needed to assemble the particle simulation (SI). With membrane parameters and wall stress held fixed, the assembled model recovered the expected increase in relative viscosity with hematocrit without tuning parameters to obtain this ordering, with variability assessed across independent seeds (Supplementary Fig.~\ref{fig:priors_and_results}B)~\cite{chai2025silico,chai2026multiscale}.

Finally, in formal mathematics, the GRAFT-guided Formal workflow established universal-approximation results for two widely used architectures, cKANs~\cite{ss2024chebyshev} and LNOs~\cite{cao2024laplace}. For cKANs, the analysis identified a standard $\tanh$ MLP within the Chebyshev-edge formulation, allowing the workflow to establish universality by applying the retrieved MLP theorem~\cite{Leshno1993} (Theorem~\ref{thm:ckan_universality}). For LNOs, the workflow instead combined proof ingredients from related operator-approximation results in the retrieved neighborhood to construct a new universality proof (Fig.~\ref{fig:semantic_memory_learning}A). The resulting theorem was machine-checked in Lean 4~\cite{moura2021lean} (Theorem~\ref{thm:lno_universality}). The user requests, theorem statements, and proof sketches are reproduced in the SI.

\subsection{Encoded dependencies keep agentic solutions faithful to the requested physics}
\label{sec:apollo_main}

We compared the GRAFT-guided Numerical workflow with Opus 4.8 and Fable 5 on two matched axisymmetric Mach-10 flow tasks over the Apollo Command Module, assessing both numerical accuracy and fidelity to the requested physics. These audits examine failure modes in the reported trajectories rather than estimate their frequency across runs. All systems received the same case-specific requests and unrestricted access to computational tools, scientific software, internet search, and subagents (comparison protocol in the SI). Both requests fixed the 1968 Arnold Engineering Development Center (AEDC) forebody geometry~\cite{griffith1968postflight}, but only one explicitly required the downstream wake to be represented (Fig.~\ref{fig:apollo_mach10}A; Supplementary Table~\ref{tab:apollo_setup}; exact requests in the SI).

In both cases, the blunt forebody generates a detached bow shock and a subsonic stagnation pocket, so omitting the wake requirement does not remove the need to preserve the geometry and choose a domain consistent with the resulting flow.

When the wake was explicitly required, Fable recovered a physically consistent flow using an at-most-second-order finite-volume scheme with minmod limiting on a uniform Cartesian immersed-boundary mesh of approximately $440{,}000$ cells~\cite{leveque2002finite,toro2013riemann}. On the other hand, the GRAFT-guided workflow used retrieved experience to select a discontinuous Galerkin spectral-element discretization~\cite{kopriva2009implementing,hesthaven2008nodal} with entropy-conservative fluxes~\cite{ranocha2018generalised}, using degree-four polynomials for fifth-order accuracy in smooth regions (Fig.~\ref{fig:apollo_mach10}B). Its adaptive mesh tracked moving features in the wake using $3{,}198$ elements with $79{,}950$ solution nodes (Supplementary Table~\ref{tab:apollo_setup})~\cite{berger1989local}. Since the available subcell-limiting implementation did not support runtime refinement, the encoded compatibility rules led the workflow to pair element-level Hennemann--Gassner shock capturing~\cite{hennemann2021provably} with Zhang--Shu positivity limiting~\cite{zhang2010positivity}. Domain and mesh construction followed Billig shock containment, published production-code margins, and mesh-quality requirements~\cite{anderson1989hypersonic,mazaheri2010laura,dyakonov2012hypersonic,mitcheltree1995wake} (implementation details in the SI), allowing the simulation to resolve the detached bow shock, subsonic stagnation pocket, and evolving wake (Fig.~\ref{fig:apollo_mach10}C,E).

When the wake was not explicitly required, the illustrated commercial runs introduced inconsistencies in the prescribed geometry or boundary treatment (Fig.~\ref{fig:apollo_ablation}A). Opus extended the aft truncation into a sting reaching the outflow, changing the prescribed capsule geometry. Fable imposed freestream values at a lateral boundary that its plan assumed would remain outside the shock, although the saved solution showed shock-disturbed flow adjacent to that boundary. To test whether the failure was isolated, we repeated the identical request in four additional fresh Fable sessions. Only $2/5$ runs completed without a protocol failure (full audits in the SI; Supplementary Table~\ref{tab:apollo_replicates}). The GRAFT-guided workflow instead maintained the prescribed geometry and shock containment as requirements during domain and mesh construction, using a fixed, graded body-fitted mesh with a provably positive subcell limiter (Supplementary Table~\ref{tab:apollo_setup}).

\begin{figure}[H]
\centering
\includegraphics[width=0.95\textwidth]{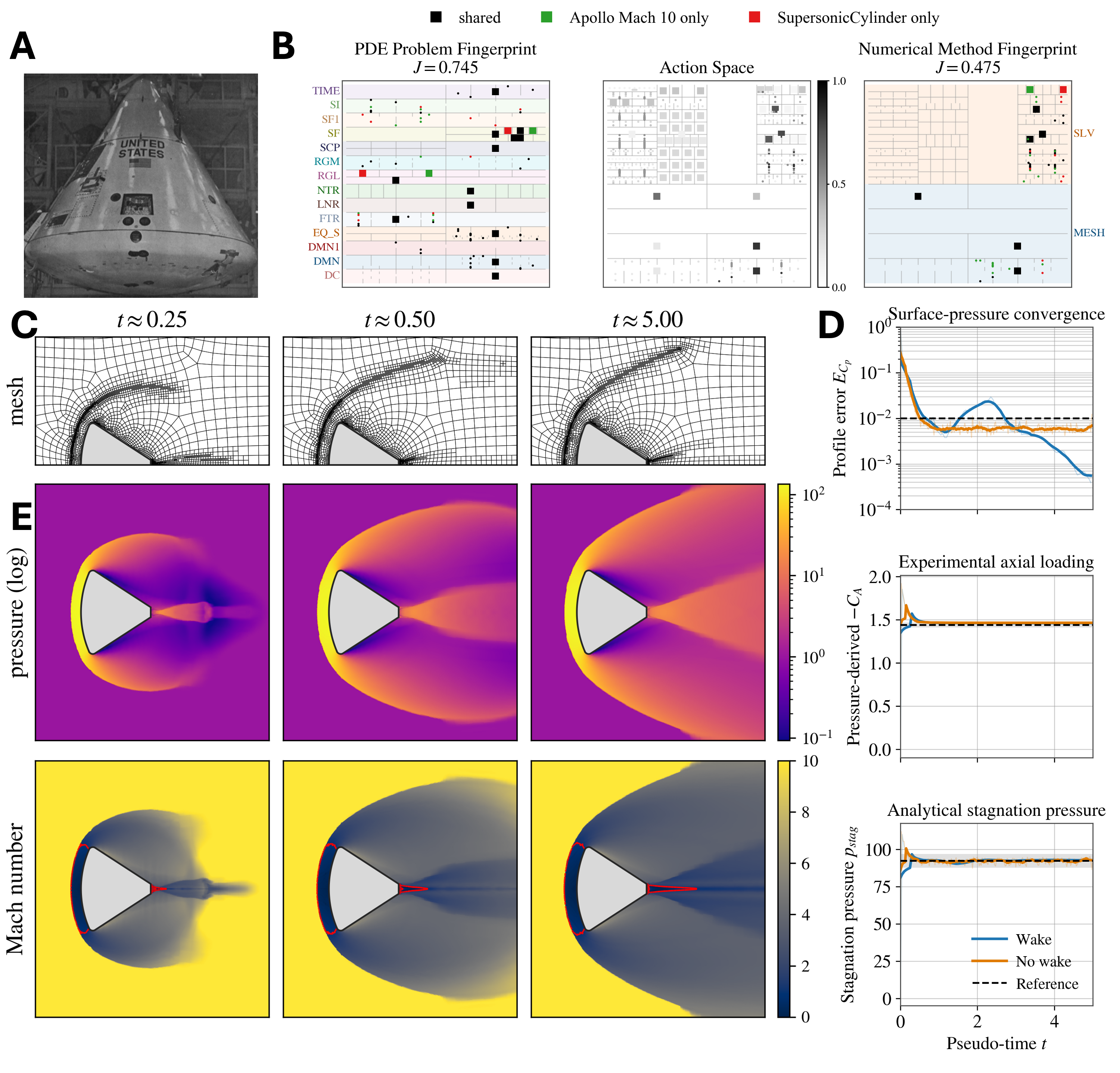}
\par\medskip
\linespread{1}\selectfont
\caption{
\textbf{Autonomous reconstruction of Mach-10 flow over the Apollo Command Module.}
(\textbf{A}) Apollo photograph from the 1968 AEDC report. Matched Mach-10 requests require the downstream wake (wake) or omit that requirement (no wake).
(\textbf{B}) Wake-case problem fingerprint, retrieved action prior, and selected method, with supersonic cylinder as the nearest solved case. Retrieved experience favors high-order discontinuous Galerkin spectral elements; encoded dependencies determine compatible shock treatment and meshing.
(\textbf{C}) Wake-case meshes at three times. The GRAFT-guided workflow generates the initial mesh and adaptively refines it during the simulation.
(\textbf{D}) Surface-pressure convergence (top), pressure-derived axial loading against AEDC (middle), and stagnation pressure against Rayleigh--Pitot (bottom), for wake (blue) and no-wake (orange) solutions. Both solutions stabilize below $1\%$ surface-pressure change and agree with loading and compression references within $1.8\%$. The grey stagnation-pressure band marks $\pm5\%$ around the Rayleigh--Pitot reference.
(\textbf{E}) Pressure and Mach-number fields at three times show wake-case evolution. Red contours mark Mach 1.
}
\label{fig:apollo_mach10}
\end{figure}

All illustrated solutions converged and produced plausible density fields (Fig.~\ref{fig:apollo_ablation}B). Both GRAFT configurations stabilized their surface-pressure profiles below $1\%$ and matched the AEDC axial-force and Rayleigh--Pitot stagnation-pressure references within $1.8\%$ (Fig.~\ref{fig:apollo_mach10}D). The commercial solutions also matched common scalar references within $4.4\%$, and their agents accepted the results despite the altered Opus geometry and unphysical Fable wall layer.

\begin{figure}[H]
\centering
\includegraphics[width=\textwidth]{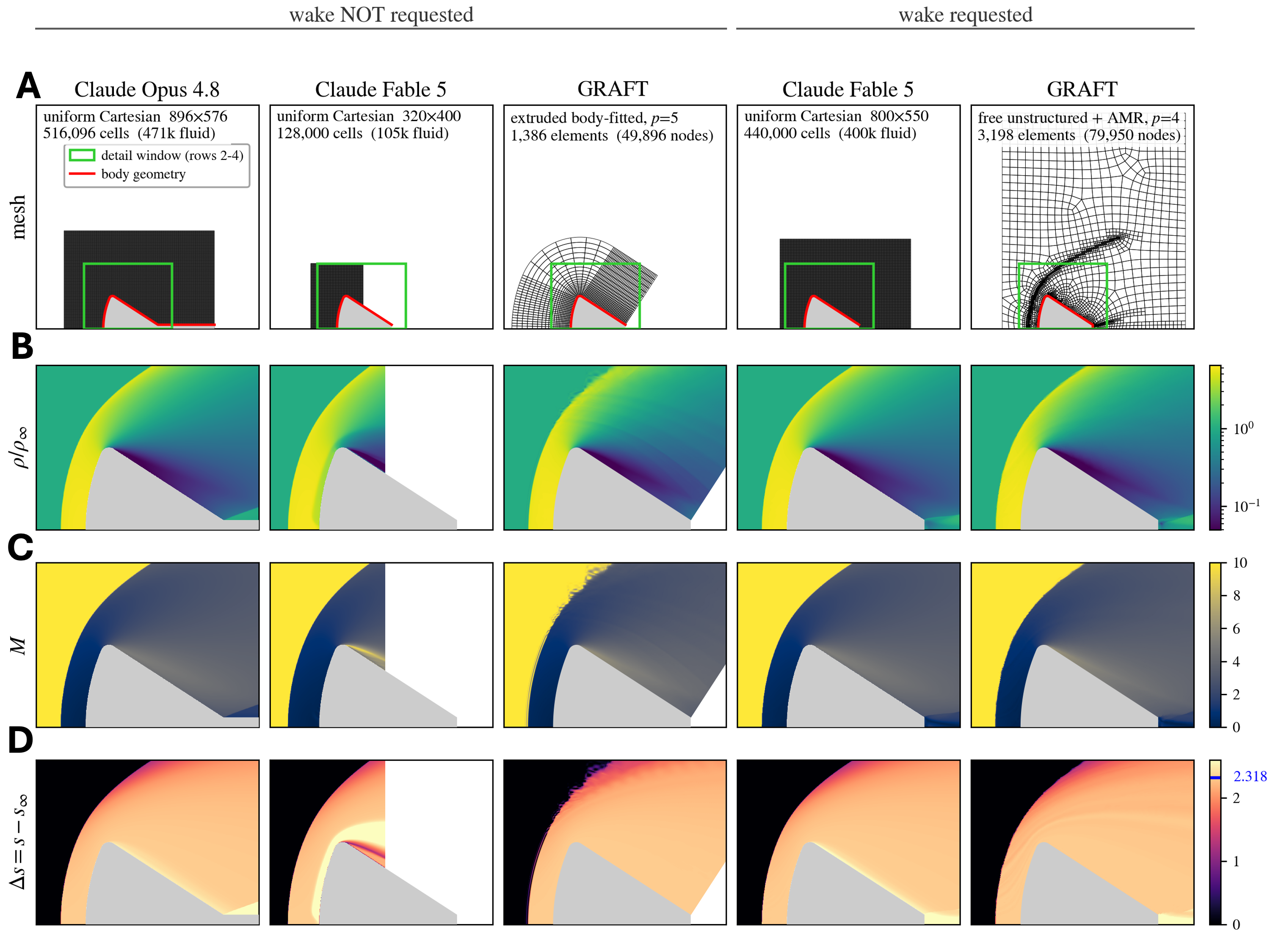}
\linespread{1}\selectfont
\caption{
\textbf{Agentic systems can validate converged Apollo solutions that do not solve the requested physical problem.}
GRAFT is compared with Opus 4.8 and Fable 5 through Claude Code on matched Apollo requests differing only in whether the wake is required. One of five Fable no-wake runs is shown.
(\textbf{A}) The commercial systems' uniform Cartesian meshes contain hundreds of thousands of cells, unresolved at this scale and therefore nearly black. In the no-wake cases, Opus adds a sting, altering the prescribed body; Fable imposes a fixed-freestream boundary adjacent to shock-disturbed flow. GRAFT preserves the geometry using high-order body-fitted or adaptively refined meshes with far fewer elements. Green boxes delimit the lower rows; red marks the prescribed body.
(\textbf{B}) Density fields appear plausible in every column, despite the altered Opus geometry and unphysical Fable wall layer. Density alone therefore does not reveal whether a converged run solves the requested physical problem.
(\textbf{C}) Mach-number fields expose a near-freestream streak beside the Fable no-wake body.
(\textbf{D}) The same solution has an unphysical layer below the normal-shock floor, $\Delta s=2.318$. The thin dark traces in the GRAFT fields are bounded Gibbs undershoots at the captured shock. The commercial audits nevertheless accepted the altered Opus geometry and unphysical Fable solution because both converged and remained close to common scalar references.
}
\label{fig:apollo_ablation}
\end{figure}

For Opus, the cylindrical sting contributed no direct axial loading, allowing this comparison to remain favorable even though the simulated vehicle had changed. For Fable, our subsequent flow audit revealed a near-freestream wall layer whose entropy fell below the physically admissible post-shock value (Fig.~\ref{fig:apollo_ablation}C,D). Both solutions nevertheless passed the agents' validation, showing that agreement with common scalar references can conceal violations of the prescribed geometry or physical admissibility.

\subsection{Mathematical analysis produces verified formulations and extends GRAFT}

Beyond established benchmarks, we applied the GRAFT-guided PIML workflow to reconstruct fluid flow in mouse perivascular spaces from in vivo data. The available measurements come from two-photon microscopy, which records sparse particle tracks within a thin observation slab~\cite{boster2023artificial,toscano2024inferring_AIV} (Fig.~\ref{fig:verified_discovery}B). Inferring the full three-dimensional velocity and pressure fields from these planar observations leaves flow variations outside the measured plane undetermined, which the Formalization team identified as a source of ill-posedness. To constrain these unobserved variations, the team proposed three closure options, from which the user selected a cross-slab smoothness regularizer. The particle tracks also contain measurement noise, for which GRAFT's retrieved neighborhood supplied an established treatment. Although the displayed nearest neighbor, an MRI inverse advection--diffusion problem, shares limited overall structure with the target, its successful noise treatment was favored by the prior constructed from the full neighborhood, contributing the denoising network used in the final method (Fig.~\ref{fig:verified_discovery}A).

With the selected regularizer, learned denoising network, and exact enforcement of mass conservation, the workflow reconstructed the three-dimensional velocity and pressure fields (Fig.~\ref{fig:verified_discovery}C). Predicted particle velocities agreed with held-out measurements within the observed slab (Supplementary Fig.~\ref{fig:verified_discovery_supp}C), and the reconstruction was consistent with published results despite using a smaller network and fewer training iterations~\cite{toscano2024inferring_AIV}. To establish what the regularizer guarantees, the Formal branch derived a machine-checked a posteriori error estimate. Under the stated data-coercivity condition, this estimate bounds the reconstruction error in terms of the achieved regularized loss, observation noise, and regularization bias (Supplementary Fig.~\ref{fig:verified_discovery_supp}A and Theorem~\ref{thm:perivascular_aposteriori}).

Beyond repairing ill-posed problems, the Formalization team also proposed a spectral PINN by combining neural approximation with classical spectral discretization. The user selected this proposal for further development in a separate periodic viscous Burgers problem with $\nu=1/100$ (Supplementary Table~\ref{tab:human_intervention}). The workflow represented the solution in Fourier modes, used a neural network to learn their time-dependent coefficients, and evaluated the nonlinearity pseudospectrally (Fig.~\ref{fig:verified_discovery}D). As the number of retained modes increased, the resulting time-only spectral PINN displayed the exponential error decay characteristic of spectral schemes~\cite{basdevant1986spectral} (Fig.~\ref{fig:verified_discovery}F; Supplementary Fig.~\ref{fig:verified_discovery_supp}D). The Formal branch also established an a priori bound showing that the architecture can attain exponentially decreasing approximation error under the stated assumptions, and machine-checked the proof in Lean 4~\cite{moura2021lean} (Supplementary Fig.~\ref{fig:verified_discovery_supp}B and Theorem~\ref{thm:spectral_apriori}). To encode this formulation, GRAFT added the required representation and constraint nodes to its action space, making these operations available for later problems (Fig.~\ref{fig:verified_discovery}E).

\section{Discussion}

LLM-based agentic systems are increasingly used to generate, execute, and evaluate solutions in scientific workflows~\cite{lupoiu2025multi,chen2026emseek,kaiser2025large}, but reliability requires more than executable or plausible answers. Consistent with previous studies~\cite{chan2025mle,huang2023mlagentbench,jiang2025aide}, audited commercial systems solved familiar, tightly specified tasks with suboptimal methods (Fig.~\ref{fig:semantic_memory_learning}B and Fig.~\ref{fig:apollo_ablation}). However, the Apollo audits showed that even a small change in the task specification could lead to solutions that converged and passed the agents' validation while violating the requested physics (Fig.~\ref{fig:apollo_ablation}). This behavior is consistent with language models being optimized for text prediction and preference alignment, objectives that do not ensure that generated choices preserve the assumptions and dependencies of a scientific problem~\cite{ouyang2022training,bai2022constitutional}.

\begin{figure}[H]
\centering
\includegraphics[width=0.9\textwidth]{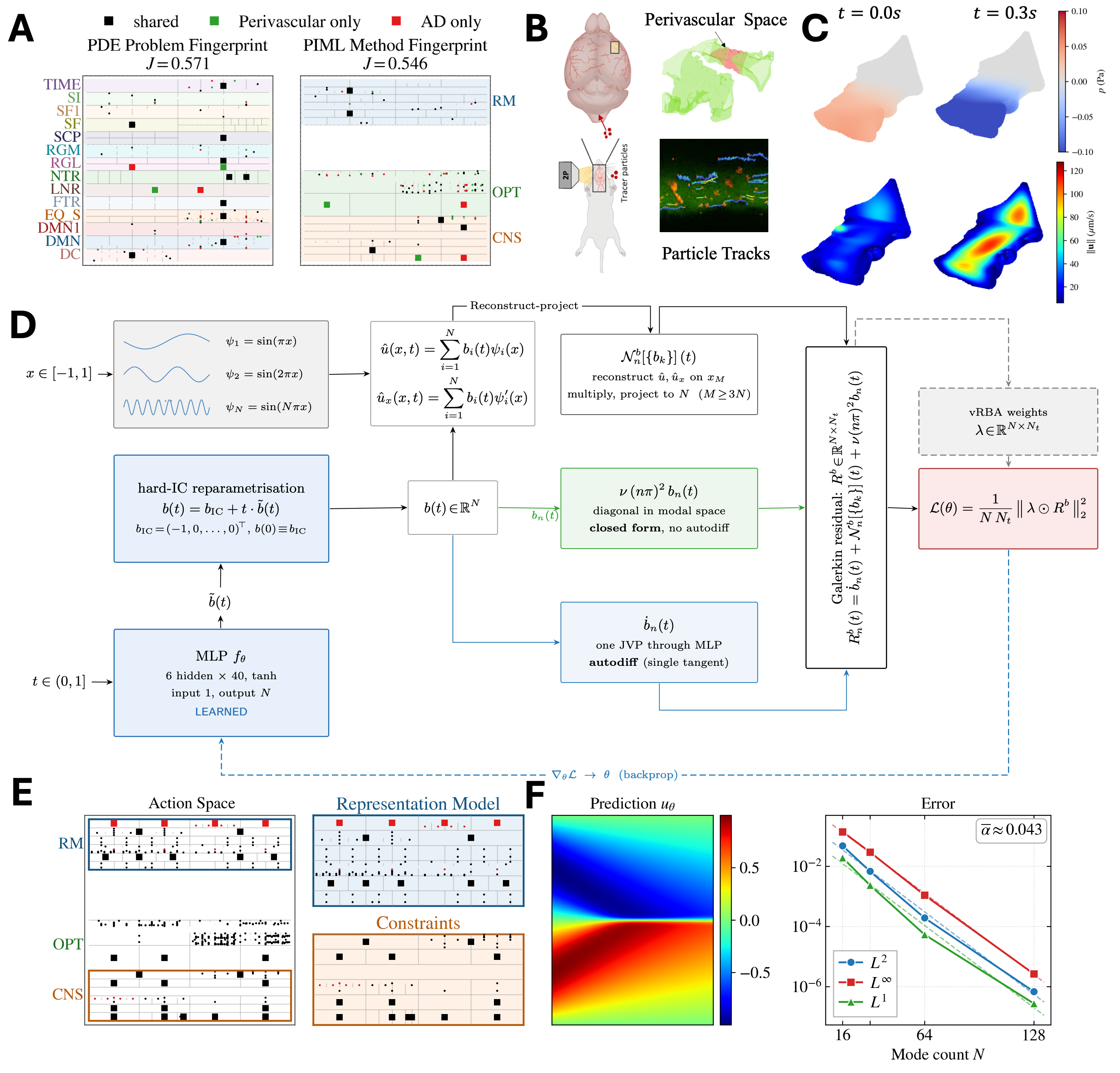}
\linespread{1}\selectfont
\caption{\small
\textbf{Mathematical diagnosis repairs an ill-posed inverse problem and produces a spectrally convergent PINN.}
(\textbf{A}) Perivascular problem and method fingerprints against the nearest inverse advection--diffusion case; the full retrieved neighborhood supplies the prior. In A, black entries are shared, green are target-specific, and red are neighbor-specific. Despite limited structural similarity, the neighborhood supplies successful noise treatment retained in the reconstruction.
(\textbf{B}) Sparse two-photon tracer tracks in mouse-brain perivascular space. Complete velocity and pressure fields must be recovered from these single-plane observations.
(\textbf{C}) Reconstructed pressure (top) and velocity magnitude (bottom) at two times. Formalization proposed three closures. The selected cross-slab regularizer constrains data-invisible modes; a vector potential enforces mass conservation exactly, supporting full-field reconstruction from B.
(\textbf{D}) Mathematical simplification of periodic viscous Burgers yields a spectral PINN whose time-only modal network satisfies the initial condition exactly, diagonalizes diffusion, evaluates nonlinearity pseudospectrally, and weights mode--time residuals through vRBA.
(\textbf{E}) The GRAFT-guided workflow adds the required representation and constraint operations to the PIML action space, selects the complete method, and retains the additions for later problems. Black nodes are existing actions and red nodes are newly added actions.
(\textbf{F}) Dense-grid prediction and $L^2$, $L^\infty$, and $L^1$ errors across mode counts. Exponential decay with $N$ demonstrates inherited spectral convergence; the corresponding theorem is machine-checked in Lean 4. Dashed curves fit $E(N)=C\,10^{-\alpha N}$. The mean fitted rate $\bar\alpha$ is measured in decades per mode across the three reported error measures.
}
\label{fig:verified_discovery}
\end{figure}

Common remedies are repeated sampling~\cite{romera2024mathematical,novikov2025alphaevolve,jiang2025aide}, human oversight, and expert knowledge encoded in system prompts. Sampling can recover good solutions but becomes impractical for complex, coupled workflows: in Apollo, sampling without dependency rules produced at most one admissible method configuration in 100 samples per procedure (Fig.~\ref{fig:sampling_sensitivity}). Expert prompts can narrow the search and support within-problem learning~\cite{toscano2025athena} (Fig.~\ref{fig:loop_efficiency}). Nevertheless, within-problem refinement does not by itself retain validated experience to guide subsequent problems. Retrieval-augmented generation or search over previous solutions offers one response, but textual or embedding proximity is not guaranteed to reflect solver-relevant scientific similarity (Fig.~\ref{fig:graft_semantic_geometry}D; Supplementary Fig.~\ref{fig:retrieval_baselines} and Table~\ref{tab:topk_benchmark}), and retrieving a relevant precedent does not by itself enforce its dependencies and admissibility rules or convert its evaluated outcome into action probabilities for a new problem. This suggests that a persistent substrate organized by scientific similarity, admissibility, and observed outcomes could help address these limitations.

This substrate follows naturally from scientific semantics. GRAFT organizes expert-defined characteristics, actions, dependencies, and admissibility rules into expandable problem and method graphs. Factorization keeps the resulting policy tractable while preserving encoded dependencies, and recursive projection yields unique fingerprints that make scientific distance measurable and link structural similarity to evaluated experience. With expert knowledge and evaluated experience available in this representation, the guided workflows exceeded published expert baselines in physics-informed learning and met task-specific validation criteria in the other evaluated families. The represented relationships also allow structurally similar problems to accelerate one another's search (Fig.~\ref{fig:loop_efficiency}) and expose connections obscured by terminology. For instance, in the cKAN universal-approximation theorem~\cite{ss2024chebyshev}, shared MLP structure allowed the system to establish universality directly despite different architectural terminology (Theorem~\ref{thm:ckan_universality}).

Reusing accumulated experience and exposing hidden connections can solve new problems, but scientific research also requires new formulations. Sampling can be inefficient for discovery in large action spaces~\cite{romera2024mathematical,georgiev2025mathematical}. For instance, even with the required actions available, no uniform procedure recovered the validated spectral-PINN architecture within 1,000 candidates (Fig.~\ref{fig:sampling_sensitivity}). In contrast, discovery in our study originated in mathematical analysis by the Formalization team. For the perivascular inverse problem, ill-posedness analysis identified a need for regularization and proposed closures. Similarly, mathematical simplification of periodic viscous Burgers produced the spectral-PINN architecture (Fig.~\ref{fig:verified_discovery}A--F; Supplementary Table~\ref{tab:human_intervention}). GRAFT's expandable structure then incorporated the new actions and dependencies without invalidating existing encodings or outcomes, making each discovery available to future problems. The proposals were subsequently implemented and evaluated, and their mathematical guarantees were machine-checked (Theorems~\ref{thm:perivascular_aposteriori} and \ref{thm:spectral_apriori}). Together, these stages enable GRAFT-guided agents to formulate hypotheses, test them computationally, and verify their mathematical consequences.

\begin{figure}[H]
\centering
\includegraphics[width=1\textwidth]{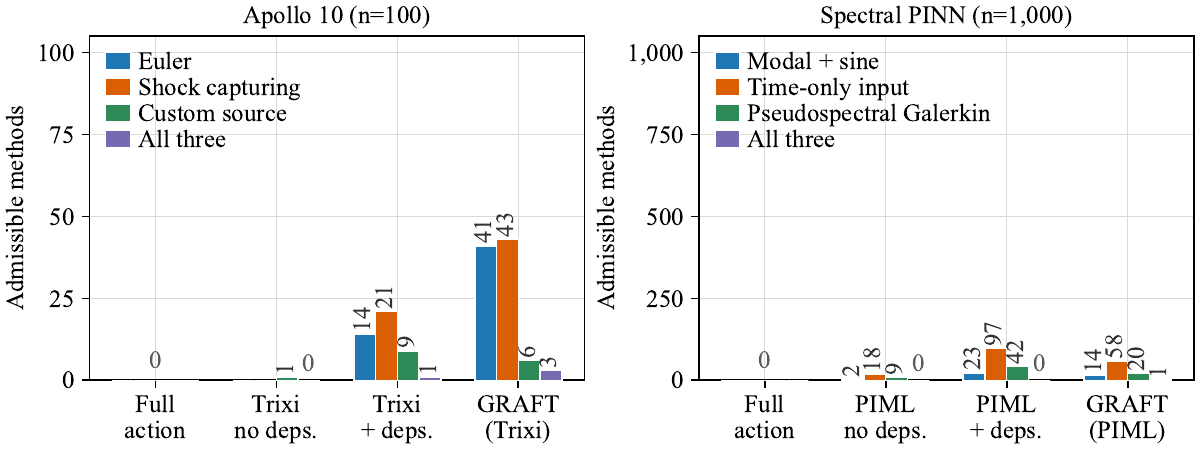}
\caption{
\textbf{Sampling alone rarely assembles the coupled choices required for scientific method discovery.}
Each panel asks whether sampling can recover three choices that must appear together in the final method. Component counts include only methods that pass the admissibility validator.
(\textbf{Left}) The Apollo Mach-10 method requires the Euler equations, shock capturing, and a custom axisymmetric source term. After 100 samples, neither flattened uniform sampling nor tree-structured sampling without dependencies recovers the complete method. Adding the encoded dependencies prevents incompatible choices and makes valid methods much easier to find, although only one sample contains all three required components. The GRAFT prior raises the count containing all three components to three in 100 and further favors the Euler equations and shock capturing because nearby solved problems used those choices. It provides little guidance toward the custom source term because that problem-specific choice is absent from the retrieved cases.
(\textbf{Right}) The spectral PINN requires a modal sine representation, a time-only network, and a pseudospectral Galerkin residual. None of the uniform procedures recovers all three choices in 1,000 samples. The GRAFT prior produces one sample containing all three, but the resulting method uses four networks and omits the hard initial and boundary constraints of the validated architecture. In both examples, previous solutions guide sampling toward familiar parts of the method but provide little guidance toward the unfamiliar choices that distinguish a new method. Mathematical formalization supplies that missing direction; GRAFT then encodes the resulting method so that it can guide later problems.}
\label{fig:sampling_sensitivity}
\end{figure}

Although most evaluations proceeded autonomously once the semantic spaces and seeds were established, human review remained important for the scientific validity of the representation and its outputs (Supplementary Tables~\ref{tab:human_intervention} and \ref{tab:human_intervention_b}). During graph construction and expansion, experts must verify that the encoded characteristics, actions, and dependencies preserve their intended scientific meaning. When a problem or proposed method cannot be encoded, the workflow returns to the user for guidance and expansion. Limited prior support triggers a warning and blocks execution in autonomous mode. Although final evaluation includes deterministic checks, some assessments still rely on LLM judgments (Methods). These judgments can accept a result that does not fully satisfy the original request. For instance, Lean correctly verified the encoded LNO theorem, but the final LLM audit accepted its weaker statement as satisfying the request, requiring human review to strengthen the result (Theorem~\ref{thm:lno_universality}).

Our results across four scientific families suggest that GRAFT provides a common structure for organizing scientific knowledge and guiding agentic workflows. Future work could extend this representation to additional families while using expert knowledge and evaluated experience to refine the hierarchies and dependencies through which problems and methods are organized. Beyond expanding coverage, such refinement could expose relationships that suggest new combinations of methods or new formulations. Coupling these proposals with mathematical analysis, computational evaluation, and formal verification could allow the representation to support the development of scientific knowledge as well as its reuse.

\section{Methods}

\subsection{GRAFT: Graph Reduction to Adaptive Factored Trees}

\subsubsection{Setup}\label{sec:setup}

We model a general agentic framework as a stochastic policy
\begin{equation*}
    \pi : \mathcal{P} \to \Delta(\mathcal{M}),
\end{equation*}
where $\mathcal{P}$ is the space of admissible problems (PDE specifications, inverse-data problems, particle-dynamics cases, and formal statements in our setting), $\mathcal{M}$ is the space of admissible methods, and $\Delta(\mathcal{M})$ is the probability simplex over $\mathcal{M}$. Concretely, $\pi(m \mid p)$ is the probability of selecting method $m$ on problem $p$. When deployment requires a single answer, we take the deterministic readout
\begin{equation*}
    \hat m(p) = \arg\max_{m \in \mathcal{M}} \pi(m \mid p),
\end{equation*}
i.e., the mode of $\pi(\cdot \mid p)$.

The method space $\mathcal{M}$ is built from an action vocabulary $\mathcal{A}$, structured as a directed knowledge graph $\mathcal{G}_A = (V, E_A)$ with vertex set $V = \mathcal{A}$. Each vertex is an individual action, for instance ``use Adam,'' ``use KAN,'' or ``use Fourier features.'' Notice that, in this sense, the action tree decomposes naturally into $k$ chains, each a subtree rooted at an $s$-decision (a ``pick one'' choice) with leaf alphabet $\mathcal{A}_j \subseteq \mathcal{A}$, the set of action choices available at chain $j$. A method is then a tuple
\begin{equation*}
    m = (a_1, \dots, a_k), \qquad a_j \in \mathcal{A}_j \cup \{\varnothing_j\},
\end{equation*}
where the null action $\varnothing_j$ marks chains left inactive because an upstream choice did not open them. The full method space
\begin{equation*}
    \mathcal{M} \;\subseteq\; \prod_{j=1}^{k} \bigl(\mathcal{A}_j \cup \{\varnothing_j\}\bigr)
\end{equation*}
is the subset of tuples consistent with the dependency and structural-nesting rules introduced below.

The policy $\pi$ is informed by past runs, recorded as the observation set $\mathcal{D} = \{(p_i, m_i, O_i, r_i)\}_{i=1}^{N}$, where $p_i \in \mathcal{P}$ is the problem, $m_i = (a_{i,1}, \dots, a_{i,k}) \in \mathcal{M}$ the method used to solve it, and $O_i$ a tuple of method-family-specific observables recorded during the run. For instance for PIML a potential choice is $O_i = (t_i, e_i)$ with $t_i$ the wall-clock time and $e_i$ the relative $L^2$ error; for numerical solvers, $O_i$ may bundle conditioning, entropy, iteration count, and residual; for data-driven methods, throughput and prediction error. The observables carry the information needed to qualify $m_i$ through a quantitative metric, the reward $r_i \in [0, r_{\max}]$, with $r_{\max} = 100$ fixed by construction (see Eq.~\eqref{eq:computational_reward}). Notice that $r_i$ ranks methods under the chosen reward metric: high-reward entries in $\mathcal{D}$ reinforce successful $(p, m)$ pairs, while low-reward entries register weak performance and contribute little positive reinforcement to nearby aggregations through the non-negative weighting of Eq.~\eqref{eq:nbr_weight}.

The reward rubric inherits ATHENA's four axes~\cite{toscano2025athena} (accuracy, efficiency, details, and optimality) and adds integrity as a fifth component in this work. It adapts the scoring to an expected-value setting: rather than ranking the run against the empirical maximum observed across past runs, each axis is scored against the expected value calibrated from the nearest neighbors of $p$ in $\mathcal{D}$ (see Eq.~\eqref{eq:nbr_weight}), so the reward responds to performance relative to what comparable problems in the corpus have achieved. The accuracy axis and the flatness component of the details axis are computed analytically from $O_i$ (the relative $L^2$ error against the reference, the residual flatness on the converged trace); the remaining components are scored by the Advisor team (see SI) under per-family rubrics, with no live human-in-the-loop step at runtime. The exact per-family weights and analytic forms vary by problem class and are tracked in the run logs alongside $\mathcal{D}$.

\subsubsection{Graph reduction: from DAG to tree}

Although the formulation above is consistent, working with $\mathcal{G}_A$ directly is impractical. A graph is convenient as a picture but cumbersome to operate on: in a DAG, a node reachable through several parents inherits a different role on each incoming path, so the information attached to it depends on context rather than on the node itself. To recover an independent meaning for each node, we project $\mathcal{G}_A$ onto a tree $\mathcal{T}$, on which every node has a single parent, every method corresponds to one root-to-leaf path per chain, and each node lands at a unique location under the embedding of Eq.~\eqref{eq:phi_embedding}. The cross-edge information sacrificed in the projection is preserved as a separate set of rules, introduced below.

The same construction applies to the problem graph $\mathcal{G}_P$, so we describe it generically for an arbitrary knowledge graph $\mathcal{G}$ and instantiate it independently on $\mathcal{G}_P$ and $\mathcal{G}_A$ to obtain $\mathcal{T}_P$ and $\mathcal{T}_A$. Two edge types in $\mathcal{G}$ are read directly from the underlying decisions: ``characterized by'', denoted $c$ and meaning ``all of these apply'', is the multi-attribute family-membership backbone, and ``subdivides in'', denoted $s$ and meaning ``pick one'', is the mutually exclusive discriminative decision that constitutes a chain alphabet $\mathcal{A}_j$. The reduction selects a spanning tree of $\mathcal{G}$ rooted at its root, so that every non-root vertex retains exactly one parent and multi-parent links are broken; for instance, a Fourier-features node reachable from both ``periodic domain'' and ``spectral architecture'' keeps a single canonical parent. The resulting tree
\begin{equation*}
  \mathcal{T} = (V_{\mathcal{T}},\ E_{\mathcal{T}})
\end{equation*}
has its internal vertices given by the $c$- and $s$-decision nodes, i.e., the chain alphabets $\{\mathcal{A}_j\}_{j=1}^{k}$ together with the $c$-categories that aggregate them, and its leaf vertices given by the actions $a \in \mathcal{A}_j$. Consequently, a method $m = (a_1, \dots, a_k)$ with $a_j \in \mathcal{A}_j$ is precisely a leaf-per-chain selection on $\mathcal{T}$.

\subsubsection{Encoding dependencies}\label{sec:encoding_dependencies}

The tree $\mathcal{T}$ exposes a natural unit of decision below the chain alphabets themselves. Cutting $\mathcal{T}$ at every $c$-node other than its own root partitions the tree into a forest whose pieces are rooted at $c$-nodes, contain only $s$-decision interior vertices, and terminate at leaves or at the next $c$-node down. We call each such piece a subchain and write $\mathcal{C} = \{C_1, \dots, C_{|\mathcal{C}|}\}$ for the resulting set. A subchain is a small Markov decision process: at each $s$-node along it the policy picks one child, the picked child becomes the next $s$-node, and the process terminates when a leaf or another $c$-node is reached. Subchains are the natural unit on which $\pi$ acts, and the chain alphabets $\{\mathcal{A}_j\}_{j=1}^{k}$ are recovered as the leaf alphabets of the subchains that terminate in leaves. Each tree node lies in exactly one subchain, recorded by the enclosing-chain map $\nu : V(\mathcal{T}) \to \mathcal{C} \cup \{\bot\}$, where $\bot$ is a sentinel marking the absence of an enclosing chain ($\nu(u) = \bot$ only at the global root). The nesting of subchains within $\mathcal{T}$ is recorded by the nearest-enclosing $c$-parent map $\rho : \mathcal{C} \to \mathcal{C} \cup \{\bot\}$: for each subchain $C$ rooted below a $c$-node, $\rho(C)$ is the first chain re-entered by walking up along $c$-edges from $C$'s root, or $\bot$ if the global root is reached without crossing into another chain. The pair $(\nu, \rho)$ records the structural nesting of $\mathcal{T}$.

The cross-edges discarded by the spanning-tree projection of the previous section are still needed: they carry the couplings between distinct subchains that $\mathcal{T}$ alone cannot represent. We collect them as a rule set $\mathcal{R}$ and defer the formal structure of an individual rule to the operator definitions below; for now $\mathcal{R}$ is the explicit list of cross-subchain couplings between the subchains in $\mathcal{C}$.

A direct way to use this structure would be to flatten $\mathcal{T}$ into a single decision step over the full leaf tuple $m = (a_1, \dots, a_k) \in \mathcal{A}_1 \times \cdots \times \mathcal{A}_k$, modeling $\pi$ as a distribution on the Cartesian product. The formulation is correct but treats every leaf tuple as its own atomic action, so the action space grows as $\prod_j |\mathcal{A}_j|$, exponentially in the number of subchains, and the rules in $\mathcal{R}$ are enforced only by carving forbidden tuples out of the product after the fact rather than by structuring the decision.

Observe instead that a $c$-node carries the meaning ``all of these apply'', so its descendant subchains are co-instantiated as parallel attributes rather than alternatives. We model them as conditionally independent in the absence of rules, and the rules in $\mathcal{R}$ collect exactly the cross-chain couplings the model chooses not to drop. Most pairs of subchains are therefore taken to be independent, and the few that are not are linked by an explicit cross-edge whose endpoints we know. The flat product wastes this structure by treating every pair as if it were coupled.

The easiest way to encode this observation would be to declare the subchains outright independent and write the policy as $\pi(m) = \prod_{j=1}^{k} \pi_j(a_j)$, with each $\pi_j \in \Delta(\mathcal{A}_j)$ taken from $\mathcal{T}$. Notice that under this construction the action space collapses from $|\mathcal{M}_0|$ to a sum of $k$ small distributions, but the cross-chain couplings collected in $\mathcal{R}$ disappear with it, and those couplings are precisely what records that a choice on one chain can change what is admissible on another, which we cannot afford to lose. However, most of the structural information needed to assemble a working method lives in $\mathcal{R}$, so the goal is a factorization that keeps those couplings while still reducing the action space below the naive Cartesian product.

\subsubsection{The action space and its factorization}\label{sec:decision_operators}

Notice that a useful factorization must reduce the joint while simultaneously preserving the constraints in $\mathcal{R}$ at the chains where they fire. We achieve this by lifting each rule into a local edit on the affected chain's distribution. In this setting, a rule $R_\ell \in \mathcal{R}$ is a four-tuple
\begin{equation}\label{eq:rule_tuple}
R_\ell \;=\; \bigl(h_\ell,\ T^{\mathrm{trig}}_\ell,\ T^{\mathrm{tgt}}_\ell,\ e_\ell\bigr),
\end{equation}
where $h_\ell$ is the natural-language sentence from which the rule is lifted, $T^{\mathrm{trig}}_\ell$ is a tuple of values on chains other than the affected one whose presence activates the rule, $T^{\mathrm{tgt}}_\ell \subseteq \mathcal{A}_{j(\ell)}$ is the target slice on the affected chain $j(\ell)$, and $e_\ell \in \{\mathcal{F}, \mathcal{Z}\}$ is the effect, namely force or zero out. For instance, in Trixi.jl the documented constraint that the Galerkin formulation \texttt{MHD-GLM} requires the surface flux to be either \texttt{LLF} or \texttt{HLLE} lifts to a rule with $T^{\mathrm{trig}} = \{\texttt{MHD-GLM}\}$ on the equation chain, $T^{\mathrm{tgt}} = \{\texttt{LLF}, \texttt{HLLE}\}$ on the surface-flux chain, and effect $\mathcal{F}$. Each rule contributes, for every $t \in T^{\mathrm{trig}}_\ell$, a directed edge $\nu(t) \to j(\ell)$ between distinct chains, and we collect these into the rule-induced edge set $E$ on $\mathcal{C}$ (dropping any $\nu(t) = j(\ell)$ self-loop, which contributes nothing).

Notice that the two effects $\mathcal{Z}$ and $\mathcal{F}$ correspond to the way a domain expert is most likely to phrase each constraint: $\mathcal{Z}$ for ``this combination is not allowed'' and $\mathcal{F}$ for ``this combination requires that''. Formally, given a current distribution $\pi_j \in \Delta(\mathcal{A}_j)$ on the leaves of chain $j$ and a target slice $T \subseteq \mathcal{A}_j$,
\begin{equation}\label{eq:zero_op}
\mathcal{Z}(\pi_j;\, T)(v) \;=\; \begin{cases} 0 & v \in T,\\[2pt] \dfrac{\pi_j(v)}{1 - \sum_{w \in T} \pi_j(w)} & v \notin T, \end{cases}
\end{equation}
\begin{equation}\label{eq:force_op}
\mathcal{F}(\pi_j;\, T)(v) \;=\; \begin{cases} \dfrac{\pi_j(v)}{\sum_{w \in T} \pi_j(w)} & v \in T,\\[2pt] 0 & v \notin T, \end{cases}
\end{equation}
where $\mathcal{Z}$ zeroes the target mass and renormalizes over the survivors and $\mathcal{F}$ does the dual, restricting support to the target and renormalizing there. The two are duals under set complement, $\mathcal{F}(\pi_j;\, T) = \mathcal{Z}(\pi_j;\, \mathcal{A}_j \setminus T)$. Both operators are well defined whenever the surviving target carries positive pre-mass: $\mathcal{Z}(\pi_j;\, T)$ requires $\pi_j(\mathcal{A}_j \setminus T) > 0$, and $\mathcal{F}(\pi_j;\, T)$ requires $\pi_j(T) > 0$. Sequential composition is order-independent on the open set where every intermediate denominator is positive; if cumulative effects exhaust support, the rule set is internally inconsistent at that trigger and the build aborts.

Notice that the effect of these operators can be combined. By construction, three properties of $\mathcal{F}$ and $\mathcal{Z}$ make them safe to apply in sequence on the same chain. (1) A leaf that survives keeps post-mass proportional to its pre-mass, so the policy's preference ordering inside the allowed set is untouched. (2) Both operators act on $\mathcal{A}_j$ alone (the trigger encodes where the constraint comes from, the operator only enforces it locally), so the per-chain factorization is preserved. (3) Multiple rules firing on the same chain commute under composition: two zero-outs reduce to support intersection, two forces likewise, and a force followed by a zero-out lands on $T^{\mathrm{tgt}}_F \setminus T^{\mathrm{tgt}}_Z$ with the original mass ratios irrespective of order. The composite distribution at sample time is therefore well defined regardless of application order, on the open set where the per-step support condition above holds. If the cumulative effect leaves empty support, the build aborts with the offending pair as a witness.

With operators in hand, the chains can be organized so that every rule's triggers are resolved before the chain it acts on. Collecting the rule-induced edges $E$ together with the structural-nesting edges $E_\rho = \{\rho(C) \to C : C \in \mathcal{C},\ \rho(C) \neq \bot\}$ defined above gives the chain dependency graph $H = (\mathcal{C},\, E \cup E_\rho)$, and the level of a chain is its longest-path length in $H$, well-defined whenever $H$ is acyclic (checked at build time, below). Equivalently, a level is a maximal set of chains whose triggers all reference values on chains assigned to strictly lower levels.

Under this construction, sampling proceeds level by level: at each chain $j$ active at the current level, the prior $\pi_j$ is taken from $\mathcal{T}$, every rule $R_\ell$ with $j(\ell) = j$ whose trigger has been resolved by previously sampled values is applied through its effect operator (in any order, by commutativity), and $a_j$ is sampled from the resulting distribution
\begin{equation}\label{eq:edited_chain}
\widetilde\pi_j\bigl(\,\cdot\,\bigm|\, a_{<\mathrm{level}(j)}\bigr) \;=\; e_{\ell_r}\!\Bigl(\,\cdots\, e_{\ell_1}\bigl(\pi_j;\, T^{\mathrm{tgt}}_{\ell_1}\bigr) \cdots\, ;\, T^{\mathrm{tgt}}_{\ell_r}\Bigr),
\end{equation}
where $\{\ell_1, \dots, \ell_r\}$ are the rules whose triggers are met by the lower-level values $a_{<\mathrm{level}(j)}$. To assign mass to the inactive case, write $\mathrm{act}_j(a_{<\mathrm{level}(j)}) \in \{0, 1\}$ for the indicator that chain $j$ is opened by the lower-level picks (determined by the structural nesting of $\mathcal{T}$ and any $\mathcal{F}/\mathcal{Z}$ effects in $\mathcal{R}$ that close $j$ outright), and extend the per-chain kernel to the augmented alphabet $\mathcal{A}_j \cup \{\varnothing_j\}$ by
\begin{equation}\label{eq:bar_pi_kernel}
\bar\pi_j\bigl(a_j \bigm| a_{<\mathrm{level}(j)}\bigr) \;=\;
\begin{cases}
\widetilde\pi_j\bigl(a_j \bigm| a_{<\mathrm{level}(j)}\bigr), & \mathrm{act}_j = 1 \text{ and } a_j \in \mathcal{A}_j,\\[2pt]
1, & \mathrm{act}_j = 0 \text{ and } a_j = \varnothing_j,\\[2pt]
0, & \text{otherwise.}
\end{cases}
\end{equation}
The full policy factorizes, for fixed problem $p$ and fixed repository $\mathcal{D}$, as
\begin{equation}\label{eq:level_factorisation}
\pi(m \mid p, \mathcal{D}) \;=\; \prod_{j=1}^{k} \bar\pi_j\bigl(a_j \bigm| a_{<\mathrm{level}(j)},\, p,\, \mathcal{D}\bigr),
\end{equation}
where each factor is the extended per-chain kernel of Eq.~\eqref{eq:bar_pi_kernel}, with $\widetilde\pi_j$ from Eq.~\eqref{eq:edited_chain} on its active branch. The construction reduces the action space from $|\mathcal{M}_0|$ to a sum of $k$ small distributions while preserving the cross-chain constraints in $\mathcal{R}$ where they fire, which was the goal stated above.

The procedure that derives $\mathrm{level}(\cdot)$ from $\mathcal{T}$ and $\mathcal{R}$ is the build-time pipeline given in Algorithm~\ref{alg:graft_build} in the SI: it reduces $\mathcal{G}_A$ to $\mathcal{T}$, identifies the chain set $\mathcal{C}$ with its structural maps $(\nu, \rho)$, expands each rule into the directed edges it induces on $\mathcal{C}$, and assigns levels by a monotone fix-point. Acyclicity of $H$ is checked via Tarjan's strongly-connected-components algorithm~\cite{tarjan1972}: a non-singleton component is a structural cycle witness and aborts the build. Each fix-point pass then extends levels by one edge along the longest path of $H$, so the loop converges in $\mathrm{longest\text{-}path}(H)$ passes and the total build cost is $O\bigl((|\mathcal{C}| + |E \cup E_\rho|) \cdot \mathrm{longest\text{-}path}(H)\bigr)$, dominated by the level assignment.

What the decomposition gives us is more than a rewrite of $\pi$ in fewer parameters. The cross-chain couplings collected in $\mathcal{R}$ are exactly the dependencies we cannot afford to lose, and the construction has them all: $\widetilde\pi_j$ depends on the lower-level values $a_{<\mathrm{level}(j)}$ only through the chains its rule triggers reference, so every coupling has an explicit graphical witness in the dependency structure. The conditional independencies asserted by the factorization are read off the absent edges of $H$, which by construction are the chain pairs that $\mathcal{R}$ does not couple. Writing $\mathrm{pa}_H(j) = \{\, i \,:\, i \to j \in E \cup E_\rho \,\}$ for the parents of chain $j$ in $H$, this is the tail-boundary condition of Verma \& Pearl's $I$-map theorem~\cite{vermapearl1988causal}, which yields the following.

\begin{proposition}[$I$-map of the decision-level factorization]\label{prop:imap}
For fixed problem $p$ and fixed repository $\mathcal{D}$, $H = (\mathcal{C},\, E \cup E_\rho)$ with parent function $\mathrm{pa}_H$ is an $I$-map of the policy $\pi(\,\cdot\,\mid p, \mathcal{D})$ of Eq.~\eqref{eq:level_factorisation}, and Eq.~\eqref{eq:level_factorisation} is equivalently the canonical Bayesian-network factorization on $H$~\cite[Ch.~3 \S3.3.2]{pearl1988probabilistic},
\begin{equation}\label{eq:bn_factorisation}
\pi(m \mid p, \mathcal{D}) \;=\; \prod_{j=1}^{k} \bar\pi_j\bigl(a_j \bigm| a_{\mathrm{pa}_H(j)},\, p,\, \mathcal{D}\bigr).
\end{equation}
\end{proposition}

The proof, which verifies the hypotheses of Verma \& Pearl's theorem for the construction of Eq.~\eqref{eq:bn_factorisation}, is given in the SI. The upshot is that every conditional independence read off $H$ via d-separation holds in $\pi(\,\cdot\,\mid p, \mathcal{D})$, so no coupling in $\mathcal{R}$ or in the nesting structure of $\mathcal{T}$ is silently dropped by the per-chain factorization.

Notice that the Bayesian-network form of Eq.~\eqref{eq:bn_factorisation} also yields a sharp storage bound. A dense conditional table for chain $j$ on its parent context would scale as $|\mathcal{A}_j| \prod_{i \in \mathrm{pa}_H(j)} |\mathcal{A}_i|$, and the full joint over the naive method space $\mathcal{M}_0 = \mathcal{A}_1 \times \cdots \times \mathcal{A}_k$ as $\prod_{j=1}^{k} |\mathcal{A}_j|$, exponential in $k$. GRAFT stores neither: for each chain $j$ and each internal $s$-node $u$, it keeps a single row over the children of $u$, shared across all parent contexts, and the rule set $\mathcal{R}$ then acts as row masks and overrides whenever triggers from $a_{\mathrm{pa}_H(j)}$ fire. Total storage therefore scales as $\sum_{j=1}^{k} |\mathcal{A}_j| + |\mathcal{R}|$ rather than $\prod_{j=1}^{k} |\mathcal{A}_j|$, reducing the parameter footprint from exponential to linear in $k$ (the row-level construction is taken up in Eqs.~\eqref{eq:policy_data}--\eqref{eq:policy_blend}). GRAFT compactly represents the same declared admissibility constraints and shared base preferences as their fully expanded conditional form: the dense table stores one categorical row per parent context, while GRAFT stores one shared row per decision plus the declared force/zero-out rules. The count therefore measures the cost of storing the declared expert policy, not equivalence to an arbitrary context-dependent conditional table. Finer context dependence is carried by the node hints and the proposer--critic agent (see below) rather than by stored probabilities.

\subsubsection{Knowledge graphs and embedding}\label{sec:embedding}

Notice that the factorization of Eq.~\eqref{eq:level_factorisation} fixes $\pi$'s structure but does not yet say how the per-chain priors $\pi_j$ should be populated for a new problem. We want the prior to reflect past experience: when a new problem $p$ arrives, $\pi_j$ should be biased by what worked on problems already solved that resemble it, with closer neighbors weighted more heavily than distant ones. The categorical paths in $\mathcal{T}_P$ do not give us a numerical distance directly, so we need to define an embedding of the descriptors of $\mathcal{T}_P$ to identify which problems are close to each other and to the new problem $p$.

Observe that the problem tree $\mathcal{T}_P$ has nodes that are categorical descriptors (equation type, dimensionality, boundary conditions, regularity, and so on), and the action tree $\mathcal{T}_A$ has nodes that are method decisions (architecture, optimizer, loss, weighting). Each problem $p$ picks out one root-to-leaf path $\sigma_P(p) = (v_1, \dots, v_d)$ through $\mathcal{T}_P$, and each method $m = (a_1, \dots, a_k)$ picks out one root-to-leaf path through $\mathcal{T}_A$ per active chain $j$ (inactive chains carry $a_j = \varnothing_j$ and contribute none); we write $\sigma_A(m)$ for the union of those chain paths. In short, $\sigma_P$ records which descriptors a problem instantiates, $\sigma_A$ records which decisions a method makes, and together they are the vertex sets that every downstream construction in this section acts on.

To turn these categorical paths into vectors that can be compared geometrically, each vertex $v$ of $\mathcal{T}$ receives a deterministic position in the unit cube via the recursive map
\begin{equation}\label{eq:phi_embedding}
    \Phi : V \to [0,1]^3, \qquad v \mapsto \bigl(x(v),\ y(v),\ z(v)\bigr), \qquad z(v) = d(v)/D,
\end{equation}
where $D = \max_{v \in V} d(v)$ is the tree depth, so $z(v) \in [0,1]$ and the leaf layer lands on $z = 1$. The $(x, y)$ coordinates are computed by Algorithm~\ref{alg:partition_layout} in the SI: the root owns the unit square, every interior vertex sits at the center of its $(x, y)$-rectangle, and its children are slotted into off-center sub-rectangles by splitting along $x$ for $s$-children (the mutually-exclusive picks) and along $y$ for $c$-children (the joint attributes). For odd-$n$ groups the algorithm opens $n + 1$ slots and drops the middle one, so no child centroid coincides with the parent's midpoint and the parent's projection is never shadowed by any descendant.

We make explicit the one structural property of $\mathcal{T}$ that the construction relies on: every internal node has children of a single edge type, i.e., a node is either subdivided ($s$-children) or characterized ($c$-children), never both. The ontology enforces this throughout, so the recursion at every $v$ performs exactly one partition along a single axis, and the children's sub-rectangles tile $v$'s rectangle without overlap. Under that assumption, $\Phi$ is injective:

\begin{proposition}[Injectivity of $\Phi$]\label{prop:phi_injective}
Let $\mathcal{T}$ be a finite rooted tree with edge-type labels $\tau$ satisfying the uniform-children structural assumption above. Let $\Phi : V(\mathcal{T}) \to [0,1]^3$ be the embedding produced by Algorithm~\ref{alg:partition_layout}. Then $\Phi$ is injective on $V(\mathcal{T})$, and so is its planar projection $\pi_{xy} \circ \Phi : V(\mathcal{T}) \to [0,1]^2$.
\end{proposition}
The proof is given in the SI.

To compare problems through this embedding, the unit cube can be discretized into a $K \times K \times (D+1)$ grid, each vertex being assigned to its bin via the boundary-clamped bin map
\begin{equation}\label{eq:bin_K_methods}
    \beta_K(v) = \bigl(\,\min(K - 1,\ \lfloor K \cdot x(v) \rfloor),\ \min(K - 1,\ \lfloor K \cdot y(v) \rfloor),\ d(v)\,\bigr),
\end{equation}
so that the planar indices stay in $\{0, \dots, K - 1\}$ even at the boundary $x(v) = 1$ or $y(v) = 1$. In all results and figures reported in this work, the systems operate at the injective resolution $K^\star$ introduced below. The fingerprint of a solved problem $P_i \subset V$ at resolution $K \in \mathbb{N}$ is then
\begin{equation}\label{eq:fingerprint}
    F(P_i, K) = \bigl\{ \beta_K(v)\ :\ v \in P_i \cap V_{\text{keep}} \bigr\},
\end{equation}
a finite cell set in the discretized unit cube, equivalently a sparse binary tensor of shape $K \times K \times (D+1)$. To compare fingerprints, each depth layer carries a level weight $\omega(d) > 0$, non-increasing in $d$, so agreement on early, high-level decisions counts at least as much as agreement on deep refinements; because the third bin coordinate is the exact depth $d(v)$, every cell contains nodes of a single depth and inherits its weight unambiguously. Pairwise similarity between two fingerprints is the hierarchically weighted Jaccard
\begin{equation}\label{eq:fingerprint_jaccard}
    J_K(P_i, P_j) = \frac{\sum_{c \in F(P_i, K) \cap F(P_j, K)} \omega\bigl(d(c)\bigr)}
                         {\sum_{c \in F(P_i, K) \cup F(P_j, K)} \omega\bigl(d(c)\bigr)}
    \in [0, 1],
\end{equation}
where $d(c)$ denotes the depth index of cell $c$.
The active node set $V_{\text{keep}}$ controls which edge types contribute: restricting to $s$-edges (the default) yields a discriminative measure, while including the $c$-backbone instead measures cross-family resemblance.

The grid resolution $K$ controls how finely the embedding is resolved. Let $K^\star$ denote the smallest $K$ at which $\beta_K$ is injective on the current $V(\mathcal{T})$. At $K = K^\star$, distinct nodes never share a cell and $J_K$ separates problems down to exact set membership in $\mathcal{T}_P$. Note that injectivity is asserted at $K^\star$ on the present tree, not for every $K \geq K^\star$: floor binning shifts grid lines as $K$ varies, so two nodes separated at $K^\star$ may collide at a coarser or finer resolution. Neighbor ranking and policy aggregation operate at $K^\star$, where cells and nodes are in bijection and $J_{K^\star}$ is evaluated directly on the node sets $P_i \cap V_{\text{keep}}$, and the rendered spaces of Fig.~\ref{fig:graft_semantic_geometry} and Supplementary Figs.~\ref{fig:supp_full_space_pde_piml} and \ref{fig:supp_full_space_formal_dpd} draw the continuous coordinates of Eq.~\eqref{eq:phi_embedding}, so every fingerprint reported in this work is the full representation (see Eq.~\eqref{eq:fingerprint_jaccard} with $\omega \equiv 1$; sensitivity to this choice is reported in Supplementary Table~\ref{tab:weight_sensitivity}).

The two vector spaces obtained from $\mathcal{T}_P$ and $\mathcal{T}_A$ support four downstream operations used throughout the rest of the paper: a distance between problems (how similar two PDEs are), a distance between methods (how similar two strategies are), a reward calibration that takes the nearest neighbor's reward as a baseline, and a landscape visualization in which problem vectors are projected onto one axis (via PCA), method vectors onto another, and points are colored by loss, wall time, or system size. Each solved problem becomes a point, and the landscape fills in as GRAFT-guided workflows solve more problems.

\subsubsection{Policy and prior update}\label{sec:policy_update}

Each chain $j$ carries a transition matrix $M^{(j)}$, with row $M^{(j)}(u, \cdot)$ the categorical distribution over the children of node $u$ and $M^{(j)}(u, v)$ the probability of selecting child $v$ given parent $u$. The chain-level prior $\pi_j \in \Delta(\mathcal{A}_j)$ defined above is the path product of these rows along $\mathcal{T}_A$, so populating $\pi_j$ for a new problem reduces to populating the rows $M^{(j)}(u, \cdot)$. We write $\mu^{(j)}$ for the uniform prior on chain $j$, with $\mu^{(j)}(u, \cdot)$ assigning equal probability to every child of $u$. Rule firings in $\mathcal{R}$ act on these rows via the operators $\mathcal{F}, \mathcal{Z}$ defined in Eqs.~\eqref{eq:zero_op} and \eqref{eq:force_op}, applied after the data-driven blend below so that documentation rules cannot be undone by neighbor evidence; the build-time validator rejects any zero-out operation that leaves empty support.

When a new problem $p_{\text{new}}$ arrives, the rows of $M^{(j)}$ are re-estimated from neighbor evidence. Notice that closeness to $p_{\text{new}}$ alone is not enough: a similar problem that was solved badly tells us little about which children of $u$ to favor, and a high-reward problem that is unrelated to $p_{\text{new}}$ should not bias the choice either. We therefore want neighbors that are simultaneously close to $p_{\text{new}}$ and were solved well, and the prior should be biased towards their actions in proportion to that joint score. Concretely, the top-$N_{\text{nbr}}$ nearest solved problems are picked by $J_{K^\star}$, and each neighbor $i$ contributes a sigmoid-gated reward weight
\begin{equation}\label{eq:nbr_weight}
    w_i \;=\; \sigma\bigl(\mathrm{sim}(p_{\text{new}}, p_i)\bigr) \cdot \frac{r_i}{r_{\max}}, \qquad \sigma(s) = \frac{1}{1 + e^{-\kappa(s - s_0)}}, \qquad (\kappa,\, s_0) = (7,\, 0.55),
\end{equation}
with $\mathrm{sim} := J_{K^\star}$ on $s$-leaves and $r_i / r_{\max} \in [0, 1]$ (see above). Each neighbor $i$ then contributes a per-node distribution $S_i$ over the children of every internal $u \in \mathcal{T}_A$, set to a one-hot vote on rows along $\sigma_A(m_i)$ and to the prior on rows it did not visit,
\begin{equation}\label{eq:partial_spec}
    S_i(u, \cdot) =
    \begin{cases}
        \delta_{m_i(u)}, & u \in \sigma_A(m_i),\\[2pt]
        \mu(u, \cdot), & u \notin \sigma_A(m_i),
    \end{cases}
\end{equation}
where $\delta_v$ is the one-hot row placing all mass on child $v$ and $m_i(u)$ is the child of $u$ chosen by neighbor $i$'s method along its root-to-leaf path. The data-driven estimate is the $w_i$-weighted average of these contributions across the full top-$N_{\text{nbr}}$ set, written here with the path-on / path-off split made explicit,
\begin{equation}\label{eq:policy_data}
    M_{\text{data}}(u, \cdot) = \frac{1}{W_{\text{tot}}} \Biggl[\sum_{i\,:\,u \in \sigma_A(m_i)} w_i\, \delta_{m_i(u)} \;+\; \sum_{i\,:\,u \notin \sigma_A(m_i)} w_i\, \mu(u, \cdot)\Biggr], \qquad W_{\text{tot}} = \sum_{i=1}^{N_{\text{nbr}}} w_i,
\end{equation}
with the trivial fallback $M_{\text{data}} = \mu$ when $W_{\text{tot}} = 0$. On rows every neighbor visits, the second sum is empty and the average reduces to a $w_i$-weighted vote among their picks; on rows no neighbor visits, the first sum is empty and the average collapses to $\mu$; intermediate rows interpolate, with shrinkage toward $\mu$ growing as the visiting weight share $\sum_{i:\,u \in \sigma_A(m_i)} w_i\, /\, W_{\text{tot}}$ decreases. The final policy is a row-wise convex combination of $M_{\text{data}}$ with $\mu$,
\begin{equation}\label{eq:policy_blend}
    M = \bar{W} \cdot M_{\text{data}} + (1 - \bar{W}) \cdot \mu, \qquad \bar W = \mathrm{clip}\!\left(\frac{W_{\text{tot}}}{N_{\text{eff}}},\, 0,\, 1\right), \qquad N_{\text{eff}} = \bigl|\{i : w_i > 0\}\bigr|,
\end{equation}
where the mixing weight $\bar W$ is the average confidence among neighbors that contribute non-trivial evidence (with the convention $\bar W = 0$ when $N_{\text{eff}} = 0$)~\cite{hoeting1999bma}. High $\bar W$ trusts the data-driven estimate, low $\bar W$ falls back to the prior, and $\bar W = 0$ degenerates $M$ cleanly to $\mu$. Rule operators are then applied on top of $M$, so any forced or forbidden support encoded in $\mathcal{R}$ overrides the blended estimate. Algorithm~\ref{alg:policy_update} in the SI bundles the full procedure.

After solving $p_{\text{new}}$ and observing $(m, r, O)$, the tuple $(p_{\text{new}}, m, O, r)$ is stored in $\mathcal{D}$, so future queries near $p_{\text{new}}$ may include it in their top-$N_{\text{nbr}}$ aggregation and the policy is updated without explicit retraining. When the action graph itself is expanded, new nodes inherit priors from their siblings; the repository $\mathcal{D}$ therefore grows monotonically in stored observations. The local prior used for any given query may still shift if the top-$N_{\text{nbr}}$ neighbor set changes, so monotone growth of $\mathcal{D}$ does not imply monotone improvement of the policy at every individual problem.

\subsubsection{Initialization of the solved-case memory}
\label{sec:memory_initialization}

Before any evaluation target in the Results was posed, the repository $\mathcal{D}$ was initialized with 40 solved records: three PIML, thirteen Numerical, four DPD, and twenty Formal cases. Each record stores a problem fingerprint, the method that solved it, its observed outputs, and its reward. This inventory is the fixed initial support from which the target-dependent neighborhoods defined in Eq.~\eqref{eq:nbr_weight} are retrieved; it is not itself a list of the neighbors selected for any one target. Tables~\ref{tab:memory_seeds_piml}--\ref{tab:memory_seeds_formal_b} report the complete initialization set and the provenance of each implementation or proof.

\subsubsection{From prior to agent}\label{sec:agent_sampler}

Notice that $\pi(m \mid p)$ in Eq.~\eqref{eq:bn_factorisation}, with rows $M^{(j)}(u, \cdot)$ populated as in Eq.~\eqref{eq:policy_blend}, is built from $\mathcal{G}_A$ alone: those rows inherit problem context only through that policy update, which leans on neighbor problems already solved. Two gaps remain at sampling time. First, the present problem $p$ has its own descriptors in $\mathcal{T}_P$ (dimensionality, boundary type, regularity, and so on) that the prior cannot read directly; a single draw from $\pi(\cdot \mid p)$ ignores any structure of $p$ that the neighbor aggregation has not already absorbed. Second, the leaves of $\mathcal{T}_A$ are categorical labels, but the actual specification of an action is not. A choice such as ``MLP'' carries continuous knobs (learning rate, width, depth) and structural sub-decisions (variable ordering for KAN-style models, layer-wise loss weighting, implementation prescriptions) that are virtually non-finite and that no categorical sampler over $\mathcal{T}_A$ can produce.

Observe that closing both gaps with $\pi$ alone is not realistic. Enumerating every continuous knob inside the tree is impossible, and waiting for the row prior to sharpen enough that any sample from it is automatically calibrated to the present problem is expensive in practice: a flat prior at small $|\mathcal{D}|$ has to rediscover by trial what the literature already records. ATHENA's inviscid-Burgers runs~\cite{toscano2025athena} illustrate this cost: several attempts were needed before the row mass concentrated on the conservative formulation that is, in retrospect, a textbook choice for that family of equations. The same cost would recur on every new problem class until the corpus is large enough to cover it.

We close both gaps with two agents working in tandem on the prior $\pi(\cdot \mid p)$: a proposer that samples from $\pi$ at each level and a critic that audits the sample. The pair traverses $H$ level by level. At each level, the proposer is shown the active chains' children together with their current row probabilities $M^{(j)}(u, \cdot)$, the hints attached to those children, and the descriptors of $p$ read from $\mathcal{T}_P$, and it emits one action $a_j$ per active chain, realizing any sub-leaf knobs (learning rate, width, depth, variable orderings, loss weightings) inline. The proposer's default is to follow the probabilities; it deviates only when a hint is more informative than the prior, for instance steering ``problem is periodic'' toward a periodic embedding even if row mass favors a generic choice. Once a level's picks are fixed, the rule operators defined in Eqs.~\eqref{eq:zero_op} and \eqref{eq:force_op} act deterministically on the next level's rows: $\mathcal{F}$ forces support where a triggered rule says ``must'', while $\mathcal{Z}$ removes it where the rule says ``forbidden''. Thus $a_{\mathrm{pa}_H(j)}$ is fully resolved before chain $j$ is sampled, matching the conditioning structure of Eq.~\eqref{eq:bn_factorisation}. The critic then verifies that the proposer's emission respected both the probabilities and the hints; rejections trigger a re-proposal at the same level, and the next level is taken up only after the current level passes~\cite{jha2023counterexample}.

With the agent in the loop, the realized method distribution is not $\pi$ but an agent-induced distribution $P_{\text{agent}}$ over $\mathcal{M}$, sampled chain by chain in the topological order of $H$. To certify that $H$ remains an $I$-map of $P_{\text{agent}}$, we make precise the role of the hints.
\begin{assumption}[Locality of the proposer--critic pair]\label{ass:hint_locality}
For every chain $j$, the hints presented to the proposer are a deterministic function
\begin{equation*}
\mathrm{hints}_j \;=\; h_j\!\bigl(j,\, a_{\mathrm{pa}_H(j)}\bigr)
\end{equation*}
of the chain identity and the values selected on its parents in $H$, and the proposer's policy at chain $j$ depends on the predecessor history only through $a_{\mathrm{pa}_H(j)}$ and $\mathrm{hints}_j$. The critic's acceptance event at chain $j$ is likewise a function only of the candidate $a_j$, the parent picks $a_{\mathrm{pa}_H(j)}$, the local row $M^{(j)}(u, \cdot)$, and $\mathrm{hints}_j$, and does not depend on non-parent predecessor history. Finally, when the strategist consults the short-term history $\mathsf{History}_n$ to avoid revisiting failed configurations, the filter applied at chain $j$ depends only on the per-chain projection $\mathsf{History}_n^{(j)} = \{(a_{i,j}, a_{i,\mathrm{pa}_H(j)}, R_i) : i < n\}$, that is, on prior values at chain $j$ and its $H$-parents together with their realized rewards.
\end{assumption}
\begin{proposition}[Per-step $I$-map under the agent]\label{prop:imap_agent}
For fixed problem $p$, fixed repository $\mathcal{D}$, and fixed within-trial history $\mathsf{History}_n$ at iteration $n$, under Assumption~\ref{ass:hint_locality}, $H$ with parent function $\mathrm{pa}_H$ is an $I$-map of the per-step distribution $P_{\text{agent}}(\,\cdot\,\mid p, \mathcal{D}, \mathsf{History}_n)$, and
\begin{equation}\label{eq:bn_factorisation_agent}
P_{\text{agent}}(m \mid p, \mathcal{D}, \mathsf{History}_n) \;=\; \prod_{j=1}^{k} P_{\text{agent}}\!\bigl(a_j \bigm| a_{\mathrm{pa}_H(j)},\, p,\, \mathcal{D},\, \mathsf{History}_n^{(j)}\bigr).
\end{equation}
\end{proposition}
The proof, which checks the tail-boundary condition of Theorem~\ref{thm:verma_pearl} for $P_{\text{agent}}$, is given in the SI.

Beyond sampling, $\mathcal{T}_A$ itself is flexible: once a problem is solved, the user or an agent may grow $\mathcal{G}_A$ (and thereby $\mathcal{T}_A$) by attaching new actions or by adding rules to $\mathcal{R}$. New nodes inherit priors from siblings and $\mathcal{D}$ persists across the expansion. The extension re-enters Algorithm~\ref{alg:graft_build} with the new pair $(\mathcal{T}', \mathcal{R}')$, and the recompiled $H'$ inherits the $I$-map certification of the SI as long as $H'$ is acyclic, which the build-time check enforces.

This places the architecture in the lineage of expert systems and earlier symbolic-AI search~\cite{quillian1966semantic, newell1956logic, newell1963guide, hart1968formal, wooldridge2002introduction} but defends against their classical limitations on four counts. The rule set $\mathcal{R}$ comes from documentation prose rather than handcrafted code, so authoring scales with the literature and updates with it. The prior $\pi$ is data-driven and sharpens with $|\mathcal{D}|$, so the system learns rather than freezing at its initial encoding. The system remains operational without node hints because the learned prior and the admissibility rules remain available. And sub-leaf realization, which classical expert systems could not represent at all, lives in the agent rather than in the symbolic layer. The empirical contribution of the hints, their robustness to incorrect guidance, and their scaling with repository size are left to future work.

\subsubsection{The closed loop: GRAFT-guided scientific workflows}\label{sec:closed_loop}

The system carries two memories. The repository $\mathcal{D}$ is the long-term store, persistent across problems and the only object the prior is compiled from. The trial history $\mathsf{History}_n$ is the short-term store, scoped to the current problem $p_{\text{new}}$ and used by the strategist to avoid re-sampling configurations that have already failed within the trial; its entries are also recorded in $\mathcal{D}$, and $\mathsf{History}_n$ is reset for the next problem. In Pearl's vocabulary, $(\mathcal{D}, \mathsf{History}_n)$ is the beginning of an abduction substrate for retrospective queries: each recorded tuple $(p_i, m_i, O_i, r_i)$ provides the evidence such a query would condition on, but full abduction would require richer trace logging of implementation edits, solver seeds, proposal state, and other exogenous variables.

A trial proceeds as a sequence of steps $n = 0, 1, \dots$. At step $n$, the system maintains a context $\mathsf{Context}_n$ (the problem statement $p_{\text{new}}$, the current code state $S_n$, and the short-term history $\mathsf{History}_n$); the strategist agent, guided by the closed-loop policy $\Pi(A_n \mid \mathsf{Context}_n)$, selects an action $A_n = m_n = (a_{n,1}, \dots, a_{n,k}) \in \mathcal{M}$. Writing $f$ for the active scientific family, the system steps according to
\begin{equation}\label{eq:bandit_step}
    S_{n+1} = \mathcal{I}(A_n, S_n), \qquad O_n = \mathsf{Exec}(S_{n+1}), \qquad R_n = R_f\!\left(O_n \mid p_{\text{new}}, m_n, \mathcal{D}, \mathcal{R}\right),
\end{equation}
where $\mathcal{I}$ is the implementation operator (realizes $A_n$ as code edits applied to $S_n$), $\mathsf{Exec}$ is the execution operator (runs or verifies the resulting state and returns observations), and $R_f$ is the family-specific informed reward defined below. In Pearl's structural-causal vocabulary, conditional on the problem instance $p_{\text{new}}$, the code state, solver settings, seeds, and runtime environment, $\mathsf{Exec}$ plays the role of the structural map carrying configured code states to their observable consequences. The agent's pick $A_n$ is therefore an intervention on the executable solver state, not a passive observation. The chain dependency graph $H$ should be read more narrowly: it is a policy-dependency $I$-map over action choices (Proposition~\ref{prop:imap_agent}), while the causal intervention occurs when the selected method is realized and executed through $\mathsf{Exec}$. The short-term history grows as $\mathsf{History}_{n+1} = \mathsf{History}_n \cup \{(A_n, O_n, R_n)\}$. The closed-loop policy $\Pi$ draws through the rows of the compiled prior $\pi(\cdot \mid p_{\text{new}}, \mathcal{D})$, which is held fixed for the duration of the trial; the proposer's and critic's selection within each row is filtered by $\mathsf{History}_n$ so that configurations already tried on $p_{\text{new}}$ are avoided. Sampling itself proceeds via the explicit factored distribution of Eq.~\eqref{eq:bn_factorisation_agent}: the agent traverses the chain DAG in topological order and, at each chain $j$, is guided by the row $M^{(j)}(u, \cdot)$ together with the within-trial filter from $\mathsf{History}_n$. Graph-encoded hints baked into $\mathcal{T}_A$ at construction (the dependency rules $\mathcal{R}$ defined above and the grading prose attached to each node) enter $\mathsf{Context}_n$ alongside the runtime state, giving the agent direct read access to the structural constraints at every step rather than re-deriving them from the policy distribution.

Given a new problem $p_{\text{new}}$, the cycle proceeds as in Algorithm~\ref{alg:graft_athena_cycle}: GRAFT compiles the prior, the proposed agentic workflow runs the inner trial loop, and the resulting trace is appended to $\mathcal{D}$ at every iteration so the repository grows monotonically. All agent roles run on pinned models, namely Claude Opus 4.8, GPT 5.6 Sol, GPT 5.4, Gemini 3.1 Pro Preview, and Gemini 2.5 Pro; the role-to-model assignments (Supplementary Table~\ref{tab:model_identifiers}) and the stages defining each family-specific workflow (Supplementary Table~\ref{tab:graft_athena_family_stages}) are detailed in the Supplementary Information. Because $w_i = \sigma(\mathrm{sim}) \cdot r_i / r_{\max}$ (Eq.~\eqref{eq:nbr_weight}) is non-negative throughout, low-reward attempts receive low reinforcement rather than explicit negative evidence, and successful strategies dominate the weighted aggregation while failed ones merely enlarge the support over which it operates.

The policy evolves across three timescales. Within a problem, the compiled prior $\pi(\cdot \mid p_{\text{new}}, \mathcal{D})$ is held fixed, and adaptation occurs through the short-term history filter, advisor corrections, and code-state updates: observed low-reward configurations enter $\mathsf{History}_n$ and are avoided on later samples from the same trial, while the advisor proposes node-level edits to $m_{n-1}$ on $\mathcal{T}_A$ from $(O_{n-1}, R_{n-1})$ at iterations $n \geq 1$. Across problems, every iteration, successful or failed, persists into $\mathcal{D}$, so the next problem's prior is built from a repository containing weakly more observations than before through the sigmoid-gated, reward-weighted blend of Eq.~\eqref{eq:policy_blend}. At the schema-growth timescale, when an agent proposes a method outside $\mathcal{T}_A$'s current support, $\mathcal{G}_A$ is grown, the spanning-tree projection regrows $\mathcal{T}_A$, and $\pi$'s domain literally expands; new nodes inherit priors from siblings, and $\mathsf{History}$ and $\mathcal{D}$ persist across the expansion.

The three Pearl ingredients are therefore scaffolded, though not yet fully exercised. The chain dependency graph $H$ on $\mathcal{G}_A$ is a policy-dependency $I$-map of the agent-induced action distribution (Proposition~\ref{prop:imap_agent}), the operator $\mathsf{Exec}$ is the structural map carrying configured code states to observations, and $(\mathcal{D}, \mathsf{History}_n)$ together with the per-trial trace $(A_n, O_n, R_n)$ provides the evidence on which future retrospective queries would condition. A query of the form ``had the agent descended the alternative branch at decision node $v \in \mathcal{T}_A$, what would the outcome have been?'' is therefore formulable in principle on this scaffold via the standard three-step recipe (abduction, action, prediction)~\cite{pearl2018book,pearl2019seven}: it would require abducting the relevant exogenous state from a sufficiently rich recorded trace, intervening on $\mathcal{T}_A$ by forcing the alternative branch at $v$, and predicting by re-evaluating $\mathsf{Exec}$ under that abducted state. Notice that the Pearl machinery rests on $\mathcal{G}_A$ alone: the problem graph $\mathcal{G}_P$ enters the construction as a structured descriptor space whose fingerprint $F$ (Eq.~\eqref{eq:fingerprint}) identifies the conditioning event $p_{\text{new}}$, but its taxonomic edges encode vocabulary relationships rather than mechanistic ones, and the system never intervenes on a problem descriptor. In this work we exercise the scaffold at Levels~1 and~2, namely associational priors (Eqs.~\eqref{eq:policy_data}--\eqref{eq:policy_blend}) and the interventional closed loop above; full counterfactual identification would require richer trace logging of proposal state, solver seeds, implementation edits, and other exogenous variables, and is left to follow-up.

\subsection{GRAFT enables informed rewards}
\label{sec:informed_rewards}

GRAFT supplies not only an action space but also problem-conditioned standards for evaluating an action. Retrieved records provide expected errors, elapsed times, reference methods, and hardware baselines; the action graph and rule set provide admissibility conditions; and $\mathsf{Exec}$ supplies measured outputs. The reward in Eq.~\eqref{eq:bandit_step} is therefore evidence-conditioned rather than a free-form language-model grade. For the three computational families ($f \in \{\mathrm{PIML},\mathrm{Numerical},\mathrm{DPD}\}$), it has the common $100$-point envelope
\begin{equation}\label{eq:computational_reward}
R_f = r_f^{\mathrm{acc}} + r_f^{\mathrm{int}} + r_f^{\mathrm{det}} + r_f^{\mathrm{eff}} + r_f^{\mathrm{opt}},
\end{equation}
corresponding to accuracy, integrity, diagnostic detail, efficiency, and optimality, with respective maxima $25$, $25$, $15$, $20$, and $15$. Accuracy is computed from family-specific numerical observables. Integrity and diagnostic detail combine thresholded failure tests with an agent's rubric-constrained synthesis of logs, fields, and figures. Efficiency begins from the measured-to-expected wall-time ratio and the recorded baseline/current methods and hardware; the ratio is deterministic, while the Advisor team attributes explained cost and penalizes unjustified over-resolution or over-engineering. Optimality audits the selected path against the admissible alternatives in $\mathcal{T}_A$. Thus the language model interprets and attributes evidence, but does not invent the measured quantities or override a failed gate.

For PIML, the primary validation observable is the relative $L^2$ error
\begin{equation}\label{eq:RL2}
\epsilon_{RL_2}
=\frac{\|\hat{u}_{\theta}-u\|_2}{\|u\|_2}
=\frac{\sqrt{\sum_{i=1}^{n}\left(\hat{u}_{\theta}(\mathbf{x}_i)-u(\mathbf{x}_i)\right)^2}}{\sqrt{\sum_{i=1}^{n}u(\mathbf{x}_i)^2}},
\end{equation}
where $u$ is the reference field and $\hat{u}_{\theta}$ is the prediction. We evaluate it on a fixed dense grid of approximately $2\times10^5$ points over the full domain, using the exact Cole--Hopf solution evaluated numerically for viscous Burgers~\cite{cole1951quasilinear,hopf1950partial} and analytical solutions for Helmholtz, Poisson, KdV, and inviscid Burgers. Observed validation errors and training losses are mapped to $r_f^{\mathrm{acc}}\in[0,25]$ relative to expected and problem-dependent worst-case anchors, with validation errors carrying $0.85$ of the weight and losses $0.15$. For Numerical problems, the current deterministic grader assigns $r_f^{\mathrm{acc}}=25$ when the measured entropy change is non-positive and $0$ otherwise. For DPD, it assigns full accuracy when all required physical and rheological quantities are present and finite with zero detected NaNs, zero for a non-finite required value or nonzero NaN count, and proportional credit when clean required quantities are missing.

Formal verification combines two independent audits with a Lean kernel check. Before implementation, a proof critic checks the complete mathematical sketch against the requested theorem, including its cited results, assumptions, lemma obligations, and final assembly. The accepted plan is then implemented in Lean 4~\cite{moura2021lean}; elaboration and kernel checking certify that every supporting lemma and the main theorem close under their declared types, while source scans for \texttt{sorry}, \texttt{admit}, \texttt{unsafe}, and local \texttt{axiom} declarations, together with the reported axiom frontier, provide deterministic evidence about unresolved proof obligations and assumptions. A fresh final auditor compares the completed Lean artifact with the original user request, checking that the proved statement is faithful and that no assumption or axiom trivializes or silently strengthens the result. This final audit uses a separate $100$-point rubric: faithfulness ($30$), reasonable assumptions ($20$), a \texttt{sorry}-free proof ($40$), and proof efficiency ($10$). Lean checking is therefore necessary but not sufficient. In every family, acceptance requires both the score threshold and the family gate,
\begin{equation}\label{eq:reward_gate}
\mathrm{Accept}_f \quad\Longleftrightarrow\quad R_f \geq \tau_f \ \land\ G_f(O_n)=1,
\end{equation}
where $G_f$ checks conditions such as successful execution, finite outputs, entropy stability, physical tolerances, or a faithful and \texttt{sorry}-free Lean proof. Ordinarily $\tau_f=90$ for PIML, $85$ for Numerical and DPD, and $91$ for Formal; PIML admits the documented $85$-point exception only after a sustained performance plateau. The complete family-specific gates and re-entry paths are given in the SI. Each observation, component score, gate outcome, and final reward is retained with $(p,m)$ in $\mathcal{D}$, so subsequent priors learn from measured accuracy, validity, and efficiency rather than from method identity alone.

\subsection{Agent teams in the GRAFT-guided workflows}\label{sec:graft-athena_teams}
\subsubsection{Formalization team}
\label{sec:formalization_team}

The performance of a GRAFT-guided workflow depends heavily on the quality of the problem statement, so problems that are defined properly are strongly preferred~\cite{roggeveen2025hardmath2}. To address this gap, the first team of the pipeline is the Formalization team (Fig.~\ref{fig:r1_system_overview}C), which interacts with the user and is composed of four agents acting in sequence. The first is a formalization agent that extracts the information needed to specify the problem (e.g., the specific PDE, the boundary and initial conditions, whether the problem is inverse, whether data are available and whether they carry noise, and the acceptable assumptions), and defines it in a form suitable for the downstream spine.

Based on this information, a subsequent analysis takes place. A second agent reads the formalized user request and looks for two kinds of features, namely exact solutions and simplifications, and proposes several candidate alternatives in each category. A third agent then ranks these proposals against a set of constraints, namely dimension reduction, nonlinear-term reduction, regularity constraints, the cost of the boundary conditions, the implementation cost, and composability, while interacting with the user to surface a modified user request whenever a change is needed. The two agents iterate as a proposer--critic pair until they converge on a winning candidate. Notice that not every problem can or should be simplified, so the ranker may also conclude that none of the alternatives is suitable, in which case the original statement is forwarded unchanged. Once a winner has been selected, the ranker rederives it step by step and engages a fourth agent that audits each step; the two go back and forth in a second proposer--critic loop until the derivation is clean and consistent with the original statement.

After the user request has been formalized and, where applicable, simplified and cleanly rederived, the same fourth agent performs a well-posedness analysis on the resulting problem. If existence, uniqueness, or stability cannot be established, the agent interacts with the user to gather the missing information and proposes additional constraints that render the problem conditionally well-posed before it enters the encode-select-solve spine. The stage spine of each scientific family (Supplementary Table~\ref{tab:graft_athena_family_stages}) and the assignment of models to agent roles (Supplementary Table~\ref{tab:model_identifiers}) are given in the SI.

\subsubsection{Support and interaction gates}
\label{sec:support_interaction_gate}

Each GRAFT-guided workflow calibrates its autonomy according to the experience available for each target. When the closest solved problem has fingerprint similarity $J_{K^\star}<J_{\mathrm{threshold}}$ (Eq.~\eqref{eq:fingerprint_jaccard}), the resulting prior is classified as weakly supported. The system then asks the user either to solve a structurally closer problem first or to continue interactively, with proposed methods and advisor recommendations available for review. Interaction is also triggered when the target problem or a proposed method cannot be represented in the current problem or action space. In that case, the corresponding graph is expanded, the new nodes and dependencies are validated, and the problem or method is re-encoded before execution resumes. This gate regulates when accumulated experience can support autonomous execution; it does not replace the subsequent scientific evaluation of the proposed solution. The default threshold and implementation details are given in the SI.

\subsection{Problem and action-tree construction}\label{sec:scaffolding}

GRAFT operates over eight coded spaces: three problem spaces covering PDEs, dissipative particle dynamics, and formal mathematics, and five action spaces covering analytical, numerical, physics-informed machine-learning, DPD, and formal-reasoning methods. Together they contain roughly three thousand nodes, most of them leaves, nested to depths of up to thirteen. Their provenance differs from space to space. Table~\ref{tab:graft_space_inventory} reports the size of each space; the complete registries, with every node and its stable codename, will be released together with the full proof developments upon publication.

The construction and provenance of each coded space are described in the Supplementary Information (\textit{Construction of the coded spaces}).

These scaffolds are not claimed as new theory; their role is to initialize a scientifically meaningful action vocabulary and rule set, and the contribution of GRAFT is that this vocabulary becomes a factored, inspectable, expandable, and reward-calibrated substrate on which agents can search, revise, and accumulate experience. The complete vocabulary of all eight spaces, printed node by node, will be released upon publication. Constructing this substrate is a one-time investment that is not trivial. Generating the tree structures is mechanized through graph builders, wiki generation, and agent-assisted ingestion of solver documentation; the dominant cost is expert curation, namely choosing which decisions enter each space, declaring the cross-branch dependency rules, and validating the compiled admissibility tables against the intended couplings. We do not reduce this effort to a single figure because it depends strongly on the curator's familiarity with the solver and with the encoding conventions: an expert can seed a coarse space in days, learning either side can extend the same task to weeks, and seeding the solved-case memory adds encoding effort that grows with the number of curated cases.

\section*{Acknowledgements}
We thank Daniel T. Chen from Brown University for engaging in insightful discussions on tree-structured policies and Markov decision processes.

\section*{Declarations}

\paragraph*{Funding}
\begin{sloppypar}
We acknowledge the support of the NIH grants R01AT012312 and R01HL154150, MURI/AFOSR FA9550-20-1-0358 project, the DOE-MMICS SEA-CROGS DE-SC0023191 award, and the ONR Vannevar Bush Faculty Fellowship (N00014-22-1-2795).
\end{sloppypar}

\paragraph*{Author contributions}
\begin{enumerate}
    \item \textbf{Conceptualization:} J.D.T., G.E.K.
    \item \textbf{Methodology:} J.D.T.
    \item \textbf{Software:} J.D.T., Z.C.
    \item \textbf{Formal analysis:} J.D.T.
    \item \textbf{Investigation:} J.D.T., Z.C.
    \item \textbf{Resources:} G.E.K.
    \item \textbf{Writing -- original draft:} J.D.T., Z.C., G.E.K.
    \item \textbf{Writing -- review \& editing:} J.D.T., Z.C., G.E.K.
    \item \textbf{Visualization:} J.D.T., Z.C.
    \item \textbf{Supervision:} G.E.K.
    \item \textbf{Project administration:} G.E.K.
    \item \textbf{Funding acquisition:} G.E.K.
\end{enumerate}

\paragraph*{Competing interests}
The authors declare no competing interests.

\paragraph*{Data availability}
All data needed to evaluate and reproduce the results in the paper are present in the paper and the Supplementary Information. The perivascular inputs derive from the published two-photon dataset~\cite{boster2023artificial,toscano2024inferring_AIV}, and the Apollo geometry and aerodynamic reference values derive from AEDC-TR-67-238~\cite{griffith1968postflight}.

\paragraph*{Code availability}
All code needed to evaluate and reproduce the results in the paper will be available upon publication.

\paragraph*{Materials availability}
This study did not generate new materials.

\bibliographystyle{elsarticle-num}
\bibliography{sn-bibliography}
\newpage
\appendix
\counterwithout{equation}{section}

\renewcommand{\figurename}{Supplementary Fig.}
\renewcommand{\thefigure}{\arabic{figure}}
\renewcommand{\tablename}{Supplementary Table}
\renewcommand{\thetable}{\arabic{table}}
\renewcommand{\theequation}{\arabic{equation}}
\renewcommand{\thetheorem}{S\arabic{theorem}}
\setcounter{figure}{0}
\setcounter{table}{0}
\setcounter{equation}{0}
\setcounter{theorem}{0}
\renewcommand{\theHequation}{S\arabic{equation}}
\renewcommand{\theHtheorem}{S\arabic{theorem}}
\renewcommand{\theHfigure}{S\arabic{figure}}
\renewcommand{\theHtable}{S\arabic{table}}

\section{Mathematical foundations of GRAFT}
\label{sec:formalization_details}

\subsection{Policy parameter-footprint computation}
\label{sec:policy_parameter_footprint}

For each method family, we counted the number of decision chains $k=|\mathcal{C}|$ and the size of each chain alphabet $|\mathcal{A}_j|$ directly from $\mathcal{T}_A$. The total and mean numbers of choices are $\sum_j|\mathcal{A}_j|$ and $\sum_j|\mathcal{A}_j|/k$, respectively. We then recovered each chain's incoming dependencies $\mathrm{pa}_H(j)$ from $H=(\mathcal{C},E\cup E_\rho)$ and computed the complete dense conditional footprint as $\sum_{j=1}^{k}|\mathcal{A}_j|\prod_{i\in\mathrm{pa}_H(j)}|\mathcal{A}_i|$. Table~\ref{tab:policy_footprint_anatomy} reports the complete method-space counts and, for each family, the chain producing the largest dense conditional table. The learnable categorical entries count stored probabilities only; the rules and hints are symbolic objects whose authoring and storage are not measured in float32 entries, so the comparison quantifies learnable probability storage, not the total knowledge-engineering footprint of the graphs.

\subsection{$I$-map certification of the decision-level factorization}
\label{sec:imap_certification}

This supplementary section recasts the decision-level factorization of Eq.~\eqref{eq:bn_factorisation} in Bayesian-network vocabulary and certifies it as an $I$-map of the joint distribution over methods, formalizing the claim made informally in Methods after Eq.~\eqref{eq:bn_factorisation} that the decomposition preserves the encoded conditional-independence assertions used by the factorization. Throughout, the problem $p$ and the repository $\mathcal{D}$ are held fixed: the rows $M^{(j)}(u, \cdot)$ defined in Eq.~\eqref{eq:policy_blend} are compiled, and Proposition~\ref{prop:imap} is a statement about the resulting conditional distribution $\pi(\cdot \mid p, \mathcal{D})$, not the marginal over a varying $p$. Proposition~\ref{prop:imap_agent} is in addition a per-step statement at iteration $n$, with the within-trial history $\mathsf{History}_n$ held fixed alongside $p$ and $\mathcal{D}$, so the I-map applies to $P_{\text{agent}}(\cdot \mid p, \mathcal{D}, \mathsf{History}_n)$ for each $n$.

The variable set is the chain set $\mathcal{C} = \{C_1, \dots, C_{|\mathcal{C}|}\}$ defined in Methods, with chain $j$ realized as $a_j \in \mathcal{A}_j \cup \{\varnothing_j\}$. The chain dependency graph $H = (\mathcal{C},\, E \cup E_\rho)$ defined in Methods carries two edge families: the rule-induced edges $E$, where each rule $R_\ell \in \mathcal{R}$ contributes for every value $t \in T^{\mathrm{trig}}_\ell$ a directed edge from $\nu(t)$ to the affected chain $j(\ell)$, and the structural-nesting edges $E_\rho = \{\rho(C) \to C : C \in \mathcal{C},\ \rho(C) \neq \bot\}$, which encode that a subchain rooted below a $c$-node is unreachable until that $c$-parent's $s$-decision is made. Its combined parent set is
\begin{equation}\label{eq:pa_H_app}
\mathrm{pa}_H(j) \;=\; \{\, i \,:\, i \to j \in E \cup E_\rho \,\},
\end{equation}
and we write $a_{\mathrm{pa}_H(j)} = (a_i : i \in \mathrm{pa}_H(j))$. The level assignment defined in Methods is the longest-path length to $j$ in $H$, with level~$0$ for chains with no parent in $H$. Acyclicity of $H$ is a precondition: it is decided in $O(|\mathcal{C}| + |E \cup E_\rho|)$ time at build time by Tarjan's strongly-connected-components algorithm~\cite{tarjan1972} (Algorithm~\ref{alg:graft_build}); a non-singleton component returns a structural cycle witness for the build-time abort message.

\begin{algorithm}[H]
\DontPrintSemicolon
\KwIn{action knowledge graph $\mathcal{G}_A$, dependency rule set $\mathcal{R}$.}
\KwOut{chain set $\mathcal{C}$, level assignment $\mathrm{level} : \mathcal{C} \to \mathbb{N}$.}
\BlankLine
$\mathcal{T} \gets \textsc{Reduce}(\mathcal{G}_A)$\;
$(\mathcal{C}, \nu, \rho) \gets \textsc{Chains}(\mathcal{T})$\;
$E \gets \bigl\{\, \nu(t) \to \nu(g) \,:\, R_\ell = (h_\ell, T^{\mathrm{trig}}_\ell, T^{\mathrm{tgt}}_\ell, e_\ell) \in \mathcal{R},\ (t, g) \in T^{\mathrm{trig}}_\ell \times T^{\mathrm{tgt}}_\ell,\ \nu(t) \neq \nu(g) \,\bigr\}$\;
$E_\rho \gets \{\, \rho(C) \to C \,:\, C \in \mathcal{C},\ \rho(C) \neq \bot \,\}$\;
$H \gets (\mathcal{C},\ E \cup E_\rho)$\;
\BlankLine
$\mathcal{S} \gets \textsc{Tarjan-SCC}(H)$\;
\lIf{$\exists\, S \in \mathcal{S} :\ |S| > 1$}{\textbf{abort}(``cycle witness $S$ in $H$'')}
\BlankLine
$\mathrm{level}(C) \gets 0 \quad \forall\, C \in \mathcal{C}$\;
\Repeat{no change}{
  \ForEach{$C \in \mathcal{C}$ with $\rho(C) \neq \bot$}{
    $\mathrm{level}(C) \gets \max\bigl(\mathrm{level}(C),\ \mathrm{level}(\rho(C)) + 1\bigr)$\;
  }
  \ForEach{$a \to b \in E$}{
    $\mathrm{level}(b) \gets \max\bigl(\mathrm{level}(b),\ \mathrm{level}(a) + 1\bigr)$\;
  }
}
\Return $(\mathcal{C}, \mathrm{level})$\;
\caption{GRAFT build-time pipeline: action knowledge graph to factored decision levels.}
\label{alg:graft_build}
\end{algorithm}

\begin{definition}[$I$-map]\label{def:imap}
A DAG $H$ over a finite variable set $V$ is an $I$-map of a probability distribution $P$ over $V$ if every conditional independence relation read off $H$ via d-separation holds in $P$.
\end{definition}

\begin{theorem}[Verma \& Pearl, 1988~\cite{vermapearl1988causal}]\label{thm:verma_pearl}
Let $M$ be a semi-graphoid over a finite variable set $V$, $\theta$ a total order on $V$, and $L_\theta = (\theta, B)$ a stratified protocol of $M$, that is, for each $x \in V$, $B(x) \subseteq \mathrm{pred}_\theta(x)$ and $I\!\bigl(x,\,B(x),\,\mathrm{pred}_\theta(x) \setminus B(x)\bigr)$ holds in $M$. Then the DAG with parent function $B$ is an $I$-map of $M$.
\end{theorem}

\begin{proof}[Proof of Proposition~\ref{prop:imap}]
We verify the hypotheses of Theorem~\ref{thm:verma_pearl} for the variable set $V = \mathcal{C}$, total order $\theta$ given by any topological extension of the level partial order defined in Methods, and parent function $B(j) = \mathrm{pa}_H(j)$.

\textit{Acyclicity.} Tarjan's algorithm applied to $H$ certifies, in time $O(|\mathcal{C}| + |E \cup E_\rho|)$, that every strongly connected component is a singleton, hence $H$ is a DAG and the level partial order extends to a total topological order $\theta$. The hypotheses of Tarjan's theorem are met by construction: $\mathcal{C}$, $E$, and $E_\rho$ are finite; $H$ is directed; the rule-expansion filter that drops triggers and targets resolving to the same chain rules out $E$-self-loops; and $\rho(C) \neq C$ for every $C$ rules out $E_\rho$-self-loops.

\textit{Tail boundary.} By the construction of $\widetilde\pi_j$ in Eq.~\eqref{eq:edited_chain}, $\widetilde\pi_j$ is obtained from $\pi_j$ by composing operators $e_{\ell_1}, \dots, e_{\ell_r}$ for those rules with $j(\ell) = j$ whose triggers $T^{\mathrm{trig}}_\ell$ are matched by the lower-level values. Each trigger $T^{\mathrm{trig}}_\ell$ references chains in $\mathrm{pa}_H(j)$ by definition of $E$, and the structural-nesting parent (when present) is in $\mathrm{pa}_H(j)$ by definition of $E_\rho$, so $\widetilde\pi_j$ depends on $a_{<\mathrm{level}(j)}$ only through $a_{\mathrm{pa}_H(j)}$. The activity indicator $\mathrm{act}_j$ used in Eq.~\eqref{eq:bar_pi_kernel} is a function of the structural-nesting parent's pick and any $\mathcal{F}/\mathcal{Z}$ effects fired by triggers in $\mathrm{pa}_H(j)$, so it too depends on $a_{<\mathrm{level}(j)}$ only through $a_{\mathrm{pa}_H(j)}$. Both branches of $\bar\pi_j$ therefore depend on the lower-level history only through $a_{\mathrm{pa}_H(j)}$:
\begin{equation*}
\bar\pi_j\bigl(\,\cdot\, \bigm|\, a_{<\mathrm{level}(j)}\bigr) \;=\; \bar\pi_j\bigl(\,\cdot\, \bigm|\, a_{\mathrm{pa}_H(j)}\bigr).
\end{equation*}
By the chain rule along $\theta$, the joint factorizes into the per-chain kernels $\bar\pi_j(\,\cdot\,\mid a_{<\mathrm{level}(j)})$, so the kernel equality lifts to the conditional independence $a_j \perp \mathrm{pred}_\theta(j) \setminus \mathrm{pa}_H(j) \mid \mathrm{pa}_H(j)$ in the policy-induced distribution, which is the tail-boundary condition for $B(j) = \mathrm{pa}_H(j)$.

\textit{Conclusion.} The semi-graphoid axioms hold for any probability distribution; the topological order is supplied by acyclicity; the tail boundary holds by the previous step. Theorem~\ref{thm:verma_pearl} then yields that $H$ with parent function $\mathrm{pa}_H$ is an $I$-map of $\pi(\,\cdot\,\mid p, \mathcal{D})$, and substituting the parent-set conditioning into Eq.~\eqref{eq:level_factorisation} gives Eq.~\eqref{eq:bn_factorisation}.
\end{proof}

The $I$-map property is what formalizes the informal claim that the decomposition is structurally conservative: every conditional independence read off $H$ via d-separation holds in $\pi(\,\cdot\,\mid p, \mathcal{D})$, so no coupling encoded in $\mathcal{R}$ or in the nesting structure of $\mathcal{T}$ is silently dropped by the per-chain factorization. The converse direction (faithfulness, every independence in $\pi(\,\cdot\,\mid p, \mathcal{D})$ has a d-separation witness in $H$) is not claimed: $H$ may be a non-minimal $I$-map, which is the standard situation for hand-specified Bayesian networks.

\begin{proof}[Proof of Proposition~\ref{prop:imap_agent}]
We verify the hypotheses of Theorem~\ref{thm:verma_pearl} for $V = \mathcal{C}$, the same total order $\theta$ used in the proof of Proposition~\ref{prop:imap}, and parent function $B(j) = \mathrm{pa}_H(j)$. Acyclicity and the topological order are inherited from that proof; only the tail boundary changes, since the distribution is now $P_{\text{agent}}$.

\textit{Tail boundary.} The I-map is a per-step statement at iteration $n$, with $\mathsf{History}_n$ treated as fixed parameters of the policy at that step (analogous to the rule operators defined in Eqs.~\eqref{eq:zero_op} and \eqref{eq:force_op}). By Assumption~\ref{ass:hint_locality}, the proposer's policy at chain $j$ depends on $\mathrm{pred}_\theta(j)$ only through $a_{\mathrm{pa}_H(j)}$, $\mathrm{hints}_j$, and the per-chain history projection $\mathsf{History}_n^{(j)}$, so
\begin{equation}\label{eq:agent_tail_boundary}
P_{\text{agent}}\!\bigl(a_j \,\big|\, \mathrm{pred}_\theta(j),\, \mathrm{hints}_j,\, \mathsf{History}_n^{(j)}\bigr) \;=\; P_{\text{agent}}\!\bigl(a_j \,\big|\, \mathrm{pa}_H(j),\, \mathrm{hints}_j,\, \mathsf{History}_n^{(j)}\bigr).
\end{equation}
Since $\mathrm{hints}_j = h_j(j, a_{\mathrm{pa}_H(j)})$ is a deterministic function of $a_{\mathrm{pa}_H(j)}$, conditioning on $\mathrm{hints}_j$ is redundant given $a_{\mathrm{pa}_H(j)}$; and since $\mathsf{History}_n^{(j)}$ is fixed at step $n$ (a parameter of the per-step policy, not a random variable in the joint over $\mathcal{C}$), it can be absorbed into the policy specification without affecting the conditional-independence statement, so
\begin{equation}\label{eq:agent_ci}
a_j \;\perp\; \mathrm{pred}_\theta(j) \setminus \mathrm{pa}_H(j) \;\bigm|\; \mathrm{pa}_H(j)
\end{equation}
in $P_{\text{agent}}$, which is the tail-boundary condition for $B(j) = \mathrm{pa}_H(j)$. The critic clause of Assumption~\ref{ass:hint_locality} ensures that acceptance is a function only of $(a_j, a_{\mathrm{pa}_H(j)}, M^{(j)}(u, \cdot), \mathrm{hints}_j)$, so the rejection-resampling step at chain $j$ reweights $a_j$ within the parent context without introducing dependence on non-parent predecessors, preserving the conditional structure of Eq.~\eqref{eq:agent_tail_boundary}.

\textit{Conclusion.} Theorem~\ref{thm:verma_pearl} yields that $H$ with parent function $\mathrm{pa}_H$ is an $I$-map of $P_{\text{agent}}(\,\cdot\,\mid p, \mathcal{D}, \mathsf{History}_n)$ for each iteration $n$, and the chain rule under $\theta$ collapses to Eq.~\eqref{eq:bn_factorisation_agent}.
\end{proof}

\subsection{Injective partition-of-unity projection}
\label{sec:partition_of_unity}

This supplementary section supplies the recursive layout procedure (Algorithm~\ref{alg:partition_layout}) for the embedding $\Phi$ defined in Eq.~\eqref{eq:phi_embedding}, together with the proof of Proposition~\ref{prop:phi_injective}. The uniform-children structural assumption the proof relies on is stated in the body.

\begin{algorithm}[H]
\DontPrintSemicolon
\KwIn{rooted tree $\mathcal{T}$ with root $r$ and edge-type labels $\tau \in \{s, c\}$.}
\KwOut{$\Phi : V(\mathcal{T}) \to [0,1]^3$.}
\BlankLine
\textbf{procedure} $\textsc{Layout}(v,\, [x_0, x_1],\, [y_0, y_1],\, d)$\;
\Indp
$\Phi(v) \gets \bigl((x_0 + x_1)/2,\ (y_0 + y_1)/2,\ d\bigr)$\;
\lIf{$v$ has no children}{\Return}
partition children of $v$ by edge type into groups $\{C_\tau\}_\tau$\;
\ForEach{group $C_\tau$, with children sorted by name}{
  $n \gets |C_\tau|$\;
  \uIf{$n$ even}{$q \gets n$;\quad $k \gets -1$ \tcp*[h]{no skip}}
  \Else{$q \gets n + 1$;\quad $k \gets \lceil q/2 \rceil$ \tcp*[h]{skip middle}}
  $U \gets (0, 1, \dots, q - 1)$ with index $k$ removed\;
  \uIf{$\tau = c$}{
    $h \gets (y_1 - y_0)/q$\;
    \ForEach{$(i, u) \in \mathrm{zip}(U, C_\tau)$}{
      $\textsc{Layout}\bigl(u,\ [x_0, x_1],\ [y_0 + i\,h,\ y_0 + (i + 1)\,h],\ d + 1\bigr)$\;
    }
  }
  \Else(\tcp*[h]{$\tau = s$}){
    $w \gets (x_1 - x_0)/q$\;
    \ForEach{$(i, u) \in \mathrm{zip}(U, C_\tau)$}{
      $\textsc{Layout}\bigl(u,\ [x_0 + i\,w,\ x_0 + (i + 1)\,w],\ [y_0, y_1],\ d + 1\bigr)$\;
    }
  }
}
\Indm
\textbf{end procedure}\;
\BlankLine
$\textsc{Layout}(r,\ [0, 1],\ [0, 1],\ 0)$\;
$D \gets \max_{v \in V(\mathcal{T})} d(v)$\;
\lForEach{$v \in V(\mathcal{T})$}{$z(v) \gets d(v) / D$}
\Return $\Phi$\;
\caption{\textsc{PartitionOfUnityLayout}: recursive subdivision with parent-midpoint exclusion.}
\label{alg:partition_layout}
\end{algorithm}

For each node $v \in V(\mathcal{T})$, let $R(v) \subset [0,1]^2$ denote the closed rectangle owned by $v$ during the recursion of Algorithm~\ref{alg:partition_layout}: $R(\textsc{root}) = [0,1]^2$, and for each child $c$ of $v$ the rectangle $R(c)$ is the slot rectangle constructed by the recursive call to $\textsc{Layout}$. By construction $\pi_{xy}(\Phi(v)) = \mathrm{centroid}(R(v))$.

\begin{proof}[Proof of Proposition~\ref{prop:phi_injective}]
We establish three sub-claims and then dispatch by lowest common ancestor.

\textbf{(C1) Containment.} For every descendant $u$ of $v$, $R(u) \subseteq R(v)$, and $R(u)$ is a non-degenerate (positive-measure) closed rectangle. By induction on path length: each recursive step constructs $R(c)$ as a slot of $R(v)$ along the active axis with width $(b - a)/q > 0$ for $q \geq 1$, so containment and non-degeneracy are preserved.

\textbf{(C2) Sibling rectangles have disjoint interiors.} Under the uniform-children assumption, all children of $v$ are processed in a single group with a single active axis. Their slot indices $i, i' \in U$ are distinct integers, giving disjoint open slot intervals along the active axis. The inactive axis is shared, so the rectangles tile $v$'s active-axis interval without overlap; their interiors are pairwise disjoint.

\textbf{(C3) Centroid interior to own rectangle, but not to any child's rectangle.} The centroid of a non-degenerate axis-aligned rectangle lies in its topological interior; combined with (C1), $\mathrm{centroid}(R(v)) \in \mathrm{int}(R(v))$. We show further that $\mathrm{centroid}(R(v)) \notin \mathrm{int}(R(c))$ for any child $c$ of $v$. Normalize $R(v)$'s active-axis interval to $[0, 1]$, so $\mathrm{centroid}(R(v))$ has active-axis coordinate $1/2$.

\emph{Even $n$:} $q = n$ is even, slot $i$ has active-axis interval $[i/q, (i+1)/q]$, and $1/2 = (q/2)/q$ is the boundary between slots $q/2 - 1$ and $q/2$. Both slots are populated, so $1/2$ is the shared boundary of $R(c_{q/2 - 1})$ and $R(c_{q/2})$, in the interior of neither.

\emph{Odd $n$:} $q = n + 1$ is even and the skipped slot is $k = q/2$, with active-axis interval $[1/2, 1/2 + 1/q]$. The point $1/2$ is the right boundary of slot $q/2 - 1$ (populated) and the left boundary of slot $q/2$ (empty). Hence $1/2$ lies on the boundary of $R(c_{q/2 - 1})$ and outside every other child's rectangle, in the interior of none.

In both cases $\mathrm{centroid}(R(v))$ is on the boundary of at most one child's rectangle and not in any child's interior. Since for any descendant $u$ of any child $c$, $\mathrm{centroid}(R(u)) \in \mathrm{int}(R(u)) \subseteq \mathrm{int}(R(c))$ (the latter inclusion is standard for axis-aligned sub-rectangles), no descendant centroid can equal $\mathrm{centroid}(R(v))$.

\textbf{Case analysis.} Suppose, for contradiction, that $u \neq v$ with $\pi_{xy}(\Phi(u)) = \pi_{xy}(\Phi(v))$. Let $w$ be the lowest common ancestor of $u$ and $v$.

\emph{Case A: $w \in \{u, v\}$.} Without loss of generality $w = v$, so $u$ is a strict descendant of $v$. By (C3), $\mathrm{centroid}(R(u)) \in \mathrm{int}(R(c))$ for some child $c$ of $v$, while $\mathrm{centroid}(R(v))$ is not in $\mathrm{int}(R(c))$. Contradiction.

\emph{Case B: $u, v$ are in disjoint subtrees rooted at distinct children $c_u, c_v$ of $w$.} By (C2), $R(c_u)$ and $R(c_v)$ have disjoint interiors. By (C1) and the standard inclusion for axis-aligned sub-rectangles, $\mathrm{centroid}(R(u)) \in \mathrm{int}(R(c_u))$ and $\mathrm{centroid}(R(v)) \in \mathrm{int}(R(c_v))$. Disjoint open sets cannot share a point, so the two centroids differ. Contradiction.

Both cases yield a contradiction, hence $\pi_{xy} \circ \Phi$ is injective. The full embedding $\Phi$ inherits injectivity since it carries strictly more information than $\pi_{xy} \circ \Phi$. \qedhere
\end{proof}

\paragraph{Discretization for fingerprinting} The fingerprint of Eq.~\eqref{eq:fingerprint} discretizes the continuous embedding by binning $(x, y)$ at resolution $K$ and pairing with the integer depth $d$, using the boundary-clamped bin map $\beta_K$ defined in Eq.~\eqref{eq:bin_K_methods}. By Proposition~\ref{prop:phi_injective}, the continuous coordinates are pairwise distinct on the finite set $V(\mathcal{T})$, so there exists at least one finite $K$ at which $\beta_K$ is injective on $V(\mathcal{T})$. We do not claim injectivity for every $K \geq K^\star$: floor binning shifts grid lines with $K$, so injectivity at one resolution does not imply injectivity at every larger one. The role of Algorithm~\ref{alg:min_k} is therefore not to secure identity for an interval of resolutions but to choose a single operating resolution $K^\star$ at which $\beta_K$ is injective on the current $V(\mathcal{T})$ and at which the per-cell footprint is small enough to produce a usable visualization on a finite grid. We pick $K^\star$ as the smallest such $K$, with a hard cap $K_{\max} = 4096$ as a numerical safety stop. All downstream uses of fingerprints (neighbor ranking, policy aggregation) are evaluated at $K = K^\star$, where the binned fingerprint coincides with the node set $P_i \cap V_{\text{keep}}$.

\begin{algorithm}[H]
\DontPrintSemicolon
\KwIn{embedding $\Phi$, depth map $d$, cap $K_{\max}$ (implementation: $K_{\max} = 4096$).}
\KwOut{$K^\star$, the minimum resolution at which $\beta_K$ is injective on $V(\mathcal{T})$.}
\BlankLine
\For{$K \gets 1$ \KwTo $K_{\max}$}{
  $\textit{cells} \gets \emptyset$;\quad $\textit{ok} \gets \mathbf{true}$\;
  \ForEach{$v \in V(\mathcal{T})$}{
    \uIf{$\beta_K(v) \in \textit{cells}$}{$\textit{ok} \gets \mathbf{false}$;\quad \textbf{break}\;}
    \Else{$\textit{cells} \gets \textit{cells} \cup \{\beta_K(v)\}$\;}
  }
  \lIf{\textit{ok}}{\Return $K$}
}
\textbf{abort}(``no $K \leq K_{\max}$ separates all nodes'')\;
\caption{\textsc{MinKForUniqueCells}: smallest grid resolution at which the binned cells of $V(\mathcal{T})$ remain pairwise distinct.}
\label{alg:min_k}
\end{algorithm}

\noindent The cap $K_{\max} = 4096$ is a numerical safety stop, not a correctness condition: by Proposition~\ref{prop:phi_injective} the continuous coordinates of $V(\mathcal{T})$ are pairwise distinct, so some finite $K$ at which $\beta_K$ is injective exists. The cap bounds the implementation's search for the smallest such $K$; if it is hit before injectivity is achieved on the current $V(\mathcal{T})$, the tree's active-axis gaps have shrunk past the chosen grid budget. This is a visualization concern (the chosen $K_{\max}$ is too small for identity-preserving rendering on this tree), not a failure of the continuous embedding. Once $K^\star$ is fixed, the fingerprint of a problem $P_i$ in Eq.~\eqref{eq:fingerprint} reduces to $F(P_i, K^\star) = \{\beta_{K^\star}(v) : v \in P_i \cap V_{\mathrm{keep}}\}$, and this fingerprint set is the problem's identifier under the current tree version and the chosen $K^\star$ (no second-stage hashing is needed).

\subsection{Fingerprint distance is a metric}
\label{sec:fingerprint_metric}

The neighbor ranking and policy aggregation defined in Eqs.~\eqref{eq:nbr_weight}--\eqref{eq:policy_blend} use the fingerprint Jaccard $J_K$ of Eq.~\eqref{eq:fingerprint_jaccard} as a similarity. The complementary form $1 - J_K$ is in fact a true metric on the space of non-empty fingerprint sets, which is what makes distance-based notions on the substrate (nearest neighbors, balls, triangle inequality) well-defined.

\begin{proposition}[Fingerprint distance is a metric]\label{prop:fingerprint_metric}
Let $w : \mathbb{Z}^3 \to (0, \infty)$ be a positive weight function and write $w(S) = \sum_{c \in S} w(c)$ for finite $S$. For non-empty finite sets $A, B \subseteq \mathbb{Z}^3$, define
\begin{equation}\label{eq:jaccard_distance}
    d_J(A, B) = 1 - \frac{w(A \cap B)}{w(A \cup B)}.
\end{equation}
Then $d_J$ is a metric on the family $\mathcal{F}$ of non-empty finite subsets of $\mathbb{Z}^3$; with $w(c) = \omega(d(c))$ this is $1 - J_K$ of Eq.~\eqref{eq:fingerprint_jaccard}, and the unweighted case is $w \equiv 1$. Applied separately on the problem tree $\mathcal{T}_P$ and the method tree $\mathcal{T}_A$, and assuming for every realized object the kept-node encoding $P_i \mapsto P_i \cap V_{\text{keep}}$ is non-empty and injective, $d_J$ is a metric on the realized problem fingerprints in $\mathcal{D}$, on the realized method fingerprints in $\mathcal{D}$, and pulls back to a metric on the realized problems and on the realized methods. The cell-level injectivity needed for the pullback comes from $\beta_{K^\star}$ being injective on $V(\mathcal{T})$ (Algorithm~\ref{alg:min_k}); the object-level injectivity of the kept-node encoding is an additional condition, and it holds under the default $V_{\text{keep}}$ of $s$-leaves since two problems making identical $s$-decisions on every chain are operationally identical.
\end{proposition}

\begin{proof}
Non-negativity follows from $0 \leq w(A \cap B) \leq w(A \cup B)$, and $d_J(A, A) = 0$ by inspection. Identity holds since $d_J(A, B) = 0$ iff $w(A \cap B) = w(A \cup B)$ iff $A = B$, the last step by positivity of $w$. Symmetry follows from $A \cap B = B \cap A$ and $A \cup B = B \cup A$. The triangle inequality is the non-trivial axiom: rewriting $d_J(A, B) = w(A \triangle B) / w(A \cup B)$ and applying the symmetric-difference triangle inequality $w(A \triangle C) \leq w(A \triangle B) + w(B \triangle C)$ (valid for any additive set weight) together with the Steinhaus normalization argument for the measure $w$ yields $d_J(A, C) \leq d_J(A, B) + d_J(B, C)$. The detailed chain is given in~\cite{lipkus1999tanimoto}; an earlier set-based statement appears in~\cite{levandowsky1971distance}.
\end{proof}

Two consequences for the substrate, on each tree separately: balls $$\{P : d_J(F(P, K^\star), F(p_{\text{new}}, K^\star)) \leq r\}$$ around an incoming problem $p_{\text{new}}$ are well-defined neighborhoods on the problem-fingerprint subspace of $\mathcal{D}$, with an analogous construction on the method-fingerprint subspace, and the triangle inequality bounds two-hop closeness in each. The neighbor selection by $J_{K^\star}$ in Eq.~\eqref{eq:nbr_weight} therefore corresponds to ranking by a true distance, not merely by an unstructured similarity.

\begin{proposition}[Similarity invariance under additive extension]\label{prop:additive_invariance}
Let $\mathcal{T}'$ be obtained from $\mathcal{T}$ by attaching new nodes, with no existing node renamed, removed, or re-parented, and let $P_i, P_j \subseteq V(\mathcal{T})$ be fingerprints encoded on $\mathcal{T}$. Then, for any level weights $\omega$, the similarity of Eq.~\eqref{eq:fingerprint_jaccard} is unchanged by the extension: $J_{K^\star}(P_i, P_j)$ computed on $\mathcal{T}'$ at its operating resolution equals its value computed on $\mathcal{T}$.
\end{proposition}

\begin{proof}
On either tree the operating resolution makes the binning $\beta_{K^\star}$ injective on the ambient node set (Algorithm~\ref{alg:min_k}), so cells are in bijection with nodes and each cell carries the depth of its unique node. The numerator and denominator of Eq.~\eqref{eq:fingerprint_jaccard} therefore reduce to $\sum_{v \in P_i \cap P_j} \omega(d(v))$ and $\sum_{v \in P_i \cup P_j} \omega(d(v))$, sums over the node sets themselves in which the layout no longer appears. An additive extension changes neither the sets $P_i$ and $P_j$, since the added nodes belong to neither, nor the depth $d(v)$ of any pre-existing node, since no node is re-parented; both sums, and hence the similarity, are unchanged.
\end{proof}

\begin{remark}[Two compatible metrics on the substrate]
\label{rem:two_metrics}
Combining Proposition~\ref{prop:phi_injective} with Proposition~\ref{prop:fingerprint_metric} equips the substrate with two compatible metrics. The first is a point metric on individual nodes,
\begin{equation*}
    d_E(v_1, v_2) = \|\Phi(v_1) - \Phi(v_2)\|_2,
\end{equation*}
on $V(\mathcal{T})$, separating points by injectivity of $\Phi$. The second is the set metric $d_J$ of Proposition~\ref{prop:fingerprint_metric} on configurations $P \subseteq V(\mathcal{T}_P)$ or $P \subseteq V(\mathcal{T}_A)$, treated separately on each tree, with cell-level separation supplied by injectivity of $\beta_{K^\star}$ and object-level separation by the kept-node-encoding condition of Proposition~\ref{prop:fingerprint_metric}. Both inherit from the same partition-of-unity construction: $d_J$ aggregates $\beta_{K^\star}$-binned positions, which themselves come from $\Phi$. Two qualifications. First, both metrics live on the projected tree $\mathcal{T}$, not on the original DAG $\mathcal{G}$, on which the spanning-tree projection induces only a pseudo-metric (graph nodes identified by the projection have $d = 0$). Second, both are tied to the current tree version, since a substrate update relayouts $\Phi$ and reshuffles cells; the metrics are recomputed on tree edits. Stored records persist across such edits: node identities are carried by stable codenames independent of the layout (printed in full in the materials to be released upon publication), so a schema edit relayouts coordinates and cells but does not rename recorded selections, and stored fingerprints carry over as node sets. Moreover, the extension protocol is additive (nodes and rules are added, never renamed or removed), and pairwise similarities between previously encoded problems are invariant under such growth (Proposition~\ref{prop:additive_invariance}).
\end{remark}

\subsection{Policy update procedure}
\label{sec:policy_update_pseudocode}

Algorithm~\ref{alg:policy_update} is the pseudocode bundling of the policy and prior update described in Eqs.~\eqref{eq:nbr_weight}--\eqref{eq:policy_blend}: it ranks $\mathcal{D}$ by the fingerprint Jaccard $J_{K^\star}$, builds per-neighbor partial specifications $S_i$, blends them row-wise into $M_{\text{data}}$ and then with the uniform prior $\mu$ to obtain $M$, and finally applies the rule operators defined in Eqs.~\eqref{eq:zero_op} and \eqref{eq:force_op}. Every step is one of Eqs.~\eqref{eq:nbr_weight}--\eqref{eq:policy_blend}; the algorithm fixes the order in which they are evaluated.

\begin{algorithm}[H]
\DontPrintSemicolon
\KwIn{new problem $p_{\text{new}}$; dataset $\mathcal{D}$; neighbor count $N_{\text{nbr}}$; uniform prior $\mu$; rule set $\mathcal{R}$.}
\KwOut{updated rows $M^{(j)}(u, \cdot)$ for every chain $j$ and every internal $u \in \mathcal{T}_A$.}
\BlankLine
rank $\mathcal{D}$ by $\mathrm{sim}(p_{\text{new}}, \cdot) := J_{K^\star}(p_{\text{new}}, \cdot)$ on $s$-leaves; keep top $N_{\text{nbr}}$\;
\ForEach{neighbor $i$ in the top $N_{\text{nbr}}$}{
    compute the weight $w_i$ from Eq.~\eqref{eq:nbr_weight}\;
    build $S_i$ by Eq.~\eqref{eq:partial_spec} (one-hot on path, $\mu$ off path)\;
}
\ForEach{internal node $u \in \mathcal{T}_A$}{
    set $M_{\text{data}}(u, \cdot)$ by Eq.~\eqref{eq:policy_data}\;
    set $M(u, \cdot) \gets \bar W \cdot M_{\text{data}}(u, \cdot) + (1 - \bar W) \cdot \mu(u, \cdot)$ via Eq.~\eqref{eq:policy_blend}\;
    \ForEach{rule $R_\ell \in \mathcal{R}$ firing on $u$}{
        update $M(u, \cdot)$ via the operator $e_\ell \in \{\mathcal{F}, \mathcal{Z}\}$ defined in Eqs.~\eqref{eq:zero_op} and \eqref{eq:force_op}\;
    }
}
\Return $M$\;
\caption{GRAFT policy and prior update at problem arrival.}
\label{alg:policy_update}
\end{algorithm}

\section{Construction of the coded spaces}

\paragraph{Problem-space scaffold.} The problem spaces encode what a problem is rather than how it is solved, and they are organized as jointly applicable attribute axes rather than as a single classification. The PDE problem space carries fourteen such axes, spanning dimensionality, equation structure, forward or inverse nature, domain, time structure, linearity, regularity, scope, solution formulation, dimensionless regime, flow features, forcing, symmetries and invariants, and data characteristics. Its leaves record the distinctions that later govern method admissibility: hyperbolic against parabolic against elliptic character, classical against distributional solution formulation together with the entropy conditions that select admissible weak solutions, and binned Reynolds, Mach, Peclet, Strouhal, Knudsen, and Damkohler ranges. Seventeen hints fix the binning arithmetic so that a given statement lands in the same cell on repeated encodings, and fourteen cross-rules keep the dimensionless axes consistent with the physics above them, for example forcing the Mach axis to its inapplicable leaf whenever an incompressible flow model is selected. The DPD problem space adds a clinical axis over vessel class, patient profile, pathophysiology, and clinical question, and the formal problem space adds result polarity alongside analysis setting and problem kind.

Unlike the Trixi.jl and Nektar++ branches of $\mathcal{T}_A$, which are constructed by ingesting solver documentation, the PIML branch was seeded manually from the PINN/PIML and numerical-analysis literature. Its seeding fixed which action families were admitted, which dependencies were encoded as cross-rules, and which literature-backed hints were attached to each node before GRAFT projected the seed graph into $\mathcal{T}_A$. The goal is not a catalog of promptable tricks, but an action vocabulary in which representation choices, residual objectives, optimizers, PDE reformulations, and numerical discretizations enter as inspectable leaves and rules, so that the proposer--critic loop can sample, audit, and revise them rather than emit them as free text.

\paragraph{Representation scaffold.} The representation branch follows the standard approximation-theoretic view that a neural solver chooses a function class in which the PDE solution is represented~\cite{cybenko1989approximation,hornik1989multilayer,raissi2019physics,karniadakis2021physics}. This scaffold seeds leaves for MLP backbones, input embeddings, Fourier features, and hard output ansatzes that absorb boundary or initial conditions structurally, with the spectral-bias view of Fourier feature networks supplying the diagnostic that flags multi-scale targets where plain MLPs underfit high-frequency content~\cite{wang2021eigenvector}. Additional approximation families enter as optional backbone leaves rather than as load-bearing assumptions, covering Kolmogorov-Arnold networks and their smooth-superposition refinement~\cite{Kolmogorov1957,liu2024kan}, the KKAN variant designed for PIML representation~\cite{toscano2025kkans}, and expressivity and spectral-bias bounds that delimit when KAN bases recover or improve on MLP behavior~\cite{wang2024expressiveness}. Compatibility hints link periodicity, smoothness, dimensionality, and prescribed boundary data to the representations most likely to be admissible for them, with the FAIR cross-architecture comparison and a recent KAN review acting as the umbrella references against which these leaves are admitted~\cite{shukla2024comprehensive,faroughi2026kolmogorov}.

\paragraph{Residual-objective scaffold.} The loss branch treats a PIML objective as residual minimization under a chosen weighting and sampling measure on the domain. Uniform residual mean-squared error, self-adaptive weighting, NTK and gradient-balancing diagnostics, residual-based attention (RBA, in which high-residual evaluations receive larger weights or sampling probability), and vector RBA (vRBA, its multi-index extension for residuals indexed by field component, equation, mode, or time slab) live as alternatives within a single action family rather than as independent tricks~\cite{wang2022and,wang2021understanding,mcclenny2023self,wu2023comprehensive,anagnostopoulos2024residual,anagnostopoulos2025learning}. Noise-aware likelihoods enter as a parallel leaf in the same family, with a negative-log-likelihood objective that learns per-point noise variance jointly with the field and downweights observations whose residual is consistent with their noise level~\cite{toscano2024inferring_AIV}. Across these alternatives, the variational residual-adaptivity framework supplies the organizing principle: changing the sampling, weighting, or likelihood rule changes the effective norm being optimized, so each leaf carries compatibility hints that match it to stiff, localized, noisy, or multi-component residuals~\cite{toscano2026variational}.

\paragraph{Optimizer scaffold.} The optimizer branch records the update rule, curvature model, line-search or step-size policy, and training schedule as separate actions on $\mathcal{T}_A$. First-order methods such as Adam, quasi-Newton methods such as L-BFGS, and SSBroyden-style curvature-aware updates therefore occupy comparable positions in the tree rather than competing inside a single decision~\cite{kingma2014adam,liu1989limited,nocedal1999numerical,urban2024unveiling}. Hints connect stiff or ill-conditioned PINN regimes to second-order or line-search choices, and tie schedule edits and warm-restart strategies to the curvature regime in which they tend to be productive~\cite{kiyani2025optimizing,jnini2026curvature}.

\paragraph{PDE-reformulation scaffold.} The PDE branch exposes transformations of the mathematical problem itself: exact or approximate reductions, hard boundary or initial-condition ansatzes, conservation-law reparametrisations such as stream-function or vector-potential forms, continuation schedules in physical or numerical parameters, and regularizers introduced to close identifiability gaps~\cite{raissi2019physics,karniadakis2021physics,evans2022partial}. The inverse-problem lineage enters this branch as a distinct family of moves, since hidden fields are encoded directly as network outputs and inferred through the residual itself~\cite{raissi2020hidden}, with downstream extensions covering in vivo brain flow with simultaneous pressure and permeability inference~\cite{toscano2025mr} and vorticity reformulations that eliminate the pressure unknown and use theoretical profiles as sequential-training scaffolds~\cite{toscano2025aivt}. This scaffold is what allows the Formalization team (see above) to treat well-posedness, constraint enforcement, and PDE simplification as selectable, auditable actions before implementation rather than as silent prompt-level reformulations. The perivascular inverse problem (see the perivascular run record, to be released upon publication) exercises this branch when mass conservation is encoded structurally through a vector-potential formulation and an unobservable slab mode is regularized; the spectral Burgers case (see the spectral-PINN run record, to be released upon publication) exercises it when the Fourier-Galerkin reduction, hard initial-condition ansatz, parity fold, and dealiased nonlinear projection are promoted into action-tree leaves rather than left implicit in code~\cite{toscano2026variational}.

\paragraph{Numerical-method scaffold.} The numerical branch organizes solver choices by PDE character and expected solution regularity. Smooth periodic problems activate spectral and Fourier-Galerkin hints~\cite{shen2011spectral,karniadakis2005spectral,basdevant1986spectral,canuto2006spectral,boyd2001chebyshev}; conservation laws with shocks activate finite-volume and discontinuous-Galerkin branches together with flux-function, limiter, positivity-preservation, and time-integration constraints~\cite{cockburn1998runge,cockburn2012discontinuous,leveque2002finite,ranocha2021adaptive,kopriva2009implementing,hesthaven2008nodal,toro2013riemann,zhang2010positivity,wanner1996solving}; nonlinear parameterized solves activate continuation or homotopy leaves that schedule a sequence of related problems rather than a single hard one. These categories supply the compatibility rules GRAFT uses both when ingesting the Trixi.jl and Nektar++ documentation and when validating agent-proposed methods against the foundational numerical-analysis literature underlying that documentation and against the PDE class observed in the formalized problem~\cite{cantwell2015nektar,canuto2006spectral,hesthaven2008nodal,toro2013riemann,leveque2002finite}.

\paragraph{Analytical-method scaffold.} The analytical branch collects transformations that change the mathematical statement of a problem before any discretization is chosen, organized as seven jointly applicable families: coordinate and variable transformations, integral and spectral transforms, symmetry and invariance reductions, dimensional and ansatz reductions, structure-preserving reformulations, asymptotic reductions, and variational and weak reformulations. Its leaves are named classical moves rather than tunable settings, including the method of characteristics, Cole-Hopf and Miura-type substitutions, streamfunction and potential variables, Green's function and Sturm-Liouville representations, self-similar and traveling-wave ansatzes, Riemann invariants and characteristic diagonalization, conservative and entropy-variable forms, matched asymptotic expansions and homogenization, and Galerkin or least-squares weak forms~\cite{evans2022partial,karniadakis2005spectral}. This branch is deliberately flat and carries no cross-rules: the reductions it offers are admissible on the merits of the problem rather than through interaction with one another, so their admissibility is adjudicated against the problem fingerprint by the Formalization team rather than by compiled rules.

\paragraph{Numerical solver ingestion.} The numerical action space is the only branch built by ingesting package documentation rather than seeded from the literature directly. It separates a shared mesh subtree from three solver families, Nektar++~\cite{cantwell2015nektar}, Trixi.jl~\cite{ranocha2021adaptive}, and an in-house Fourier-spectral Julia code, and it is the largest and most tightly coupled space in the system (Supplementary Table~\ref{tab:graft_space_inventory}), reaching depth eight, with on the order of $10^2$ admissibility rules compiled into more than $300$ conditional tables. Ingestion preserves provenance: each harvested hint carries its source document and a verbatim excerpt, node pages reproduce upstream constructor signatures and docstrings, and incompatibilities found at runtime are recorded with the version tested and the failure they produced. The rules encode the couplings that make a solver configuration succeed or fail as a unit, including the requirement that a compressible Navier-Stokes run select the matching diffusion operator, the exclusion of discontinuous-Galerkin advection under a continuous projection, and the incompatibility between subcell limiting and runtime adaptive refinement that decides the wake branch reported in the Apollo results. Ten further rules run from a solver choice into the mesh subtree, so that selecting a discretization partly determines the element shapes, conformity, and generator available downstream.

\paragraph{DPD scaffold.} The DPD branch encodes a particle simulation rather than a discretized field, so its families are the simulation's physical and procedural commitments: system setup, force field, integration, thermostat, boundary handling, driving, equilibration, production, output, and a clinical group covering parsing, rheological target derivation, blood-model construction, and analysis. Resolution enters as a modeling choice rather than a mesh parameter, through a coarse-graining ladder running from a single-bead plasma to a spring-network red-blood-cell membrane~\cite{pivkin2008accurate,fedosov2010multiscale}, and a dedicated subtree converts clinical quantities into reduced units, carrying vessel-level hematocrit under the Fahraeus effect, physiological shear ranges, and membrane shear moduli in SM units. Measurement is likewise an action rather than a readout, with apparent viscosity extracted through an Irving-Kirkwood stress tensor, a velocity-profile fit, or a Green-Kubo integral. Five cross-rules couple the branches, three of them carrying the clinical branch into the numerics, so that a sickled or spherocytic membrane target forbids the packaged dissipative parameters, while a Lees-Edwards driving choice excludes wall-bounded boundaries. The branch stops at the granularity of the validated package: bond, angle, and dihedral styles follow from the membrane model rather than appearing as separate leaves, and the run records document that boundary explicitly as a vocabulary gap.

\paragraph{Formal-reasoning scaffold.} The formal branch encodes proof technique rather than proof assistant: it contains no tactic vocabulary, no library-lemma index, and no proof-state representation, since those belong to the Lean stages of the formal loop. Its seven families are the spaces in which a claim is stated, the construction or approximation route, duality, error control, existence and extraction, probabilistic and concentration arguments, and proof character. Most leaves are themselves citable theorems used as proof moves, including Stone-Weierstrass density and mollification, Fourier truncation, convex-hull and relaxed greedy constructions~\cite{pisier1980remarques,Barron1993universal,Siegel2023112084}, Kolmogorov-Arnold superposition~\cite{Kolmogorov1957,Kourkova1991,schmidt2021kolmogorov}, polynomial emulation carrying Sobolev-rate error~\cite{yarotsky2017error,de2021approximation,guhring2020error}, Hahn-Banach separation and Riesz representation, Lax-Milgram and the direct method, Arzela-Ascoli and weak compactness, Banach and Schauder fixed points, and McDiarmid, Dudley, and Borel-Cantelli concentration bounds. Two features distinguish this space from the others. Proof character is a decision in its own right, separating constructive arguments that exhibit an approximant from non-constructive ones that establish existence through duality or compactness, and result polarity sits beside it, so that negative and characterization results occupy the space on the same footing as positive ones, with a dedicated obstruction family for under- and over-expressivity~\cite{petersen2021topological}. The branch carries no compiled cross-rules; its couplings to the problem space are written as advisory hints that key a technique to a problem attribute, for example selecting coercivity-based a priori control when the problem operator is uniformly coercive.

\section{Agent teams in the GRAFT-guided workflows}
\label{sec:supp_graft_athena_teams}

\begin{algorithm}[H]
\DontPrintSemicolon
\KwIn{new problem $p_{\text{new}}$; repository $\mathcal{D}$; rule set $\mathcal{R}$; budget $T$.}
\KwOut{converged method $m^\star$; updated repository $\mathcal{D}$.}
\BlankLine
build $\pi(\cdot \mid p_{\text{new}}, \mathcal{D})$ via Algorithm~\ref{alg:policy_update} (rows $M^{(j)}$, then rule operators)\;
$n \gets 0$;\quad initialize $S_0$, $\mathsf{History}_0$\;
\Repeat{convergence \textbf{or} $n = T$}{
    sample method $m_n = (a_{n,1}, \dots, a_{n,k})$ from the closed-loop policy $\Pi(\,\cdot\,\mid p_{\text{new}}, \mathcal{D}, \mathsf{History}_n, S_n)$, drawing through the rows of $\pi(\cdot \mid p_{\text{new}}, \mathcal{D})$ in topological order on $H$ with the proposer/critic described in Methods\;
    $A_n \gets m_n$\;
    $S_{n+1} \gets \mathcal{I}(A_n, S_n)$;\quad $O_n \gets \mathsf{Exec}(S_{n+1})$;\quad $R_n \gets R_f(O_n \mid p_{\text{new}},m_n,\mathcal{D},\mathcal{R})$ \tcp*[r]{Eq.~\eqref{eq:bandit_step}}
    $\mathsf{History}_{n+1} \gets \mathsf{History}_n \cup \{(A_n, O_n, R_n)\}$\;
    $\mathcal{D} \gets \mathcal{D} \cup \{(p_{\text{new}}, m_n, O_n, R_n)\}$\;
    advisor proposes node-level edits to $m_n$ from $(O_n, R_n)$\;
    $n \gets n + 1$\;
}
\Return $m^\star \gets \arg\max_{t < n} R_t$, updated $\mathcal{D}$\;
\caption{GRAFT-guided closed loop on a new problem $p_{\text{new}}$.}
\label{alg:graft_athena_cycle}
\end{algorithm}

\subsection{Family-specific realizations of the GRAFT-guided loop}
\label{sec:graft_athena_family_loops}

Algorithm~\ref{alg:graft_athena_cycle} gives the family-independent GRAFT-guided cycle: a method is proposed from the graph-informed prior, realized in an executable state, evaluated, revised from the resulting observations, and committed to the repository. The evaluated workflows implement these operations through a numbered stage spine, but a stage is active only when its operation is meaningful for the scientific family. The four realizations therefore share the same closed-loop logic while differing in their executable artifacts, evaluation criteria, and re-entry points. Table~\ref{tab:graft_athena_family_stages} summarizes these differences; the text below records only the family-specific operations rather than repeating the common notation and policy construction of Methods.

\subparagraph*{PIML loop}

For PIML, the executable state $S_n$ is the model, its physical and data constraints, and its training program. Stages 5--6 are bypassed because no external mesh is constructed. Stage 7 specifies the trainer and postprocessing requirements, Stage 8 realizes them in Python/JAX, Stage 9 performs the training run, and Stage 10 extracts predictions, residuals, validation errors, and the late-training flatness trace. Stages 11--13 combine these observations into the shared $100$-point computational envelope and diagnose optimization, boundary, spectral, and convergence failures. A run is ordinarily accepted at a score of at least $90$ with no outstanding edit; after a sustained performance plateau, a score of at least $85$ can also terminate the loop. Otherwise, the advisor changes the selected method or its implementation and Stage 14 opens the next iteration.

\subparagraph*{Numerical loop}

For Numerical problems, $S_n$ contains both the discretization and the solver, so Stages 5--6 are active: the mesh is planned, constructed, and validated before code realization. Stages 7--10 then specify the solver and postprocessor, implement and execute the numerical state, and recover the solution fields together with error, entropy, and conservation observations. The Stage-11 accuracy component is $25$ when the measured entropy change is non-positive and $0$ otherwise; Stage 12 separately diagnoses instability, non-finite states, oscillation, asymmetry, loss of positivity, and inadequate resolution. Stage 13 accepts a run at a total score of at least $85$ only when the diagnostic verdict is successful and the entropy condition is satisfied. Its advice distinguishes discretization changes, which re-enter the mesh stages, from solver changes, which re-enter planning or implementation.

\subparagraph*{DPD loop}

For DPD, $S_n$ is a particle realization consisting of the geometry, force field, integrator, thermostat, boundary treatment, forcing, equilibration schedule, and production settings. The mesh stages are omitted because the particle geometry plays the corresponding structural role. Stage 7 plans the LAMMPS simulation and its rheological and clinical analysis, Stage 8 constructs and smoke-tests the input state, and Stage 10 converts the LAMMPS thermo and Tecplot outputs into simulation-health, bulk-flow, hematocrit, and viscosity observations. In the present implementation the smoke execution is performed within Stage 8 rather than through the generic Stage-9 production monitor. Stages 11--13 score the availability and finiteness of the required observables and diagnose non-finite particles, thermostat or wall drift, inadequate equilibration, degenerate bulk fits, and hematocrit mismatch. Acceptance requires a total score of at least $85$, no detected non-finite values, and temperature within the prescribed tolerance; otherwise, Stage 14 routes the next attempt to the geometry or LAMMPS state rather than to a mesh stage.

\subparagraph*{Formal loop}

For Formal problems, $S_n$ is a Lean proof state and the family-specific path diverges at Stage 4. The five closest proved neighbors are examined as possible reductions, instances, corollaries, or sources of transferable proof technique; they are not combined into an authoritative node-wise method choice. Stages 5--6 are omitted. Stage 7 runs the first audit: a planner constructs a dependency-ordered proof sketch, and an independent critic checks its match to the requested theorem, cited results, assumptions, lemma completeness, and final assembly. After approval, Stage 7c partitions the proof into verifiable sections. Stage 8 implements those sections in Lean, where every supporting lemma must type-check before it can be used downstream and the completed development must pass the kernel check. Stage 9 then performs a fresh audit of the actual Lean source, checker outputs, and axiom report against the original user request rather than relying on prior agent reports. Stage 10 renders an accepted proof without altering it. The final audit assigns $30$ points to faithfulness, $20$ to reasonable assumptions, $40$ to the absence of \texttt{sorry}, and $10$ to efficiency. Approval requires a score of at least $91$ together with the hard conditions that the requested theorem is proved, the proof is \texttt{sorry}-free, and no trivializing assumption or unapproved axiom has been introduced. Generic accuracy and diagnostic Stages 11--12 are therefore omitted. Stage 13 extracts the proved method and audit reward, and Stage 14 either updates the Formal repository or routes a rejected attempt back to Stage 7 for a plan revision or Stage 8 for a proof repair.

\subsection{Model identifiers}
\label{sec:model_identifiers}

Each agent role was assigned a pinned model and access route (Supplementary Table~\ref{tab:model_identifiers}). Roles that plan, implement, or audit executable artifacts ran headlessly through Claude Code or Codex, whose native terminal, file-editing, code-execution, and internet-search tools supplied the execution environment. Roles requiring only a single exchange (encoding, proposer--critic sampling, score parsing, or report parsing) used direct APIs. In both cases, the external system supplied inference and generic tooling only; the system prompts, role definitions, and stage instructions belong to the evaluated GRAFT-guided workflows and will be released upon publication.

A headless invocation is an autonomous run launched programmatically by the orchestrator: no user operates the harness, and only the completed artifacts return to the pipeline. User input occurs only at the designated formalization and interaction gates between invocations, with each response incorporated into the next scripted instruction. Table~\ref{tab:model_identifiers} reports the assignments and access routes for all four scientific families; its stage numbers correspond to Supplementary Table~\ref{tab:graft_athena_family_stages}. All models were accessed between May and August 2026. Operations not listed as model-backed (including Stage 3, generic Stage-9 execution, the Numerical Stage-11 grader, and the Stage-14 repository update) were deterministic. Finally, the Claude Code roles used the same pinned \texttt{claude-opus-4-8} model as the Opus 4.8 commercial baseline.

\subsection{Support-check implementation}

The minimum similarity $J_{\mathrm{threshold}}$ is configurable and defaults to $0.50$ in the implementation. A score below this threshold blocks execution in autonomous mode but raises a warning in interactive mode. In some evaluated cases with similarity below the default threshold, the user allowed the system to continue without intervening.

\section{Seed inventory of the solved-case memory}
\label{sec:memory_seed_inventory}

\section{Contents of the frozen reproducibility archive}
\label{sec:archive_contents}

The frozen archive, to be released upon publication, is the complete working repository of the system, packaged to be self-contained. It contains, for every scientific family, the problem and method ontologies with their stable node identifiers, the declared dependency rules and node hints together with the compiled admissibility reports, and the fingerprint encodings and retrieved-neighbor lists behind the retrieval experiments. It also contains the prompts, skills, and agent role definitions for both supported harnesses, the model identifiers of Supplementary Table~\ref{tab:model_identifiers} and the stated inference settings, the software and solver versions, hardware and compute budgets, the complete run records of the audited trajectories including failed iterations, the Lean sources and axiom reports together with the supporting Lean brick library, and the scripts that regenerate the supplementary footprint, scaling, retrieval, and invariance analyses. Beyond these required materials, the archive includes the curated solver libraries and per-family implementation templates, the curated worked examples that seed each repository, and tutorial documentation covering how to add new cases and examples, how to extend the ontologies with new methods or software frameworks, and how to drive the orchestration from a different agentic harness; the released skills run under both Claude Code and Codex, and the same role definitions port to any comparable coding agent.

\section{Additional scientific results}
\label{sec:supp_additional_results}

\label{sec:supplementary_images}

\subsection{PIML benchmark results}
\label{sec:R3}

Fig.~\ref{fig:priors_and_results}A reports the relative $L^2$ error against iteration count for KdV, Helmholtz, and Poisson, while Table~\ref{tab:sota_comparison_graft} adds viscous Burgers and extends the comparison to expert-authored and agentic baselines. GRAFT+ATHENA reaches a lower relative-error floor than ATHENA~\cite{toscano2025athena} on all four benchmarks, while the residual MSE remains close to machine precision. Burgers provides the clearest signature of transferred experience (Fig.~\ref{fig:semantic_memory_learning}B): the error against the target-viscosity reference holds near $10^{-1}$ while the inherited Reynolds-number continuation works through its intermediate stages, then falls to $\sim10^{-9}$ once the final viscosity switch is reached at $\approx140{,}000$ iterations, whereas ATHENA stabilizes near $10^{-8}$. Comparable gains on KdV, Helmholtz, and Poisson show that the improvement is not confined to a single equation class.

\subsection{Hematocrit dependence of red-blood-cell suspension rheology}
\label{sec:R5}
 
We posed a clinically grounded particle-based test in which the apparent viscosities of two red-blood-cell (RBC) suspensions were recovered at two vessel-level hematocrits~\cite{chai2025silico,chai2026multiscale,chai2026reticulocyte,chai2026stomatocyte}. The two arms are driven at the same wall stress and carry the same membrane parameters, so hematocrit is the only variable and the contrast is controlled by construction. It requires a non-PDE solver, a multiscale spectrin-network membrane, and a rheological measurement extracted from particle trajectories rather than read off a field. The viscosity ordering must therefore emerge from the assembled multiscale model rather than from parameter adjustment to an expected range.
 
The configuration is recovered from the action graph in a single cascade. From the LAMMPS-backed DPD subtree of $\mathcal{G}_A$~\cite{LAMMPS}, the GRAFT-guided DPD workflow selects the multiscale RBC formulation of Pivkin and Karniadakis~\cite{pivkin2008accurate} and Fedosov and co-workers~\cite{fedosov2010multiscale}, which fixes the membrane substrate as a coarse-grained spectrin network triangulated on each cell surface~\cite{chai2022periodic,chai2023dynamics}. This selection determines the downstream membrane interactions. The substrate forces a worm-like-chain bond on every triangulation edge, an area--volume-conserving angle potential on every triangle, and a bending dihedral on every adjacent triangle pair, realized as the \texttt{break/ligand} bond style and the \texttt{area/volume} angle style, which the implementation adapter derives from the spring-network-membrane leaf of $\mathcal{T}_A$ rather than from separate leaves. A standard DPD pairwise interaction kernel closes the cell--cell and cell--solvent forces, and the assembled simulation cell is a $20\!\times\!20\!\times\!20$ LJ box with two driven wall layers (atom types 3 and 4, integrated with the fluid and carrying the wall force $\pm f_x$), a frozen outer buffer film (type 5) preventing wall-image mixing, and a core slab populated by deformable RBCs (type 1) immersed in a DPD solvent (type 2). This is the same chip-flow geometry validated previously for blood biophysics in pathological states~\cite{chai2025silico,chai2026gaucherrbc}. Both arms use the validated common baseline throughout: the same membrane coefficients, the same wall force $f_x=0.05$ (giving $\tau_{xz} = N_{\mathrm{wall}} f_x / A_{\mathrm{wall}} = 0.45$), and the same thermostat target. They differ only in how densely that slab is packed, $\phi = 0.144$ against $\phi = 0.260$. The cascade from one biophysical commitment (multiscale spectrin membrane) to the downstream LAMMPS choices (two selections forced by the tree's dependency rules and the three membrane style blocks the adapter derives from them) is traced through the action tree rather than retrieved from a configuration template.
 
The result is the rheological fingerprint required of any blood-flow study. Because the bulk shear rate of the suspension relaxes over tens of thousands of steps, each condition was continued to $300{,}000$ steps and only the final $100{,}000$ were analyzed; over the first $100{,}000$ the window-averaged shear rate is still drifting, and a viscosity extracted there depends on which windows are averaged. Each condition was also repeated across independent RNG seeds, three at the lower hematocrit and six at the higher. Within the analysis window the shear profiles of Supplementary Fig.~\ref{fig:priors_and_results}B, subpanels A and B, are linear across the bulk region $z \in [3, 17]$, so the bulk shear rate $\dot\gamma = \partial_z u_x|_{\mathrm{bulk}}$ is read off as a slope: $-1.31\times 10^{-3}$ at $\phi = 0.144$ and $-7.84\times 10^{-4}$ at $\phi = 0.260$, both at $R^2 = 0.95$ once averaged over ten non-overlapping windows per seed. The thermostat trace of subpanel C confirms that both runs settle to a constant temperature within the first few time units of the $60$-unit run and hold it thereafter, so neither viscosity reading is contaminated by thermal drift; their equilibrium means, $0.1506$ and $0.1555$, differ by approximately $3\%$. Subpanel D reports the relative viscosity $\eta/\eta_{\mathrm{solvent}}$ against a cell-free control run in the identical cell with the same walls, forcing, parameters and slope estimator (a single $100{,}000$-step run averaged over all nine of its $10{,}000$-step windows; with no cells present the window slopes show no trend in time, so no burn-in is applied): $1.35 \pm 0.07$ at $\phi = 0.144$ and $2.26 \pm 0.09$ at $\phi = 0.260$ (mean $\pm$ SEM over seeds), against the $1 + 2.5\phi + 7.35\phi^2$ expectation of $1.51$ and $2.15$ for a suspension with a blood-like second-order coefficient. Raising the hematocrit by a factor of $1.8$ raises the suspension viscosity by $1.67$ ($95\%$ CI $1.51$--$1.88$, bootstrap over seeds), and the ordering is recovered without any parameter tuned to it. We report the ratio rather than an absolute viscosity because the wall-driven cell's nominal stress $\tau_{xz} = N_{\mathrm{wall}} f_x / A_{\mathrm{wall}}$ charges the fluid for momentum that the frozen buffer film (type 5) also absorbs; that bias is a single multiplicative constant of the cell, identical across these runs, and cancels in the ratio. A second, smaller bias does not cancel exactly: the velocity grid is sparsely occupied (about $0.55$ fluid particles per bin) and unoccupied bins enter the average as zeros, and because the cells displace fluid the bulk-window occupancy is a few percent lower in the suspension arms than in the control; this inflates both relative viscosities by roughly $4\%$, within the quoted SEM, and leaves the hematocrit factor unchanged ($1.68$ when corrected). Subpanels E to H are taken from a separate forcing sweep on the same geometry. E and G show the suspension at the weakest and strongest forcing ($\tau_{xz}=0.045$ and $4.5$), and F and H illustrate the cell-scale measurement: a two-dimensional principal-axis analysis of a single periodic-unwrapped cell gives an in-plane aspect ratio $\lambda_1/\lambda_2$ independent of tank-treading orientation ($1.28$ at the weakest and $1.38$ at the strongest forcing for the cell drawn). Over all twenty cells in that box the ensemble mean moves only from $1.47 \pm 0.48$ to $1.51 \pm 0.49$ (mean $\pm$ cell-to-cell s.d., 100 cell-frames each) between $\tau_{xz} = 0.045$ and $4.5$, a difference within the cell-to-cell spread, so these panels show the deformation mode the assembled membrane admits rather than a resolved dependence on shear rate.
 
Fig.~\ref{fig:semantic_memory_learning}A illustrates retrieval for a separate Gaucher-disease capillary case; the two controlled hematocrit arms and their rheological outcomes are reported in Supplementary Fig.~\ref{fig:priors_and_results}B.

The DPD branch therefore shows that the same nearest-neighbor retrieval, policy blending, and long-term-memory update extend to a particle-based, non-PDE solver, while the resulting trajectories reproduce the rheological behavior that the model must respect~\cite{fedosov2014computational,zhang2020deep,zhang2021deep}.

\subsection{From ill-posed to in vivo: reformulation and autonomous solution of a perivascular flow inverse problem}
\label{sec:R6}

The Formalization team considered three simplifications of the perivascular inverse problem. The selected reformulation enforces conservation of mass as an exact architectural constraint through a vector-potential representation. Continuity then follows from $\nabla\cdot(\nabla\times\mathbf A)\equiv0$, and the four primitive unknowns $(u,v,w,p)$ are replaced by three network outputs $(P,R,p)$. A pressure-Poisson/Leray-projection alternative was rejected because it left the primitive-variable residuals unchanged while requiring an auxiliary elliptic solve at every collocation point.

The subsequent well-posedness audit exposed two remaining deficits. First, the additive pressure gauge $p\mapsto p+c(t)$ was removed by imposing the outlet anchor $p(t,x,1.556,z)=0$. Second, observations confined to a single slab leave initial-condition modes with a node on that slab invisible to the data residual. The user therefore selected the cross-slab smoothness penalty
\[
\mathcal{L}_{\rm reg}
=
\lambda_{\rm reg}\,
\mathbb{E}\!\left[
|\partial_z^2u|^2+
|\partial_z^2v|^2+
|\partial_z^2w|^2
\right].
\]
The audit then classified the formulation as conditionally well posed. The pressure anchor removes the additive gauge, whereas the regularizer controls the non-affine cross-slab modes; the observations and Stokes residual must still constrain the affine-in-$z$ modes in its kernel. The detailed run records will be released upon publication.

GRAFT retrieved the MRI advection--diffusion inverse problem, viscous Burgers, and inviscid Burgers as the three closest solved cases. Their similarity- and reward-weighted methods were combined with the uniform prior to form the problem-conditioned action distribution. The selected preprocessing nondimensionalized the observations and used the minimum, maximum, and mean of the observed velocity time series to identify the dominant cycle. The resulting method used a six-layer, width-$40$ field network mapping $(t,x,y,z)$ to $(P,R,p)$ and a six-layer, width-$20$ denoising network for the observation variances, giving $11{,}465$ trainable parameters. The remaining method combined negative-log-likelihood noise treatment~\cite{toscano2024inferring_AIV}, vRBA residual weighting and sampling~\cite{toscano2026variational}, adaptive activation functions~\cite{jagtap2020adaptive}, and the SOAP optimizer~\cite{vyas2024soap,wang2025gradient}. Training proceeded for $100{,}000$ iterations through data- and boundary-led initialization, coupled physics training, and refinement.

Predicted and reference particle velocities show good agreement across six held-out times (Supplementary Fig.~\ref{fig:verified_discovery_supp}C), consistent with the published reconstruction~\cite{toscano2024inferring_AIV}. These held-out particles occupy the same thin slab as the training observations. They therefore validate unseen particle motion within the measured plane, but not the complete three-dimensional field.

Fig.~\ref{fig:verified_discovery_supp}A shows the Formal problem formulation used to certify the regularized reconstruction, rather than the PIML reconstruction fingerprint shown in Fig.~\ref{fig:verified_discovery}A. The target is encoded as a positive a posteriori result for a noisy, regularized inverse problem under conditional stability. Its closest stored theorem shares the inverse-error setting, noisy observations, conditional-stability structure, and a posteriori character, while the target-specific branches encode the Hilbert-space, least-squares regularization, and reliability statement required for the perivascular certificate. The stronger agreement between the method fingerprints than between the problem fingerprints shows how an established proof pattern was transferred to a distinct theorem formulation. The theorem statement and proof sketch are given in Theorem~\ref{thm:perivascular_aposteriori}.

\subsection{Spectral PINN: an agent-designed architecture with exponential convergence}
\label{sec:spectral_pinn}

The spectral PINN originated during the Formalization analysis of the canonical viscous Burgers benchmark at $\nu=1/(100\pi)$. The proposer agent generated three simplifications, including a periodic Fourier-series mode reduction that converts the PDE into a system of modal evolution equations. The ranker subsequently introduced a multi-harmonic Fourier-feature encoding and placed it first because it required only a minimal change to a standard PINN, leaving the agent-proposed spectral reduction second. Upon inspection, the user recognized that this unselected proposal defined a potentially novel PINN architecture and asked the system to develop it in a separate, milder test at $\nu=1/100$. This change reduced both the implementation risk and the required modal resolution without altering the underlying derivation. On reconsideration, the ranker identified the advantages missed by its original rubric: viscous damping becomes closed form, the initial condition can be imposed exactly, and spatial differentiation reduces to spectral multiplication rather than automatic differentiation.

The selected reformulation expands
\[
u(t,x)=
\sum_{n=1}^{N}
\left[
a_n(t)\cos(n\pi x)+b_n(t)\sin(n\pi x)
\right],
\]
with
\[
a_n(t)=t\widetilde a_n(t),
\qquad
b_n(t)=b_n^{\rm IC}+t\widetilde b_n(t),
\qquad
b_1^{\rm IC}=-1.
\]
This representation enforces the initial condition structurally. The nonlinearity $uu_x$ is evaluated pseudospectrally on a dealiasing grid of $M$ points and projected onto the $N$ retained modes. Review of the first derivation identified a sign error that would have produced anti-diffusive modal growth and strengthened the dealiasing condition from $M\geq3N$ to $M>3N$.

The mode-indexed residual also required adapting the residual-weighting strategy. Standard adaptive weighting methods are formulated for pointwise spatial residuals~\cite{anagnostopoulos2024residual,wang2022and}; here, vRBA was indexed jointly by Fourier mode $n$ and time point $t_i$. The tensor $\lambda\in\mathbb{R}^{N\times N_t}$ is therefore allocated on a $(\mathrm{mode},\mathrm{time})$ grid rather than the usual $(\mathrm{point},\mathrm{time})$ grid, and it drives time sampling through
\[
p(t_i)\propto\sum_n\lambda_{n,i}^2.
\]
During Adam, $\lambda$ provides both weighting and sampling. During SSBroyden, it is retained only for sampling because nonstationary loss multipliers would invalidate the quasi-Newton inverse-Hessian estimate.

The well-posedness audit established the energy inequality
\[
\frac{1}{2}\frac{d}{dt}\|u\|_2^2
=
-\nu\|u_x\|_2^2
\leq0.
\]
It also observed that the initial condition is odd and that the equation preserves odd parity. Uniqueness therefore implies $u(t,-x)=-u(t,x)$ and $a_n(t)\equiv0$, eliminating the cosine modes and halving the network output. A remaining smoothness condition at $t=0^+$ was closed by using $\tanh$ activations.

The first production run reached a relative $L^2$ field error of $1.107\times10^{-3}$ while the retained-mode PDE loss saturated at approximately $8.3\times10^{-13}$. The diagnostic agent attributed this difference to spectral truncation rather than failed optimization: the Galerkin loss measures only the retained modes, whereas the field error also contains the unresolved tail. The cutoff energy, $|b_{48}|^2\approx10^{-7}$, supported this diagnosis. The advisor therefore changed only the modal and dealiasing resolutions from $(N,M)=(48,192)$ to $(96,320)$ without revising the action-tree choices. The second iteration reduced the field error to $6.598\times10^{-6}$. The detailed implementation and diagnostic records will be released upon publication.

Fig.~\ref{fig:verified_discovery_supp}D shows the dense-grid prediction, pointwise field error, per-mode Galerkin-residual decomposition, and converged per-mode vRBA tensor for the best-resolved run. Together with the mode-count sweep in Supplementary Table~\ref{tab:spectral_mode_sweep}, these diagnostics show that the error decreases as the mode count grows rather than flattening at an optimization-imposed floor. The observed behavior rests on the quasi-Newton SSBroyden optimization and the adaptation of vRBA to the $(\mathrm{mode},\mathrm{time})$ tensor. The corresponding a priori result is stated, with its proof sketch, in Theorem~\ref{thm:spectral_apriori}.

\section{Machine-checked scientific theorems}
\label{sec:formal_universality_results}

The following theorems were proved and machine-checked in Lean 4. For each result, we state
the theorem and provide a proof sketch of the main argument. The complete formal statements
and proofs will be released upon publication as Lean 4 developments, together with
the checker outputs, reported axiom frontiers, and full human-readable transcriptions.

The Formal protocol retrieves five neighbors for each case; each result below names the five members of that neighbor set in rank order.

\subsection{Universality of fixed-degree Chebyshev Kolmogorov--Arnold networks}
\label{sec:ckan_universality_result}

The five closest solved problems were KKAN universality, K\r{u}rkov\'a's density theorem, the fixed-width realization obstruction, the sharp shallow-network universal approximation theorem, and the Kolmogorov superposition theorem.

\subparagraph*{User request.}
\noindent\emph{Statement to formalize.} A Kolmogorov--Arnold network places a trainable
univariate function on every edge of the network graph and adds the incoming values at the
nodes~\cite{liu2024kan,Kolmogorov1957}; a Chebyshev network~\cite{ss2024chebyshev} fixes
the shape of those univariate functions once and for all. Fix a degree $k\ge1$ and let
$T_0,\dots,T_k$ be the Chebyshev polynomials of the first kind ($T_0=1$, $T_1=t$,
$T_{n+2}=2t\,T_{n+1}-T_n$). A cKAN edge function of degree $k$ is
\[
  \psi(z)=\sum_{n=0}^{k}c_n\,T_n(\tanh z),\qquad c_0,\dots,c_k\in\mathbb{R},
\]
a finite Chebyshev expansion evaluated after the hyperbolic tangent. The $\tanh$ lands the
argument in $(-1,1)$, so the polynomials are always evaluated inside their interval of
definition, and the only freedom left on an edge is its coefficient sequence. Every edge of
every layer has this same uniform form. The request stipulates $k\ge1$: at $k=0$ every edge
is the constant $c_0T_0=c_0$ and every network realizes a constant map.

A depth-$L$ nested cKAN ($L\ge2$) with layer widths $n_0=d,\,n_1,\dots,n_{L-1},\,n_L=m$ is
the map $\mathrm{cKAN}:[-1,1]^{d}\to\mathbb{R}^{m}$ defined by
\[
  z^{(0)}=\mathbf{x},\qquad
  z^{(\ell)}_j=\sum_{p=1}^{n_{\ell-1}}\psi^{(\ell)}_{j,p}\bigl(z^{(\ell-1)}_p\bigr),
  \quad \ell=1,\dots,L,\qquad \mathbf{y}=z^{(L)},
\]
where each $\psi^{(\ell)}_{j,p}$ is a cKAN edge function with its own coefficients.
Constraining every edge to this single uniform shape is a genuine restriction, and the
request asks whether networks assembled from such edges alone can still approximate an
arbitrary continuous target.

\medskip
\noindent\emph{Prove:} the nested cKAN family is a universal approximator of
$C([-1,1]^{d};\mathbb{R}^{m})$ in the sup norm: for every continuous $F$ and every
$\varepsilon>0$ there exist a depth $L\ge2$, widths $n_1,\dots,n_{L-1}$, and coefficients
such that
\[
  \sup_{\mathbf{x}\in[-1,1]^{d}}\bigl\lVert F(\mathbf{x})-\mathrm{cKAN}(\mathbf{x})\bigr\rVert<\varepsilon .
\]
The request stipulates $L\ge2$: a single edge layer computes a sum of univariate functions
of the input coordinates, an additive model with no interaction terms.
\medskip

\noindent\emph{Source.} Shukla et al.~\cite{shukla2024comprehensive}.

\subparagraph*{The theorem obtained.}
Below the input cube $[-1,1]^{d}$ is written $Q_{d}$, and on $\mathbb{R}^{m}$ the norm
$\lVert v\rVert$ is the supremum of the coordinate magnitudes, which is the largest of
the values $\lvert v_q\rvert$ when $m\ge1$ and equals $0$ when $m=0$, the space
$\mathbb{R}^{0}$ having a single point. Three terms of the statement deserve a plain
gloss. A finite hidden index set $H$ is an abstract finite set whose elements label the
hidden nodes. A vector network $Y$ of depth two carries two layers of edges: one
first-layer edge $Y^{\mathrm{in}}_{h,p}$ for every hidden index $h\in H$ and every input
coordinate $p\in\{1,\dots,d\}$, and one second-layer edge $Y^{\mathrm{out}}_{q,h}$ for
every output coordinate $q\in\{1,\dots,m\}$ and every $h\in H$. Its realization
$\mathcal{R}_{Y}$ is the map $\mathbb{R}^{d}\to\mathbb{R}^{m}$ the network computes by
summing at the nodes,
\[
  \mathcal{R}_{Y}(x)_q=\sum_{h\in H} Y^{\mathrm{out}}_{q,h}
    \Bigl(\sum_{p=1}^{d} Y^{\mathrm{in}}_{h,p}(x_p)\Bigr),
  \qquad q=1,\dots,m .
\]

\begin{theorem}[Fixed-degree deep cKAN universality]\label{thm:ckan_universality}
Let $k\ge1$ be a degree, let $d$ and $m$ be non-negative integers, let
$F\in C(Q_{d};\mathbb{R}^{m})$, and let $\varepsilon>0$. Then there exist a depth $L\ge2$, a
finite hidden index set $H$, and a vector network $Y$ of depth two, degree $k$, input
dimension $d$, output dimension $m$ and hidden index set $H$, every edge of which is a cKAN
edge function of degree $k$, such that
\[
  \lVert F(x)-\mathcal{R}_{Y}(x)\rVert<\varepsilon\qquad\text{for every }x\in Q_{d} .
\]
\end{theorem}

First, the depth is quantified existentially, and the network constructed in the checked proof has
depth exactly two in every case (the proof takes $L=2$; no construction at any larger depth
occurs anywhere in the development, so nothing is proved about deep networks specifically,
and nothing needs to be). A construction at the minimum depth the request permits settles
the stated existence in its strongest available form. Second, the degree is arbitrary and
fixed in advance: universality holds for every fixed $k\ge1$, and the approximation burden
is carried entirely by the width, that is, by the size of the hidden index set $H$. What
the theorem does not provide is any control on that width: it is a density statement, with
no bound on $H$ and no rate in $\varepsilon$. The degenerate dimensions are included rather
than excluded: the main construction assumes $d\ge1$ and $m\ge1$, and the cases $m=0$ and
$d=0$ are handled separately and hold exactly.

\subparagraph*{Assumptions.}
Two statements are assumed rather than proved. They are the two implications from which the
universal approximation theorem for a continuous non-polynomial activation is obtained,
through the standard duality criterion: the linear span of a set of continuous
functions is dense exactly when the zero functional is the only continuous linear
functional vanishing on the whole set. For a compact $K\subseteq\mathbb{R}^{d'}$ and a
continuous $\rho:\mathbb{R}\to\mathbb{R}$, write $\mathcal{N}_\rho(K)$ for the set of
single neurons $u\mapsto\rho\bigl(\sum_{p=1}^{d'}w_pu_p+b\bigr)$ on $K$, one for each
weight vector $w\in\mathbb{R}^{d'}$ and bias $b\in\mathbb{R}$.

\begin{enumerate}[label=(A\arabic*),leftmargin=*,itemsep=0.5ex]
\item \emph{Annihilation forces a polynomial activation.} If some non-zero continuous
  linear functional on $C(K;\mathbb{R})$ vanishes on every element of
  $\mathcal{N}_\rho(K)$, then $\rho$ agrees everywhere with the evaluation of a real
  polynomial.
\item \emph{Polynomial activations are obstructed.} If $K$ has non-empty interior and
  $\rho$ agrees everywhere with a real polynomial, then the closure of the linear span of
  $\mathcal{N}_\rho(K)$ is a proper subspace of $C(K;\mathbb{R})$.
\end{enumerate}

Together these give the theorem of Leshno, Lin, Pinkus and Schocken~\cite{Leshno1993},
which specializes to the classical results of
\cite{cybenko1989approximation,hornik1989multilayer} when the activation is additionally
required to be sigmoidal; only the implication from non-polynomiality to density is used
below. The development has been checked in full by a proof assistant. Its logical
frontier, the complete list of what the verification takes on trust, is exactly the three
standard foundational principles of the underlying logic, namely propositional
extensionality, the soundness of quotients, and the axiom of choice, together with
assumptions (A1) and (A2), through which the quoted classical theorem enters in exactly
the strength applied, and nothing else. The two assumptions reach the argument at exactly
one point, the finite-sum approximation of stage two, and through nothing else.

\subparagraph*{Proof sketch.}
Throughout the sketch $d\ge1$ and $m\ge1$; the degenerate cases $m=0$ and $d=0$ are
dispatched at the end of stage four. The entire error budget is one tolerance
$\delta=\varepsilon/2$, spent once in stage two; stages three and four are exact.

\emph{Stage one: the edge class contains every affine function of $\tanh$.} The map
$z\mapsto a\tanh z+b$ is the cKAN edge whose coefficients are $c_0=b$ and $c_1=a$ and
vanish at every other index, since $T_0=1$ and $T_1(t)=t$; both indices lie in
$\{0,\dots,k\}$ precisely because $k\ge1$. Every edge built below is produced through this
one observation, including the zero edges of stage four, so the hypothesis $k\ge1$ enters
the proof only through it. What the observation buys is a containment: a depth-two network
of such edges contains an ordinary one-hidden-layer $\tanh$ network applied after
coordinatewise $\tanh$ preprocessing of the input, that is, the functions
\[
  \beta+\sum_{i=1}^{N}A_i\tanh\Bigl(\sum_{p=1}^{d}w_{i,p}\tanh x_p+b_i\Bigr),
\]
with the inner $\tanh x_p$ present because an input coordinate reaches a hidden node only
through a first-layer edge, whose own $\tanh$ transforms it first. The single uniform edge
shape therefore costs nothing at depth two, provided the classical approximation theorem
is applied on the preprocessed domain; arranging exactly that is the job of stage two.

\emph{Stage two: a finite neuron sum for each output coordinate.} Coordinatewise $\tanh$
carries the input cube $Q_{d}$ into the transformed cube $\Omega_{d}=[-\tau,\tau]^{d}$,
where $\tau=\tanh 1$. The strict inequality $\tau<1$ is load-bearing here: it places every
coordinate of every point of $\Omega_{d}$ inside the open interval $(-1,1)$ on which the
inverse hyperbolic tangent is defined, monotone and continuous, so the coordinatewise
inverse carries $\Omega_{d}$ back into $Q_{d}$ continuously. Composing this inverse with
$F$ pulls each output coordinate $q$ back to a continuous scalar function
$G\in C(\Omega_{d};\mathbb{R})$ with $G(\tanh x)=F(x)_q$ for every $x\in Q_{d}$, writing
$\tanh x$ for the point with coordinates $\tanh x_p$. Since $\Omega_{d}$ is compact with
non-empty interior and $\tanh$ is continuous and not a polynomial function, the universal
approximation theorem obtained from (A1) and (A2) applies on $\Omega_{d}$ and yields, for
the tolerance $\delta=\varepsilon/2$, a width $N\ge1$, a constant $\beta$, output weights
$A_1,\dots,A_N$, input weights $w_{i,p}$ for $i=1,\dots,N$ and $p=1,\dots,d$, and biases
$b_1,\dots,b_N$ such that
\[
  \Bigl\lvert G(u)-\Bigl(\beta+\sum_{i=1}^{N}A_i\tanh\Bigl(\sum_{p=1}^{d}w_{i,p}u_p+b_i\Bigr)
    \Bigr)\Bigr\rvert<\delta\qquad\text{for every }u\in\Omega_{d}.
\]
This is the only point of the development at which the assumptions are used, and the
tolerance $\delta$ is the entire error the proof ever incurs.

\emph{Stage three: the sum is realized exactly.} By stage one the first-layer edge from
input coordinate $p$ into hidden node $i$ may be taken to be
$z\mapsto w_{i,p}\tanh z+b_i/d$ and the second-layer edge out of node $i$ to be
$z\mapsto A_i\tanh z+\beta/N$. The two additive constants have nowhere else to live, since
a node adds incoming edge values and contributes nothing of its own, so each is
distributed by division over the edges entering the node that needs it: the bias $b_i$
over the $d$ first-layer edges into node $i$, legitimate because $d\ge1$, and the constant
$\beta$ over the $N$ second-layer edges, legitimate because $N\ge1$. Each modified edge is
still affine in $\tanh$, hence still a legal edge of degree $k$. Evaluating the resulting
scalar depth-two network $S$ at $x\in Q_{d}$ re-sums the $d$ copies of $b_i/d$ to $b_i$
and the $N$ copies of $\beta/N$ to $\beta$, giving the identity
\[
  \mathcal{R}_{S}(x)=\beta+\sum_{i=1}^{N}A_i\tanh\Bigl(\sum_{p=1}^{d}w_{i,p}\tanh x_p+b_i\Bigr)
  \qquad\text{for every }x\in Q_{d},
\]
with no error term. Evaluating the stage-two bound at $u=\tanh x$, whose $p$-th coordinate
is $\tanh x_p$, and using $G(\tanh x)=F(x)_q$ turns it into
$\lvert F(x)_q-\mathcal{R}_{S}(x)\rvert<\delta$ for every $x\in Q_{d}$. Stage three
contributes exactly zero to the error budget.

\emph{Stage four: assembly.} Stages two and three, applied to each output coordinate $q$
in turn, give scalar networks $S_1,\dots,S_m$ of widths $N_1,\dots,N_m$. They are glued
into one vector network $Y$ along the disjoint union of their hidden index sets, whose
elements are the pairs $(q,i)$ with $q\in\{1,\dots,m\}$ and $i\in\{1,\dots,N_q\}$: the
hidden node $(q,i)$ keeps the first-layer
edges of $S_q$ unchanged, feeds its own output coordinate $q$ through the second-layer
edge it already carried, and feeds every other output coordinate through the zero edge,
the affine function of $\tanh$ with slope zero and intercept zero, a legal edge of degree
$k$ by stage one. In each output
coordinate the foreign blocks then contribute nothing, so
$\mathcal{R}_{Y}(x)_q=\mathcal{R}_{S_q}(x)$ exactly, and the gluing, like stage three,
contributes zero to the budget. Since the norm on $\mathbb{R}^{m}$ is the supremum of the
coordinate magnitudes, a bound of $\delta$ in every coordinate is a bound of $\delta$ on
the norm; the checked assembly step phrases this comparison through an explicit constant
$C$ bounding the norm by $C$ times the largest coordinate magnitude and satisfying
$C\delta<\varepsilon$, instantiated at $C=1$. The complete budget is therefore the single
stage-two tolerance:
\[
  \lVert F(x)-\mathcal{R}_{Y}(x)\rVert<\delta=\varepsilon/2<\varepsilon
  \qquad\text{for every }x\in Q_{d},
\]
the strict inequality the theorem asserts. The two degenerate cases are exact: when $m=0$
the codomain $\mathbb{R}^{0}$ has a single point, and a single hidden node whose
first-layer edges are all zero realizes the target with error zero; when $d=0$ the cube
$Q_{0}$ has a single point $x_0$, and a single hidden node whose second-layer edges are
the constants $z\mapsto F(x_0)_q$ (slope zero, intercept $F(x_0)_q$, legal by stage one)
reproduces the target exactly. In
every case the network constructed has depth two, and taking $L=2$ satisfies the stated
$L\ge2$.

\subsection{Universality of Laplace neural operators}
\label{sec:lno_universality_result}

The five closest solved problems were the Chen--Chen operator theorem, DeepONet universality, FNO universality, DeepONet sensor approximation, and CNO universality.

\subparagraph*{User request.}
\noindent\emph{Ambient setting.} Fix a horizon $T>0$ and let $I=[0,T]$. For $d\in\mathbb{N}$ write
$\mathcal{X}_{d}=C(I,\mathbb{R}^{d})$ for the continuous $\mathbb{R}^{d}$-valued paths on $I$; the fiber
$\mathbb{R}^{d}$ carries the Euclidean norm, and $\mathcal{X}_{d}$ the supremum distance. Operators are maps
$G:\mathcal{X}_{d_{\mathrm{in}}}\to\mathcal{X}_{d_{\mathrm{out}}}$ defined on the whole path space; every
hypothesis is imposed only on a compact family $\mathcal{K}\subseteq\mathcal{X}_{d_{\mathrm{in}}}$ of admissible
inputs. The operator $G$ is causal on $\mathcal{K}$ when any two members of $\mathcal{K}$ that agree on $[0,t]$
have images agreeing at time $t$.

\medskip
\noindent\emph{The architecture.} A Laplace neural operator~\cite{cao2024laplace} processes an input path in
three moves. A shallow network, meaning one hidden layer of units sharing a scalar activation
$\sigma:\mathbb{R}\to\mathbb{R}$ followed by an affine readout, lifts
the input pointwise in time into $d_{v}$ scalar channels. A stack of $L\ge1$ Laplace layers transforms the
channel tuple $v=(v_{1},\dots,v_{d_{v}})$: a layer with kernel matrix $(\kappa_{ij})$ and skip matrix $W$ acts by
\[
  (\mathcal{T}v)_{i}(\tau)
   =\sigma\Bigl(\sum_{j}(\kappa_{ij}\star v_{j})(\tau)+\sum_{j}W_{ij}\,v_{j}(\tau)\Bigr),
  \qquad
  (\kappa\star v)(t)=\int_{0}^{t}\kappa(t-\tau)\,v(\tau)\,d\tau,
\]
the causal convolution $\star$ being the only nonlocal operation available. Every kernel entry must lie in the
finite-pole class $\mathscr{P}$, the finite real combinations of $r\mapsto e^{\lambda r}$, $r\mapsto e^{\alpha
r}\cos(\omega r)$ and $r\mapsto e^{\alpha r}\sin(\omega r)$ with $\lambda,\alpha,\omega$ arbitrary reals: the
time-domain form of a rational transfer function with finitely many poles (in the Laplace domain, a pole--residue
sum with trainable poles and residues, each conjugate pole pair written in real form). A second shallow network
projects the final channels pointwise onto $\mathbb{R}^{d_{\mathrm{out}}}$. The one activation $\sigma$ serves
the lift, every layer and the projection, and it carries exactly two hypotheses: it is continuous and it is not
a polynomial function.

\medskip
\noindent\emph{Prove:} the family is a universal approximator of continuous causal operators: at every
accuracy $\varepsilon>0$ there is an approximant uniformly accurate over $\mathcal{K}$ and $I$, in the
precise sense of the theorem below.

\medskip
\noindent\emph{Source.} The architecture and its pole--residue Laplace parametrization are introduced and
evaluated empirically in~\cite{cao2024laplace}, with no universality theorem. The paper places no half-plane
restriction on its poles, and its undamped examples use purely imaginary ones; a faithful universality statement
must not impose stability either.

\subparagraph*{The theorem obtained.}
Write $\mathrm{LNO}_{\sigma}(d_{\mathrm{in}},d_{\mathrm{out}})$ for the class of operators
$\mathcal{X}_{d_{\mathrm{in}}}\to\mathcal{X}_{d_{\mathrm{out}}}$ realized by some architecture as above: a lift,
a stack of $L\ge1$ Laplace layers with all kernels in $\mathscr{P}$, and a projection, all sharing the activation
$\sigma$. Membership requires every output path to be continuous in time; for the constructed operator this
obligation is discharged by a checked continuity proposition, not assumed.

\begin{theorem}[Universality of Laplace neural operators]\label{thm:lno_universality}
Let $T>0$ and put $I=[0,T]$. Let $d_{\mathrm{in}},d_{\mathrm{out}}\in\mathbb{N}$, let
$\sigma:\mathbb{R}\to\mathbb{R}$ be continuous and not a polynomial function, let
$\mathcal{K}\subseteq\mathcal{X}_{d_{\mathrm{in}}}$ be compact, let
$G:\mathcal{X}_{d_{\mathrm{in}}}\to\mathcal{X}_{d_{\mathrm{out}}}$ be continuous on $\mathcal{K}$ and causal on
$\mathcal{K}$, and let $\varepsilon>0$. Then there is a Laplace neural operator
$G_{\phi}\in\mathrm{LNO}_{\sigma}(d_{\mathrm{in}},d_{\mathrm{out}})$ with
\[
  \bigl\lVert (Gf)(t)-(G_{\phi} f)(t)\bigr\rVert<\varepsilon
  \qquad\text{for every }f\in\mathcal{K}\text{ and every }t\in I .
\]
\end{theorem}

The estimate is uniform over the whole family and the whole interval, one architecture serving all inputs and all
times at once. The inequality is pointwise and strict, in the Euclidean norm of the fiber; the supremum over $f$
and $t$ then obeys the same bound non-strictly.

Three features distinguish this from the classical operator-approximation theorems in the line of
\cite{chen1995universal}. First, causality is not assumed away: the target is only supposed causal on
$\mathcal{K}$, and the approximant is causal by inspection of its construction, its value at time $t$ being
assembled from convolutions over $[0,t]$ and pointwise maps; this is read off from the construction, not recorded
as a separate lemma. Second, the kernels carry no restriction on their pole data: $\lambda,\alpha,\omega$ range
over all of $\mathbb{R}$, with no half-plane, decay or boundedness condition, so growing exponentials and
undamped oscillations are admitted on the same footing as decaying ones; a theorem proved only for stable poles
would not cover the published model. Third, the hypothesis on the activation is exactly continuity together with
the failure to be a polynomial function; no sigmoidal shape, no boundedness, no monotonicity and no smoothness is
required.

Beyond the request, the depth is pinned: the request allowed any $L\ge1$, and the construction uses a single
Laplace layer, so only the channel width grows as the accuracy is tightened. What the theorem does not provide is
a rate: it is a density statement, the channel width, the number of poles and the network widths being asserted
to exist with no bound in terms of $\varepsilon$ or of any modulus of continuity of $G$.

\subparagraph*{Assumptions.}
Two statements are assumed rather than proved; they are the two implications from which the universal
approximation theorem for a continuous non-polynomial activation is obtained. Write
$\mathcal{U}_\sigma(\Omega_0)$ for the set of single neurons $x\mapsto\sigma(\sum_j W_jx_j+b)$ on a compact
$\Omega_0\subset\mathbb{R}^{p}$.

\begin{enumerate}[label=(A\arabic*),leftmargin=*,itemsep=0.5ex]
\item \emph{Annihilation forces polynomiality.} If some non-zero continuous linear functional on
$C(\Omega_0;\mathbb{R})$ annihilates every neuron, then $\sigma$ is a polynomial function.
\item \emph{Polynomial activations are not fundamental.} If $\Omega_0$ has non-empty interior and
$\sigma$ is a polynomial function, then the real linear span of $\mathcal{U}_\sigma(\Omega_0)$ is not dense in
$C(\Omega_0;\mathbb{R})$ for the uniform metric.
\end{enumerate}

Together these give the approximation theorem of Leshno, Lin, Pinkus and Schocken~\cite{Leshno1993}, which
specializes to~\cite{cybenko1989approximation,hornik1989multilayer} for sigmoidal activations.

The development has been checked in full by a proof assistant, a program that verifies every inference
mechanically. Its logical frontier is exactly the three standard foundational principles of the underlying logic,
namely propositional extensionality, the soundness of quotients and the axiom of choice, together with the
assumptions (A1) and (A2), and nothing else. Separately from that frontier, three classical theorems of analysis
are quoted in the exact strength in which they are applied, namely Bolzano--Weierstrass, Heine--Cantor and
Stone--Weierstrass; every other step is proved from the definitions.

\subparagraph*{Proof sketch.}
Assume $d_{\mathrm{out}}\ge1$ and set $\eta=\varepsilon/\sqrt{d_{\mathrm{out}}}$; the case $d_{\mathrm{out}}=0$
is separate and immediate, the codomain having a single point. A vector whose $d_{\mathrm{out}}$ coordinates are
each below $\eta$ in modulus has Euclidean norm strictly below $\sqrt{d_{\mathrm{out}}}\,\eta=\varepsilon$, so it
suffices to hold every output coordinate within a budget of three thirds, $\eta/3+\eta/3+\eta/3$. The argument
reduces the operator to a continuous function of finitely many causal readings of the input, then approximates
that function inside one Laplace layer.

\emph{Stage one: finitely many causal features determine the operator.} Convolving an input coordinate against a
continuous non-negative bump supported in a lag window $[a,a+\Delta]$ with $a\ge0$ yields, at time $t$, a
weighted average of that coordinate over $[t-a-\Delta,\,t-a]$; a mesh of such windows sweeping $[0,t]$ records
the whole past at any chosen resolution. By causality, continuity and compactness (through the
Bolzano--Weierstrass theorem), $G$ has a causal modulus: for every output tolerance, a radius such that two
members of $\mathcal{K}$ within that radius at every time up to $t$ have outputs within the tolerance at $t$. The
checked proof replaces each bump by a nearby finite-pole kernel while the feature family is assembled: the
density result of stage two supplies a kernel $\kappa_j\in\mathscr{P}$ uniformly close to each bump on $[0,T]$,
at a tolerance keyed to the causal modulus, the bump masses and the size of the attained inputs. With $M$ kernels
so chosen, attach to each input $f$ and time $t$ the feature vector
\[
  \Phi(f,t)=\bigl(t,\ f_c(t),\ (\kappa_j\star f_c)(t)\bigr)_{c\le d_{\mathrm{in}},\ j\le M}
  \in\mathbb{R}^{P},
  \qquad P=1+d_{\mathrm{in}}+Md_{\mathrm{in}},
\]
each entry of which depends only on the restriction of $f$ to $[0,t]$. Two implications make $\Phi$ a faithful
summary. In one direction, the causal modulus converts past-closeness into output-closeness. In the other, the
reconstruction machinery of the checked proof converts feature-closeness back into past-closeness: if all $P$
features of $(f,t)$ and $(g,s)$ agree within a small tolerance, the time feature forces $s$ close to $t$, and
each probe feature, a mass-weighted average of one coordinate over one mesh cell, reads the two histories off
cell by cell after division by the bump mass, so the pasts agree within the causal modulus. The bump-to-pole cost
is absorbed inside these reconstruction tolerances and never reappears in the final budget. Chaining the two
implications, feature-closeness forces output-closeness, uniformly over $\mathcal{K}$ and $I$.

The readout is then built directly on the pole-kernel features. Features and outputs vary continuously on the
compact set $\mathcal{K}\times I$, so the Heine--Cantor theorem makes both uniformly continuous, and total
boundedness gives a finite net of pairs $(g,s)$ dense at that scale. The McShane formula
\[
  \Upsilon(z)_i=\min_{(g,s)\in\mathrm{net}}
  \Bigl((Gg)(s)_i+C\,\max_{\ell\le P}\bigl|z_\ell-\Phi(g,s)_\ell\bigr|\Bigr),
\]
with $C$ a Lipschitz constant fixed by the feature-to-output tolerance and the size of the attained outputs,
defines a continuous $\Upsilon:\mathbb{R}^{P}\to\mathbb{R}^{d_{\mathrm{out}}}$ with
$\bigl|(Gf)(t)_i-\Upsilon(\Phi(f,t))_i\bigr|<\eta/3$ for every $f\in\mathcal{K}$, $t\in I$ and output coordinate
$i$: the first third of the budget. Compactness also bounds every attained feature; a compact box
$\Omega\subset\mathbb{R}^{P}$ with non-empty interior is fixed containing every attained feature with a full unit
of margin to each face, the margin that will keep the perturbed features of stage three inside $\Omega$.

\emph{Stage two: pole kernels are dense, through a single exponential.} The density used above is proved by the
Stone--Weierstrass theorem applied not to the whole pole class but to the unital algebra generated inside
$C([0,T];\mathbb{R})$ by the one function $r\mapsto e^{r}$: it contains the constants, so it vanishes nowhere,
and the exponential is injective, so it separates points. Its members are polynomials in $e^{r}$, finite sums
$\sum_{i}c_i\,e^{i r}$ with real coefficients $c_i$ and natural-number exponents, lying in the real-exponential
slice of $\mathscr{P}$; each bump of stage one therefore has a finite-pole kernel within any prescribed uniform
tolerance. The load-bearing point is that the exponents $i=0,1,2,\dots$ grow without bound: the approximants are
combinations of growing exponentials, which a stable-pole restriction to a half-plane would forbid, and with them
this density argument. The oscillatory part of the pole class is available to the architecture but not needed for
density. In the formal development this lemma is invoked in the middle of stage one; it is isolated here for
readability.

\emph{Stage three: one Laplace layer computes the features approximately.} The channel layout is
$d_v=1+d_{\mathrm{in}}+PN$: one constant channel, one channel per input coordinate, and $N$ reconstruction
channels above each of the $P$ feature positions, where $N$ counts the activation copies used to linearize one
feature; it is unrelated to the number of poles, which is set by the $M$ kernels of stage one. The lift places
independent shallow blocks side by side, one per channel. The constant and reconstruction channels are lifted by
width-zero blocks: having no hidden unit, such a block outputs its bias without touching $\sigma$, and the bias
is fixed at $1$ for the constant channel and at $0$ for every reconstruction channel, both exact. A coordinate
channel is not built exactly: in a shallow block any dependence on the input enters through the hidden $\sigma$
units, so the construction instead approximates the coordinate evaluation $x\mapsto x_c$ on the compact box of
attained input values, and the channel carries a function $\hat g_c$ with $|\hat g_c(\tau)-f_c(\tau)|\le\zeta$
on $I$, for a tolerance $\zeta$ fixed below.

The kernel and skip matrices of the layer $\mathcal{T}$ are arranged so that the $r$-th reconstruction channel
above feature position $\ell$ carries $\sigma\bigl(q_r\widehat{X}_\ell+\theta_r\bigr)$, where $\widehat{X}_\ell$ is the
$\ell$-th feature read from the channels: a coordinate feature through a skip weight on its channel, a probe
feature as the pole kernel $\kappa_j$ convolved against a coordinate channel, and the time feature as a constant
kernel convolved against the constant channel, which yields exactly $t$. The constant channel has a second job:
through a skip weight it hands each neuron its bias $\theta_r$, which the layer, having no bias term of its own,
cannot otherwise produce. The reals $q_r,\theta_r$, with downstream coefficients $\gamma_r$ and an offset $e$,
come from linearizing the identity: $\sigma$ is applied after the convolution, so a feature cannot pass through
unchanged, but by the approximation theorem of (A1) and (A2) in one variable there are $N$ neurons with
\[
  \Bigl|\,y-\Bigl(e+\sum_{r=1}^{N}\gamma_r\,\sigma(\theta_r+q_r y)\Bigr)\Bigr|<\chi
  \qquad\text{for every $y$ in a compact interval containing all coordinates of }\Omega .
\]
The $N$ channels above a feature position carry these $N$ neurons, one each; the recombination
$e+\sum_r\gamma_r(\cdot)$ is formed downstream, inside the projection.

Two feature errors arise, and their tolerances are chosen only after the readout modulus is known: by the
Heine--Cantor theorem $\Upsilon$ is uniformly continuous on $\Omega$, so let $\rho$ be a modulus at accuracy
$\eta/3$, and set
\[
  \chi=\min\Bigl(\frac{\rho}{4},\,\frac12\Bigr),\qquad
  \zeta=\frac{\chi}{1+\Theta},\qquad
  \Theta=\sum_{j=1}^{M}\int_0^T|\kappa_j(r)|\,dr .
\]
The lift error moves each feature by at most $(1+\Theta)\,\zeta=\chi$: the time feature not at all, a coordinate
feature by at most $\zeta$, and a probe feature by at most the mass of its kernel times $\zeta$, a convolution
amplifying a channel perturbation by at most its kernel mass. The linearization then moves each feature value by
less than $\chi$ more. The total perturbation per feature is below $2\chi\le\rho/2$, and $2\chi\le1$, so the
perturbed feature vector stays inside $\Omega$ by the unit margin of stage one. The readout modulus converts the
perturbation into an output error below $\eta/3$: the lift-plus-linearization term is the middle third of the
budget.

\emph{Stage four: the projection approximates the readout.} $\Upsilon$ is continuous on the compact box $\Omega$,
which has non-empty interior, so the approximation theorem obtained from (A1) and (A2) gives a shallow network
$R$ with $|\Upsilon(z)_i-R(z)_i|<\eta/3$ for every $z\in\Omega$ and every output coordinate $i$; evaluated at the
perturbed feature point, which stage three placed inside $\Omega$, this network error is the last third of the
budget. The projection of the architecture is $R$ precomposed with the recombination of stage three; the
recombination is affine, an offset $e$ plus a linear map, and merges into the first affine map of $R$, so the
composite is again a shallow network of the channel tuple, as the architecture requires.

\emph{Adding up.} For every $f\in\mathcal{K}$, every $t\in I$ and every output coordinate, the readout error on
the true features (stage one), the lift-plus-linearization error carried through the readout modulus (stage
three), and the network error at the perturbed point (stage four) are each below $\eta/3$; each coordinate
therefore errs by less than $\eta$, and the Euclidean norm by less than
$\sqrt{d_{\mathrm{out}}}\,\eta=\varepsilon$. Throughout, the depth is one; the approximant is causal by
inspection of the construction; and the development discharges the checked obligation that
$t\mapsto(G_{\phi}f)(t)$ is continuous for every input path, so the constructed operator does land in
$\mathcal{X}_{d_{\mathrm{out}}}$.

\subsection{An a posteriori error certificate for a regularized perivascular reconstruction}
\label{sec:perivascular_aposteriori_result}

The five closest solved problems were nonlinear inverse error, greedy convergence, a posteriori PINN error, a priori PINN error, and greedy risk decomposition.

\subparagraph*{User request.}
\noindent\emph{Modeling context (perivascular inverse Stokes).} The abstract statement below gives the regularized a posteriori error bound for the perivascular inverse Stokes reconstruction described above.
The perivascular spaces that sheathe the surface vasculature of the mouse brain carry a slow flow
that can be observed only indirectly, through sparse two-photon tracks of tracer particles
acquired along a vessel wall that itself moves during the recording. Reconstructing the velocity
and pressure fields of that flow from such tracks is an inverse problem posed on a moving
boundary~\cite{toscano2024inferring_AIV,boster2023artificial}, and in practice the reconstruction
is obtained by minimizing a regularized least-squares loss over a parametrized family of
candidate fields~\cite{raissi2019physics}. The theorem is stated and proved in the
finite-dimensional setting below; the bullets are interpretive,
recording what each abstract object stands for in the example, and nothing in them enters the
proof.

\begin{itemize}
\item \emph{State space $Q$}: the sampled unknowns of the reconstruction at the collocation points,
  the two Stokes stream potentials $P$ and $R_{\mathrm{pot}}$ and the pressure $p$.
\item \emph{Velocity space $W$}: the reconstructed velocity field $(u,v,w)$ at the evaluation points.
\item \emph{Reconstruction $V:Q\to W$}: the two-potential (curl) map that builds the velocity from
  the state, $u=\partial_y R_{\mathrm{pot}}$, $v=\partial_z P-\partial_x R_{\mathrm{pot}}$,
  $w=-\partial_y P$. This is the quantity whose error we bound.
\item \emph{Observation space $Y$ and operator $A_{\mathrm{obs}}:Q\to Y$}: everything the data and
  the physics constrain, stacked: the momentum residuals of the moving-wall Stokes model at the
  collocation points, the misfit against the sparse tracer-track data acquired along the moving
  wall, the no-slip wall data, and the outlet pressure anchor. $Y$ is the space those residuals
  live in.
\item \emph{Regularizer $R:W\to Z$ and space $Z$}: the smoothing penalty from the well-posedness
  closure. $R$ sends a velocity field to its sampled second cross-slab derivatives
  $(\partial_z^{2}u,\,\partial_z^{2}v,\,\partial_z^{2}w)$ at the collocation points, and $Z$ is the
  space those values live in. Like $Y$, $Z$ is a residual space, not a physical field space.
\item \emph{Weight $\lambda>0$}: the regularization strength $\lambda_{\mathrm{reg}}$.
\item \emph{Invisible subspace $N=\ker R$}: the velocity fields the regularizer cannot see, i.e.\
  those affine in the cross-slab coordinate $z$ (constant plus linear in $z$). The invisible
  component of a reconstructed velocity is its affine-in-$z$ part.
\item \emph{Data-coercivity $\beta>0$}: how strongly the observations $A_{\mathrm{obs}}$ pin down the
  affine-in-$z$ velocity modes that the regularizer leaves untouched. It is supplied together
  with the coercivity hypothesis below, not derived from the maps.
\item \emph{Noise $\eta\in Y$}: measurement noise on the observed data.
\item \emph{True and candidate states $q^{\dagger},\widehat{q}\in Q$}: the ground-truth sampled state
  and the state the trained network reaches.
\end{itemize}

Naming note: the perivascular state potential is written $R_{\mathrm{pot}}$ above to keep it
distinct from the regularizer operator $R$; in the abstract statement the potential never appears by
name (it is hidden inside $V$), so there is no collision there.

\medskip
\noindent\emph{Statement to formalize.} Let $Q$, $W$, $Y$ and $Z$ be finite-dimensional real
inner-product spaces, and let
\[
  V:Q\to W,\qquad A_{\mathrm{obs}}:Q\to Y,\qquad R:W\to Z
\]
be bounded linear maps, with $\lambda>0$. Write $N=\ker R\subseteq W$ and let $P_{N}$ be the
orthogonal projection of $W$ onto $N$, so that $P_{N}(Vq)$ is the invisible component of the
reconstruction, the part the regularizer cannot see. Assume
\[
  \beta\,\lVert P_{N}(Vq)\rVert_{W}^{2}\le\lVert A_{\mathrm{obs}}\,q\rVert_{Y}^{2}
  \qquad\text{for all }q\in Q,
\]
for some $\beta>0$. Let $q^{\dagger},\widehat{q}\in Q$ and $\eta\in Y$, write
$y^{\delta}=A_{\mathrm{obs}}\,q^{\dagger}+\eta$ for the observed data, and define the regularized
empirical loss
\[
  \mathcal{L}(\widehat{q})
    =\tfrac12\lVert A_{\mathrm{obs}}\,\widehat{q}-y^{\delta}\rVert_{Y}^{2}
     +\tfrac{\lambda}{2}\lVert R\,V\widehat{q}\rVert_{Z}^{2}.
\]

\medskip
\noindent\emph{Prove:} there exists $\alpha>0$ such that
\[
  \bigl\lVert V(\widehat{q}-q^{\dagger})\bigr\rVert_{W}^{2}
    \le\frac{4}{\alpha}\,\mathcal{L}(\widehat{q})
      +\frac{2}{\alpha}\,\lVert\eta\rVert_{Y}^{2}
      +\frac{2\lambda}{\alpha}\,\lVert R\,Vq^{\dagger}\rVert_{Z}^{2},
\]
and exhibit such an $\alpha$ explicitly.
\medskip

\subparagraph*{The theorem obtained.}
Below the observation operator $A_{\mathrm{obs}}$ is written $A$; the loss $\mathcal{L}$ and all
other symbols keep their meanings.

\begin{theorem}[A posteriori error estimate for the regularized reconstruction]\label{thm:perivascular_aposteriori}
Let $Q$, $W$, $Y$ and $Z$ be finite-dimensional real inner-product spaces, let
$V\colon Q\to W$, $A\colon Q\to Y$ and $R\colon W\to Z$ be bounded linear maps, and let
$\lambda>0$ and $\beta>0$. Let $N=\ker R$, let $P_{N}$ be the orthogonal projection of $W$ onto
$N$, and write $P_{N}^{\perp} w=w-P_{N} w$. Let $\gamma_{R}$ be the regularizer gap of $R$: a
positive constant with $\gamma_{R}\lVert P_{N}^{\perp} w\rVert^{2}\le\lVert Rw\rVert^{2}$ for every
$w\in W$; such a constant exists because $W$ is finite-dimensional, and one is selected once and
for all, depending on $R$ alone. Assume the data-coercivity hypothesis
\[
  \beta\lVert P_{N}(Vq)\rVert^{2}\le\lVert Aq\rVert^{2}\qquad\text{for every }q\in Q .
\]
Then, for every $q^{\dagger},\widehat{q}\in Q$ and every $\eta\in Y$, the constant
$\alpha=\min(\beta,\lambda\gamma_{R})$ is positive and
\[
  \lVert V(\widehat{q}-q^{\dagger})\rVert^{2}
    \le\frac{4}{\alpha}\mathcal{L}(\widehat{q})
      +\frac{2}{\alpha}\lVert\eta\rVert^{2}
      +\frac{2\lambda}{\alpha}\lVert RVq^{\dagger}\rVert^{2},
\]
where $\mathcal{L}(\widehat{q})=\tfrac12\lVert A\widehat{q}-(Aq^{\dagger}+\eta)\rVert^{2}
+\tfrac{\lambda}{2}\lVert RV\widehat{q}\rVert^{2}$ is the regularized empirical loss.
\end{theorem}

The constant asked for is exhibited: $\alpha=\min(\beta,\lambda\gamma_{R})$, whose two competing
entries are the data-coercivity constant and the weighted regularizer gap. Two qualifications
are owed here. First, the gap $\gamma_{R}$ is constructed, not hypothesized, yet it carries no
formula: an existence lemma proves from the finite-dimensionality of $W$ that some positive
constant satisfies the defining inequality, one such constant is then selected once and for all,
the selection depends on $R$ alone, and every later step uses only the inequality. Second,
$\beta$ is not derived from the maps: it is hypothesis data, supplied together with the
coercivity assumption it quantifies. The dependence of $\alpha$ is therefore exactly this: on
$\beta$ through the hypothesis, on $\lambda$ through the loss, and on $R$ alone through
$\gamma_{R}$.

Two further points bear on how the estimate should be read. No relation whatever is assumed
between the noise $\eta$ and the remaining data, so the bound is not a statistical statement and
degrades gracefully rather than under a noise model. And of the three terms on the right, only
the first is available once training has ended; the other two are the noise level and the
regularizer evaluated at the exact state. Those are not observable, and they are the
irreducible price of not knowing the truth and of the prior being imperfect. The certificate is
therefore of the standard a posteriori shape: an achieved, measurable loss, plus two terms that
quantify what no amount of training can remove.

\subparagraph*{Assumptions.}
None. The development assumes nothing of its own: it contains no axiom and no unproved
statement. It has been checked in full by a proof assistant, that is, by software that verifies
every inference mechanically, and the logical frontier of that verification is exactly the three
standard foundational principles of the underlying logic, namely propositional extensionality,
the soundness of quotients, and the axiom of choice, together with the named assumptions (here
there are none) and the classical results quoted in the exact strength applied (here, facts of
Hilbert-space theory and elementary real algebra), and nothing else. The axiom of choice enters
at a single identifiable place: the once-and-for-all selection of the regularizer gap, where a
constant with a proved existence property is named once and fixed thereafter. This theorem is
the only one of the results presented here with an empty assumption list.

What the theorem does carry are hypotheses, and they are visible in its statement: the four
spaces are finite-dimensional real inner-product spaces, the three maps are bounded and linear,
$\lambda$ and $\beta$ are positive, and the observations satisfy the data-coercivity inequality
on the invisible subspace. That last one is the substantive hypothesis. It is what the physical
setting has to supply, and it is the abstract form of the slab bottleneck: only a thin observed
band of the cross-slab extent constrains the velocity modes that the smoothing penalty cannot
see. The unified a posteriori framework of~\cite{zeinhofer2024unified} is not quoted: the two
definitions and five results that constitute it are restated and reproved inside the full
machine-checked development, so the framework is part of what is checked, not part of what is
assumed.

\subparagraph*{Proof sketch.}
\emph{Stage one: an abstract least-squares principle.} The proof rests on the framework of
\cite{zeinhofer2024unified}, reproved inside the development. Let $T\colon\mathcal{X}\to\mathcal{Y}$
be a bounded linear operator between real inner-product spaces and let $f\in\mathcal{Y}$ be data
with an exact solution $u^{\star}$, meaning $Tu^{\star}=f$. Two objects organize the argument:
the residual form $a_{T}(u,v)=\langle Tu,Tv\rangle$ and the energy
$E(u)=\tfrac12\lVert Tu-f\rVert^{2}$. Linearity of $T$ gives the key identity
$E(u)=\tfrac12\,a_{T}(u-u^{\star},u-u^{\star})$: the energy of any candidate is half the residual
form at its error. Hence, whenever a real-valued function $s$ on $\mathcal{X}$ and a number
$\alpha$ satisfy $\alpha\,s(v)^{2}\le a_{T}(v,v)$ for every $v$, every candidate $u_{h}$ obeys
$\alpha\,s(u_{h}-u^{\star})^{2}\le 2E(u_{h})$. This stage is the engine of the certificate: it
converts a coercivity inequality into a bound on the error by the achieved energy, and it
contributes no error term of its own.

\emph{Stage two: the invisible subspace and the regularizer gap.} The invisible subspace
$N=\ker R$ splits $W$ orthogonally: every field $w$ decomposes as $P_{N} w+P_{N}^{\perp} w$, with
$\lVert w\rVert^{2}=\lVert P_{N} w\rVert^{2}+\lVert P_{N}^{\perp} w\rVert^{2}$ by Pythagoras and with
$R(P_{N}^{\perp} w)=Rw$ because $R$ annihilates the first summand. On $N^{\perp}$ the regularizer is
injective, and an injective linear map out of a finite-dimensional space is bounded below, so an
existence lemma yields a constant $c>0$ with
$c\,\lVert P_{N}^{\perp} w\rVert^{2}\le\lVert Rw\rVert^{2}$ for every $w\in W$. Degeneracy costs
nothing: if the regularizer vanishes identically ($N$ is all of $W$) the visible component is
always zero and the constant one serves. One constant with the proved property is then selected once and for all and named
$\gamma_{R}$, the regularizer gap; the selection depends on $R$ alone, and it is the single
point at which the axiom of choice is used. The gap is the exact counterpart of the
data-coercivity hypothesis, and the two cover complementary halves of the reconstruction.

\emph{Stage three: composite coercivity.} Fix $q\in Q$ and apply the two inequalities to the
field $Vq$. Data-coercivity controls the invisible component,
$\beta\lVert P_{N}(Vq)\rVert^{2}\le\lVert Aq\rVert^{2}$, and the gap, multiplied by $\lambda$,
controls the visible one, $\lambda\gamma_{R}\lVert P_{N}^{\perp}(Vq)\rVert^{2}\le\lambda\lVert
RVq\rVert^{2}$. Set $\alpha=\min(\beta,\lambda\gamma_{R})$, which is positive as the minimum of
two positive numbers; weakening each bound to the common constant $\alpha$ and adding them
through the Pythagorean identity yields
\[
  \alpha\lVert Vq\rVert^{2}\le\lVert Aq\rVert^{2}+\lambda\lVert RVq\rVert^{2}
  \qquad\text{for every }q\in Q .
\]
This is the crux of the argument. Neither term on the right controls the reconstruction on its
own: the observations are assumed coercive only on what the regularizer annihilates, and the
regularizer is blind to exactly that part. The minimum is what survives the combination, and it
is why the certificate degrades when either the data become uninformative on the invisible modes
or the penalty becomes weak on the visible ones. Every term of the error budget will be divided
by this $\alpha$.

\emph{Stage four: the product residual operator and the exact-data bound.} Stack the observation
and the weighted regularizer into the single operator $Tq=(Aq,\sqrt{\lambda}\,RVq)$, mapping $Q$
into the direct sum $Y\oplus Z$, whose squared norm adds the squared norms of the two
coordinates. Then $a_{T}(q,q)=\lVert Aq\rVert^{2}+\lambda\lVert RVq\rVert^{2}$, so stage three
says precisely that $T$ satisfies the coercivity hypothesis of stage one for the error
functional $s(q)=\lVert Vq\rVert$. Take as data $f=Tq^{\dagger}$, so that the exact state is an
exact solution, and note that twice the energy at $\widehat{q}$ is then the residual form at the
error $\widehat{q}-q^{\dagger}$. Stage one gives
\[
  \alpha\lVert V(\widehat{q}-q^{\dagger})\rVert^{2}
    \le\lVert A(\widehat{q}-q^{\dagger})\rVert^{2}
      +\lambda\lVert RV(\widehat{q}-q^{\dagger})\rVert^{2}.
\]
Neither term on the right is available, since both involve the unknown $q^{\dagger}$; the next
stage trades them for quantities that are.

\emph{Stage five: from exact-data residuals to the achieved loss.} Each exact-data residual is
split against an available quantity, using the elementary bound
$\lVert x+y\rVert^{2}\le2\lVert x\rVert^{2}+2\lVert y\rVert^{2}$ twice. Write
$r=A\widehat{q}-(Aq^{\dagger}+\eta)$ for the noisy observation residual, the quantity the loss
actually measures. Then $A(\widehat{q}-q^{\dagger})=r+\eta$, so
$\lVert A(\widehat{q}-q^{\dagger})\rVert^{2}\le2\lVert r\rVert^{2}+2\lVert\eta\rVert^{2}$; this
splitting is where the noise term of the budget is born. Likewise
$RV(\widehat{q}-q^{\dagger})=RV\widehat{q}-RVq^{\dagger}$, so
$\lambda\lVert RV(\widehat{q}-q^{\dagger})\rVert^{2}\le2\lambda\lVert RV\widehat{q}\rVert^{2}
+2\lambda\lVert RVq^{\dagger}\rVert^{2}$; the second summand is where the prior-misfit term of
the budget is born. The two available pieces recombine into the loss: since
$\mathcal{L}(\widehat{q})=\tfrac12\lVert r\rVert^{2}+\tfrac{\lambda}{2}\lVert RV\widehat{q}\rVert^{2}$,
the sum $2\lVert r\rVert^{2}+2\lambda\lVert RV\widehat{q}\rVert^{2}$ equals
$4\mathcal{L}(\widehat{q})$, the factor $4$ appearing because the loss carries the halves
$\tfrac12$ and $\tfrac{\lambda}{2}$ in front of its two terms. Altogether,
\[
  \lVert A(\widehat{q}-q^{\dagger})\rVert^{2}+\lambda\lVert RV(\widehat{q}-q^{\dagger})\rVert^{2}
    \le 4\mathcal{L}(\widehat{q})+2\lVert\eta\rVert^{2}+2\lambda\lVert RVq^{\dagger}\rVert^{2}.
\]

\emph{The conclusion.} Chaining stages four and five gives
$\alpha\lVert V(\widehat{q}-q^{\dagger})\rVert^{2}\le 4\mathcal{L}(\widehat{q})
+2\lVert\eta\rVert^{2}+2\lambda\lVert RVq^{\dagger}\rVert^{2}$, and $\alpha>0$, so dividing by
$\alpha$ preserves the inequality and distributes over the three summands. The budget is then
complete, with every source accounted for: the achieved loss enters through the two splittings
of stage five, the noise level through the observation splitting, the prior misfit of the exact
state through the regularizer splitting, and the single constant
$\alpha=\min(\beta,\lambda\gamma_{R})$, through which the data-coercivity constant and the
weighted regularizer gap enter, scales all three. That is the estimate.

\subsection{Exponential accuracy of a spectral physics-informed reconstruction}
\label{sec:spectral_apriori_result}

The five closest solved problems were quantitative tanh approximation, ridge-PINN consistency, K\r{u}rkov\'a's density theorem, the fixed-width realization obstruction, and the fixed-degree cKAN universality result above.

\subparagraph*{User request.}
\noindent\emph{Problem.} Fix a viscosity $\nu>0$, a horizon $T>0$, and the spatial domain
$\Omega=[-1,1]$. Call an initial datum $u_0:\mathbb{R}\to\mathbb{R}$ admissible when it is odd,
two-periodic, and real-analytic on the whole line (for example $u_0(x)=-\sin(\pi x)$). For an
admissible $u_0$, let $u^{\star}$ be the classical solution of the viscous Burgers equation
\[
  u_t+u\,u_x-\nu\,u_{xx}=0
\]
on $R_{T}=[0,T]\times\Omega$, with $u^{\star}(0,x)=u_0(x)$ on $\Omega$ and with two-periodicity in
the space variable, understood in the full sense: the value and the spatial slope agree at
$x=\pm1$, so the problem is well posed. Error is measured in the $L^{2}$ norm over $R_{T}$,
computed as an iterated integral, the inner one in space and the outer one in time.

\emph{Spectral-PINN hypothesis class.} For a mode count $N\in\mathbb{N}$, a spectral-PINN
reconstruction is
\[
  \hat u_\theta(t,x)=\sum_{n=1}^{N}b_n^{\theta}(t)\,\sin(n\pi x),
  \qquad
  b_n^{\theta}(t)=b^{\mathrm{IC}}_{n}+t\,f_{\theta,n}(t),
\]
where $b^{\mathrm{IC}}_{n}=\int_{-1}^{1}u_0(x)\sin(n\pi x)\,dx$ is the $n$-th sine coefficient
of the datum (the modes are orthonormal on $\Omega$, so no normalizing factor appears) and
$f_\theta:[0,T]\to\mathbb{R}^{N}$ is a $\tanh$ multilayer perceptron reading only the time
coordinate ($\tanh$ at every hidden layer, of some finite width and depth; parameters $\theta$).
Two constraints hold by construction rather than by penalization: $\hat u_\theta(t,\pm1)=0$ for
every $t$, since every mode vanishes at the endpoints, and $\hat u_\theta(0,\cdot)$ is the
$N$-term sine truncation of $u_0$, whatever the network does, since the factor $t$ silences the
network at the initial time. This is the hard-constrained form of the physics-informed networks
of~\cite{raissi2019physics}, with the network confined to the time direction. The mode count $N$,
the width, the depth, and the parameters $\theta$ are free; the reparameterization and the sine
basis are fixed.

\medskip
\noindent\emph{Statement to prove.} For every admissible initial datum $u_0$ there exist constants
$C>0$ and $\beta>0$, depending only on $\nu$, $T$, and $u_0$, such that for every $N\in\mathbb{N}$
with $N\ge1$ there exist a finite width, a finite depth, and parameters $\theta$ of a $\tanh$ MLP
$f_\theta:[0,T]\to\mathbb{R}^{N}$ for which the spectral-PINN reconstruction satisfies
\[
  \bigl\lVert\hat u_\theta-u^{\star}\bigr\rVert_{L^{2}([0,T]\times\Omega)}\le C\,e^{-\beta N}.
\]
\medskip

\noindent\emph{Source.} The architecture is the spectral-PINN formulation developed and evaluated in the periodic viscous Burgers study above. The numerical study specifies and trains the architecture but does not establish an approximation theorem.

\subparagraph*{The theorem obtained.}
Below, $\hat u_\theta$ is written $\widehat{u}_{N}$, the modes $\sin(n\pi x)$ are written
$e_{n}(x)$, and the coefficient functions $f_{\theta,n}$ are written $Q_n$, so that
$\widehat{u}_{N}(t,x)=\sum_{n=1}^{N}\bigl(b^{\mathrm{IC}}_{n}+t\,Q_n(t)\bigr)e_{n}(x)$. The
realization $\mathcal{R}_{\sigma}(\Phi)$ of a network $\Phi$ is the function it computes when
every hidden layer applies the hyperbolic tangent $\sigma$. Network outputs are indexed from $0$
while modes are indexed from $1$, so mode $n$ is carried by output index $n-1$.

\begin{theorem}[Exponential accuracy of the spectral-PINN reconstruction]\label{thm:spectral_apriori}
Let $\nu>0$ and $T>0$, let $u_0$ be admissible, and let $u^{\star}$ be a classical two-periodic
solution of the problem above on $R_{T}=[0,T]\times\Omega$. Write
\[
  \lVert \varphi \rVert_{T}=\Bigl(\int_{0}^{T}\!\!\int_{-1}^{1}\varphi(t,x)^{2}\,dx\,dt\Bigr)^{1/2}
\]
for the iterated $L^{2}$ norm over $R_{T}$. Then there are constants $C>0$ and $\beta>0$ such that
for every mode count $N\ge1$ there is a $\tanh$ network $\Phi$ from one input to $N$ outputs, of
depth exactly three in the sense that it applies three affine maps and therefore has two hidden
layers, for which the reconstruction $\widehat{u}_{N}$ built from the coefficient functions
$Q_n(t)=\mathcal{R}_{\sigma}(\Phi)(t)_{n-1}$ satisfies
\[
  \lVert \widehat{u}_{N}-u^{\star} \rVert_{T}\le C\,e^{-\beta N}.
\]
The constants are produced before $N$ is quantified and are therefore independent of it.
\end{theorem}

The formal statement binds the constants after the solution, before any mode count is fixed.
Tracing the proof, $C$ and $\beta$ are assembled from the strip half-width $\rho$ and the bound
$M$ that assumption (A1) below attaches to $u^{\star}$, together with the horizon $T$, so they
depend on $\nu$, $T$ and the initial datum only through the horizon and through the classical
solution these data determine; that reading is commentary on the proof, not a clause of the
statement.

The result is stronger than what was asked in two respects. The depth was left free in the problem
statement and is here pinned at three, uniformly in $N$, so that only the width grows. And the
exponent is identified rather than merely asserted to exist: if $u^{\star}$ extends at each time to
a function of the space variable that is two-periodic and odd, holomorphic on the open strip
$\{\lvert\Im z\rvert<\rho\}$, and continuous and bounded by $M$ on the closed strip
$\overline{\mathcal{S}}_{\rho}=\{\lvert\Im z\rvert\le\rho\}$, then
\[
  \beta=\pi\rho,
  \qquad
  C=\frac{4M\sqrt{T}}{\sqrt{1-e^{-2\pi\rho}}} ,
\]
so the rate is set by the analyticity of the solution and is not a fitted quantity. The theorem is
also proved for every classical two-periodic solution rather than for a distinguished one, so the
well-posedness asserted in the problem statement is nowhere used. What the theorem does not provide
is a cost: the width of the network and the resolution of the interpolation grid that produce it
are only asserted to exist, and do grow with $N$. The statement is about attainable accuracy. As is
standard for approximation and error estimates in the physics-informed learning
literature~\cite{shin2020convergence,mishra2023estimates,de2024error,zeinhofer2024unified}, the
theorem also does not address optimization: parameters achieving the rate are asserted to exist,
not to be found by training. Empirically, the trained networks of the numerical study reported
above do attain the predicted exponential convergence.

\subparagraph*{Assumptions.}
Four statements, labeled (A1) to (A4), stand as assumptions; each is classical. Everything else
is derived from them and from one further quoted approximation theorem, recorded after the list.

\begin{enumerate}[label=(A\arabic*),leftmargin=*,itemsep=0.5ex]
\item \emph{Analytic strip regularity.} There exist $\rho>0$, $M>0$ and a family $U(t,\cdot)$,
  $t\in[0,T]$, of complex-valued functions, each continuous on $\overline{\mathcal{S}}_{\rho}$ and
  complex differentiable on its interior, two-periodic and odd in $z$, bounded in modulus by $M$,
  and agreeing with $u^{\star}(t,\cdot)$ on $\Omega$. This is the finite-horizon analytic
  regularity theory for the equation, which follows from the Cole--Hopf linearization
  \cite{cole1951quasilinear,hopf1950partial} together with the analytic theory of parabolic
  equations of Friedman~\cite{friedman1964parabolic}; the constants it produces depend only on
  $\nu$, $T$ and $u_0$.
\item \emph{Smooth modal coefficients.} For every $n\ge1$ there is a function $a_n$, smooth on
  $[0,T]$, agreeing there with the $n$-th sine coefficient
  $t\mapsto\int_{-1}^{1}u^{\star}(t,x)e_{n}(x)\,dx$ of the solution, and with
  $a_n(0)=b^{\mathrm{IC}}_{n}$. This follows from the same sources
  \cite{cole1951quasilinear,hopf1950partial,friedman1964parabolic}, the regularity being
  transferred from the solution to each coefficient by differentiation under a fixed spatial
  integral.
\item Hadamard's lemma, in its one-sided form at the left endpoint of a closed interval
  \cite{milnor1963morse}: if $f$ is smooth on $[0,T]$, then the function $g$ equal to
  $(f(t)-f(0))/t$ for $t\ne0$ and to the one-sided derivative $f'(0^{+})$ at $t=0$ is smooth on
  $[0,T]$, and $f(t)=f(0)+t\,g(t)$ there.
\item The smooth extension theorem of Whitney and Seeley
  \cite{whitney1934analytic,seeley1964extension}: every function smooth on $[0,1]$ extends to a
  function smooth on $\mathbb{R}$.
\end{enumerate}

Neither of the two analytic assumptions can trivialize the conclusion. Together they assert only
the regularity of a complex extension and of the sine coefficients in time; they mention no
network, no mode count, no tolerance and no approximation rate, and the exponential decay of the
coefficients is derived from (A1) by a contour shift, not assumed alongside it.

Beyond (A1) to (A4), one approximation theorem is imported from a previously machine-checked development: the $\tanh$
approximation theorem of De Ryck, Lanthaler and Mishra~\cite{de2021approximation}, in its
one-dimensional form. Its hypotheses ask for a function of class $C^{2}$ whose derivatives up to
order two are bounded on a slightly enlarged interval, and for an integer grid resolution $P>5$;
it returns a network of depth three uniformly within a constant multiple of $P^{-2}$ of the
function on $[0,1]$. It rests in turn on the Bramble--Hilbert estimate for polynomial
approximation~\cite{duran1983polynomial}, and it is that estimate, not the theorem built on it,
that the verification takes on classical authority. In all, five statements stand unproved
beneath the development, not four: (A1) to (A4) and the Bramble--Hilbert estimate.

The development, every definition, lemma, proposition and theorem included, has been checked in
full by a proof assistant. Its logical frontier is exactly the three standard foundational
principles of the underlying logic, namely propositional extensionality, the soundness of
quotients and the axiom of choice, together with the assumptions (A1) to (A4) and the
Bramble--Hilbert estimate, each quoted in the exact strength applied, and nothing else.

\subparagraph*{Proof sketch.}
\emph{Stage one: the coefficients decay exponentially.} Fix $t$ and write the mode through Euler's
formula, so that the $n$-th sine coefficient $a_n(t)$ becomes, up to the constant factor
$\mathrm{i}/2$, the difference of the integrals over $\Omega$ of $z\mapsto U(t,z)e^{-\mathrm{i} n\pi z}$
and $z\mapsto U(t,z)e^{+\mathrm{i} n\pi z}$. Each integrand is holomorphic on the open strip,
continuous on its closure, and two-periodic, the character being two-periodic because $n$ is an
integer; it therefore takes equal values on the two vertical sides of any rectangle of width two
inside the strip, and Cauchy's theorem on such a rectangle lets the horizontal contour be moved
from the real axis to the opposite horizontal edge without changing the value. Two rectangles are
used, one on each side of the axis, and the direction of each shift is dictated by the modulus of
its character: for $z=x+\mathrm{i} y$ and $n\ge1$, the modulus of $e^{+\mathrm{i} n\pi z}$ is
$e^{-n\pi y}$, which decays upward, while the modulus of $e^{-\mathrm{i} n\pi z}$ is $e^{+n\pi y}$,
which decays downward. The integrand carrying $e^{+\mathrm{i} n\pi z}$ is therefore moved upward,
through the rectangle with corners $-1$ and $1+\mathrm{i}\rho$, to the edge at height $\rho$, and the
integrand carrying $e^{-\mathrm{i} n\pi z}$ is moved downward, through the rectangle with corners
$-1-\mathrm{i}\rho$ and $1$, to the edge at height $-\rho$. On its far edge each character has
modulus $e^{-n\pi\rho}$ while the extension stays bounded by $M$, so each shifted integral has
modulus at most $2Me^{-n\pi\rho}$, the factor $2$ being the length of the segment; the prefactor
$\mathrm{i}/2$ halves the sum of the two bounds, and hence
\[
  \lvert a_n(t)\rvert\le 2M e^{-\lambda n},
  \qquad \lambda=\pi\rho,
\]
uniformly for $t\in[0,T]$ and $n\ge1$. This is the only step that uses analyticity, and it is
where the exponent of the error budget comes from.

\emph{Stage two: the truncation tail.} Write $\Pi_{K}(t,x)=\sum_{n=1}^{K}a_n(t)\,e_n(x)$ for the
$K$-mode truncation of the solution, built from the coefficient functions of (A2); this stage
bounds what truncation alone discards. At a fixed time the trace $x\mapsto u^{\star}(t,x)$ is
continuous; it is two-periodic because a classical solution is two-periodic in space by
definition, and it is odd on $\Omega$ because it agrees there with $U(t,\cdot)$, which (A1) makes
odd. Its complex Fourier coefficients on the circle of period two therefore vanish at index zero
and are $\mp\mathrm{i}/2$ times $a_n(t)$ at index $\pm n$. Parseval's identity on that circle
\cite{katznelson2004introduction} then turns the spatial energy of the trace into
$\sum_{n\ge1}a_n(t)^{2}$, and subtracting the prefix carried by $\Pi_{K}$ identifies the residual
energy with the tail $\sum_{n>K}a_n(t)^{2}$. Stage one dominates that tail by a geometric series
with ratio $e^{-2\lambda}$, uniformly in $t$; integrating over $[0,T]$ and taking the root gives
\[
  \lVert u^{\star}-\Pi_{K} \rVert_{T}\le C_{\mathrm{tail}}\,e^{-\lambda K},
  \qquad
  C_{\mathrm{tail}}=\frac{2M\sqrt{T}}{\sqrt{1-e^{-2\lambda}}} ,
\]
the first of the two terms in the error budget.

\emph{Stage three: the coefficients are matched by one network of depth three.} By (A2) each
$a_n$ is smooth on $[0,T]$, so by (A3) the anchored quotient $g_n(t)=(a_n(t)-a_n(0))/t$, closed at
$t=0$ by the one-sided derivative $a_n'(0^{+})$, is smooth on $[0,T]$ and satisfies
$a_n(t)=a_n(0)+t\,g_n(t)$. This is what makes the hypothesis class reachable: the
reparameterization asks the network to represent $g_n$, not $a_n$, and the apparent singularity of
the difference quotient at $t=0$ is removable rather than an obstruction. Rescaling time to the
unit interval and extending smoothly past its endpoints by (A4) produces a function smooth on all
of $\mathbb{R}$; on the fixed interval $[-1,2]$, which contains every enlarged evaluation interval
of the quoted theorem, compactness bounds its derivatives of order at most two. The $\tanh$
approximation theorem then applies at every integer resolution $P>5$ and returns a scalar network
of depth three within $A_n/P^{2}$ of the rescaled quotient uniformly on the unit interval, the
numerator $A_n$ being independent of $P$ but not of $n$; taking $P$ large enough brings the error
below any prescribed tolerance $\delta$. Merging the
rescaling $t\mapsto t/T$ into the first affine layer creates no new layer, so the same network, at
depth three, is within $\delta$ of $g_n$ on $[0,T]$. The first $K$ of these scalar networks share
that depth, so they stack into a single network from one input to $K$ outputs, again of depth
three, whose output of index $n-1$ supplies $Q_n$.

\emph{The balance.} The tolerance $\delta$ measures the anchored quotient: stage three guarantees
$\lvert g_n(t)-Q_n(t)\rvert\le\delta$ for all $1\le n\le K$ and $t\in[0,T]$, so $\delta$ bounds
the distance between $g_n$ and its network surrogate, not the coefficient error itself. That error
follows from the factorization: since (A2) anchors $a_n(0)$ to $b^{\mathrm{IC}}_{n}$, the $n$-th
coefficient of $\Pi_{K}-\widehat{u}_{K}$ is $a_n(t)-a_n(0)-t\,Q_n(t)=t\bigl(g_n(t)-Q_n(t)\bigr)$,
of modulus at most $T\delta$. Orthonormality of the modes on $\Omega$ turns the fixed-time energy
of $\Pi_{K}-\widehat{u}_{K}$ into a sum of $K$ such squares, and integrating over $[0,T]$ and
taking the root gives $\lVert \Pi_{K}-\widehat{u}_{K} \rVert_{T}\le T\sqrt{KT}\,\delta$, the second
term in the error budget. Choosing
\[
  \delta_K=\frac{C_{\mathrm{tail}}\,e^{-\lambda K}}{T\sqrt{KT}}
\]
makes this bound coincide with the tail bound of stage two: the two bounds, not the two errors,
are equal by this choice. The triangle inequality splits the total error into exactly these two
pieces, so the budget is complete, and
\[
  \lVert \widehat{u}_{K}-u^{\star} \rVert_{T}
  \le\lVert u^{\star}-\Pi_{K} \rVert_{T}+\lVert \Pi_{K}-\widehat{u}_{K} \rVert_{T}
  \le 2C_{\mathrm{tail}}\,e^{-\lambda K}.
\]
Taking $K=N$ proves the theorem with $C=2C_{\mathrm{tail}}$ and $\beta=\lambda=\pi\rho$.

\section{Human-intervention inventory}
\label{sec:human_intervention_inventory}

Tables~\ref{tab:human_intervention} and \ref{tab:human_intervention_b} list every problem solved through the pipeline in this work and record whether human intervention occurred and at which stage: conceptualization (formulation, method selection, or redirection before execution) or final evaluation (review of the completed result). Execution itself ran autonomously in all cases. The four DPD workflows solved during family bring-up also serve as the DPD memory seeds (Supplementary Table~\ref{tab:memory_seeds_dpd}).

\section{Ablation and sensitivity studies}
\label{sec:ablation_study}

\subsection{Retrieval baselines: embedding RAG and lexical search}
\label{sec:retrieval_baselines}

Standard retrieval provides a natural alternative to GRAFT fingerprints, but it may prioritize textual resemblance rather than the scientific structure that determines whether a method transfers. We therefore compare GRAFT with two widely used retrieval baselines over the solved-case repository: embedding-based retrieval-augmented generation (RAG)~\cite{lewis2020retrieval} and BM25 lexical search~\cite{beaulieu1997okapi}. We first examine the Rayleigh--B\'enard and Rayleigh--Taylor pair introduced in Fig.~\ref{fig:graft_semantic_geometry}D, for which lexical similarity and scientific similarity disagree.

For embedding RAG, a code-only retrieval baseline, the library contains the source code of every solved Numerical case, represented by the main solver file from each final implementation. These 13 cases span Nektar++, Trixi.jl, and pseudo-spectral realizations. Following standard practice, each file is divided into chunks of 1{,}500 characters with a 200-character overlap and embedded using Google's \texttt{gemini-embedding-001} model. Given only the problem name, such as ``Rayleigh--B\'enard convection,'' retrieval ranks the chunks by cosine similarity to the query embedding. Each case is assigned the score of its highest-ranked chunk.

BM25 searches the same chunks through direct token overlap~\cite{beaulieu1997okapi,robertson2009probabilistic}. We compute its scores with the public \texttt{rank\_bm25} Python library using the default parameters $k_1=1.5$ and $b=0.75$. Unlike embedding RAG, this baseline contains no learned components. The GRAFT arm instead uses the fingerprint Jaccard similarity deployed in the solver loop (Methods), evaluated through the same code path used during solution. Fig.~\ref{fig:retrieval_baselines} reports the resulting rankings in both directions.

The text-based baselines largely follow the problem names. Embedding RAG ranks the lexical name-twin first in both directions: Rayleigh--B\'enard retrieves the compressible Euler Rayleigh--Taylor code, while Rayleigh--Taylor retrieves the incompressible Boussinesq Rayleigh--B\'enard code. BM25 makes the same error for the Rayleigh--B\'enard query. GRAFT instead ranks solver-compatible cases first in both directions and places the lexical name-twin near the middle of each list. Its remaining rankings also follow physical structure: cases from the query's flow regime occupy the leading positions, the other flow regime follows, and both remain closer to one another than to non-flow systems. The least related phase-field systems, Allen--Cahn and Cahn--Hilliard, close both lists.

\subsection{Quantitative top-$k$ retrieval benchmark}
\label{sec:topk_benchmark}

The name-twin comparison isolates one failure mode. To determine whether the same distinction persists across the repository, we benchmark retrieval over the 13 solved Numerical cases with canonical implementations. The protocol mirrors deployment: the library stores solution codes, while an incoming problem is represented by its user request. Each case serves once as the incoming problem. Its own code is removed, its complete formalization-stage request becomes the query, and each retrieval arm ranks the remaining 12 cases.

Embedding RAG and BM25 use the same chunking, embedding, scoring, and best-chunk-per-case aggregation described above. Only the query changes from the problem name to the complete request. This choice favors the text baselines because the formalized request already states the governing equations, domain, and boundary conditions. GRAFT ranks the cases using the fingerprint similarity of Eq.~\eqref{eq:fingerprint_jaccard}, exactly as in the solver loop. Two additional arms establish the attainable range. The random arm gives the exact expectation under a uniformly random ranking, whereas the oracle ranks candidates by their accepted-method overlap, defined below, which no realizable retriever can access.

Each ranking is evaluated against the method ultimately accepted for the query. Let $M_q$ and $M_c$ denote the accepted method-selection node sets for query $q$ and candidate $c$. Their accepted-method overlap, $u(q,c)$, is the Jaccard similarity from Eq.~\eqref{eq:fingerprint_jaccard}. It is high when the candidate was solved using a method similar to the query's accepted solution and equals zero when the two methods share no taxonomy nodes. Because both GRAFT retrieval and this overlap are computed on the same taxonomy, the measure is taxonomy-aligned rather than a ground-truth measure of downstream scientific utility: it does not measure solver attempts saved, final accuracy, computational cost, or successful reuse after implementation. Three complementary metrics evaluate the ranking. Top-1 overlap measures the value of the first suggestion,
\begin{equation}\label{eq:top1_utility}
\text{Top-1}(q) = u(q, c_1).
\end{equation}
Coverage@3 measures how much of the accepted method is represented anywhere among the first three candidates, independently of their order,
\begin{equation}\label{eq:coverage3}
\text{Coverage@3}(q) =
\frac{\lvert M_q \cap \left( M_{c_1} \cup M_{c_2} \cup M_{c_3} \right) \rvert}
     {\lvert M_q \rvert}.
\end{equation}
Finally, NDCG@3 rewards rankings that place the most useful candidates first,
\begin{equation}\label{eq:ndcg3}
\text{NDCG@3}(q) = \frac{\text{DCG@3}(q)}{\text{IDCG@3}(q)},
\qquad
\text{DCG@3}(q) = \sum_{i=1}^{3} \frac{u(q, c_i)}{\log_2(i+1)},
\end{equation}
where IDCG@3 is the DCG@3 obtained from the best possible ordering of the candidates. The oracle therefore attains $\text{NDCG@3}=1$ by construction. Its Top-1 overlap and Coverage@3 remain below 1 because the repository does not contain a perfect method twin for every query. The oracle consequently measures the richness of the available repository and provides the appropriate ceiling for the other retrieval arms.

Across the 13 queries, GRAFT remains close to the oracle ceiling on all three metrics (Table~\ref{tab:topk_benchmark}). It recovers nearly all method nodes that the repository can provide, reaching a coverage of 0.940 against the oracle ceiling of 0.955, and places them close to their optimal order, with an NDCG of 0.916. Embedding RAG remains useful on many queries, but leaves approximately one-fifth of the required method nodes unretrieved and misorders the cases it does recover, despite receiving the governing equations in the query. BM25 fails more fundamentally: its coverage falls below the random floor, from 0.696 to 0.593, because lexical overlap concentrates all three retrieval positions on methodologically unrelated codes rather than distributing them across the repository.

These failures are concentrated where lexical and scientific similarity disagree (Supplementary Table~\ref{tab:topk_extremes}). Embedding RAG ranks a nearly disjoint candidate, one with an accepted-method overlap of at most 0.07, first for 3 of the 13 queries, while BM25 does so for 7. On their worst queries, both methods recover less than 8\% of the accepted method nodes. GRAFT places such a candidate first on a single query, where its top three nevertheless recover 97\% of the accepted method nodes, and even on its worst query it recovers 88\%. That single miss is the Hasegawa--Wakatani drift-wave case, and its cause is repository sparsity rather than ranking noise: no other stored problem shares its plasma physics, so the first-ranked candidate of every arm, Rayleigh--Taylor for GRAFT by problem structure and turbulence-adjacent or lexically similar cases for the text baselines, carries an almost disjoint accepted method. GRAFT's structural ranking is nevertheless the only one that still places the compatible spectral/hp cases within the top three, and its closest neighbors cover 97\% of the accepted method. 

The extremes of Supplementary Table~\ref{tab:topk_extremes} summarize the asymmetry: on the text baselines' best query the three arms are equivalent at the oracle ceiling, whereas on the worst queries GRAFT still recovers nearly all achievable coverage while the text baselines collapse. Retrieval based on wording therefore fails most severely where informed transfer is most valuable, whereas retrieval based on scientific structure degrades more gradually. Yet retrieval accuracy is only part of the distinction. Even a perfect top-$k$ ranking would not reproduce GRAFT: RAG and BM25 return ordered documents, but their scores do not encode dependencies among method choices, combine evaluated outcomes into a prior, or define probabilities over admissible actions. Adding cases therefore enlarges the index without organizing the new scientific dependencies, while the growing number of plausible distractors can degrade retrieval. Embedding RAG also represents similarity through a learned, opaque space, so a successful match does not reveal which scientific characteristics justify transfer. GRAFT instead derives its coordinates and dependencies from explicit scientific semantics and converts the retrieved neighborhood into similarity- and reward-weighted action probabilities. Repository growth therefore expands an interpretable problem--method structure rather than only a black-box index.

\subsection{Sensitivity to the level weight}

The fingerprint similarity of Eq.~\eqref{eq:fingerprint_jaccard} assigns a weight $\omega(d)$ to each tree depth. Uniform weighting treats all activated characteristics equally, whereas decaying schedules emphasize higher-level scientific characteristics over finer refinements. We evaluated the sensitivity of retrieval to this choice using leave-one-out queries within each family. Under each schedule, problem-fingerprint similarity ranked the remaining cases, and the resulting rankings were evaluated by Top-1 overlap, Coverage@3, and NDCG@3 (Eqs.~\eqref{eq:top1_utility}--\eqref{eq:ndcg3}). Only the problem-fingerprint weights used for retrieval were varied; the accepted-method overlap $u(q,c)$ used for evaluation remained fixed under its deployed uniform weighting. An oracle that ranks candidates directly by this overlap provides a weight-independent ceiling. Numerical uses the same 13-query set as Table~\ref{tab:topk_benchmark}, while PIML, DPD, and Formal contribute 8, 19, and 25 queries, respectively (Supplementary Table~\ref{tab:weight_sensitivity}).

The benefit of level weighting increases with the development of the encoded dependency structure rather than with raw tree depth alone. Numerical and PIML share the mature PDE problem tree and contain the two richest corresponding action spaces: 156 and 117 decision chains with 78 and 17 symbolic rules, respectively, compared with 37 and 30 chains and 6 and 0 rules for DPD and Formal (Supplementary Table~\ref{tab:policy_footprint_anatomy}). In these developed families, agreement on numerous fine refinements can otherwise outweigh a disagreement in the higher-level problem class. Emphasizing the shallow levels corrects this imbalance. Under $\omega(d)=10^{-d}$, a node receives ten times the weight of a node one level deeper, and mean Top-1 overlap rises from 0.742 to 0.804 for Numerical and from 0.762 to 0.802 for PIML. Coverage@3 is maintained or improved, while NDCG@3 remains unchanged for Numerical and increases from 0.970 to 0.984 for PIML.

By contrast, every DPD and Formal score remains within 0.007 of its uniform-weight value, consistent with their less developed dependency structures and with fine-level characteristics remaining directly informative. Level weighting is therefore most useful once a family contains enough hierarchical refinements and dependencies for local matches to obscure higher-level scientific differences. Uniform weighting does not privilege a particular depth, making it a stable default for the still-developing DPD and Formal spaces. We therefore deploy $\omega\equiv1$ across all four families for consistency, while the sensitivity results show that mature spaces can benefit from stronger emphasis on their upper levels. Because $\omega$ is applied only when fingerprints are compared, this calibration can be updated independently for each family without re-encoding the stored cases.

\subsection{Common comparison protocol}

In this section, we compare GRAFT-guided workflows with two commercial agentic baselines executed through Anthropic's Claude Code command-line harness using the pinned model identifiers \texttt{claude-opus-4-8} (Opus 4.8) and \texttt{claude-fable-5} (Fable 5); the Burgers case also includes ATHENA without GRAFT. The model identifiers for every GRAFT-guided agent role are listed in Supplementary Table~\ref{tab:model_identifiers}. For each task, the Formalization team generated a case-specific request that was supplied unchanged to every evaluated system, including the commercial and ATHENA-family frameworks. All systems had unrestricted computer access, internet search, scientific software, code execution, and permission to create subagents and their own validation procedures. These comparisons are designed to expose possible failure modes in the audited trajectories, not to estimate their frequency across all possible runs; they report one audited trajectory per model and task configuration, plus five independent replicates of the Fable/no-wake configuration (Supplementary Table~\ref{tab:apollo_replicates}). The exact requests are reproduced at the beginning of the corresponding cases. Together, the studies in this section ablate three layers of the system: the action-space sampling study removes the LLM agents and isolates the tree structure, dependency rules, and priors; the effect-of-GRAFT study isolates the persistent cross-problem memory against prompt-scaffolded within-problem learning; and the commercial comparisons remove the substrate entirely. Arms isolating the node hints and the reward calibration individually are left to future work.

\subsection{Viscous Burgers case}
\label{sec:burgers_ablation}

\subparagraph*{User request}
\label{sec:burgers_user_request}

The following request was supplied to each system. The wording is reproduced from the original request, with its Markdown formatting adapted to LaTeX.

\begin{quote}
\small
\emph{Viscous Burgers 1D (Conservative Form)}

\emph{Equation.} Solve the 1D viscous Burgers equation in conservative (flux-divergence) form:
\begin{equation*}
u_t + \partial_x \left( \tfrac{1}{2}\,u^2 - \varepsilon\,u_x \right) = 0,
\qquad x \in [-1,1],\quad t \in [0,1].
\end{equation*}
For constant $\varepsilon$ this is mathematically equivalent to the strong (non-conservative) form
\begin{equation*}
u_t + u\,u_x - \varepsilon\,u_{xx} = 0.
\end{equation*}

\emph{Domain.}
\begin{itemize}
\item Spatial: $x\in[-1,1]$.
\item Temporal: $t\in[0,1]$.
\item Periodic in $x$ with period $2$.
\end{itemize}

\emph{Initial condition.}
\begin{equation*}
u(0,x)=-\sin(\pi x).
\end{equation*}

\emph{Boundary conditions.} Periodic in $x$:
\begin{equation*}
u(t,-1)=u(t,+1), \qquad t\in[0,1].
\end{equation*}
Equivalently, $u$ is $2$-periodic in $x$ for all $t$.

\emph{Parameters.}
\begin{itemize}
\item Viscosity $\varepsilon=\dfrac{1}{100\pi}\approx3.183\times10^{-3}$.
\item Reynolds number $\mathrm{Re}=UL/\varepsilon\approx628$ (with $U=1$, $L=2$).
\item Canonical Raissi setup. The viscous term smooths the inviscid shock that would otherwise form near $t\approx1/\pi\approx0.318$ into a thin internal layer.
\end{itemize}

\emph{What to find.} The space--time field $u(t,x)$ on $[0,1]\times[-1,1]$, including the steep internal layer that develops near $t\approx1/\pi$.

\emph{Reference data.} A precomputed reference solution is available at
\begin{center}
\path{Library/Data/PIML/Viscous_Burgers_1D_20260505_v2/burgers.mat}.
\end{center}
The MATLAB \texttt{.mat} file contains:
\begin{itemize}
\item \texttt{t}: shape \texttt{(1, 1001)}, \texttt{float64}; time vector on $[0,1]$, $N_t=1001$.
\item \texttt{x}: shape \texttt{(1, 2001)}, \texttt{float64}; spatial vector on $[-1,1]$, $N_x=2001$.
\item \texttt{usol}: shape \texttt{(1001, 2001)}, \texttt{float64}; reference field $u(t,x)$.
\end{itemize}
Transpose \texttt{usol} to \texttt{(Nx, Nt)} for a \texttt{meshgrid(t, x)} layout.
\end{quote}

\subparagraph*{Results}

The reference fields for all four PIML benchmarks in Table~\ref{tab:sota_comparison_graft} come from analytical solutions. The first comparison used the canonical periodic viscous Burgers equation with $\nu=1/(100\pi)$. For viscous Burgers, the reference data specified above were generated by numerically evaluating the exact Cole--Hopf solution~\cite{cole1951quasilinear,hopf1950partial}. This case additionally includes ATHENA without GRAFT, which receives expert scientific and refinement guidance through its system prompts but begins each new problem from first principles~\cite{toscano2025athena}. GRAFT encodes that guidance in persistent problem and method graphs and retrieves both successful actions and their measured outcomes, conditioning the action prior and the informed reward for the new target. 

The commercial agents produced technically strong solvers from scratch. Both built PyTorch PINNs, enforced periodicity through Fourier features~\cite{wang2021eigenvector}, and used Adam followed by L-BFGS; Opus additionally imposed the initial condition exactly and used causal and residual-adaptive sampling~\cite{wu2023comprehensive}. These are sensible, widely used choices. Fable reached a relative $L^2$ error of $2.03\times10^{-4}$, while Opus reached $1.50\times10^{-5}$. The results are credible, but they remain well above the errors attained by specialized methods on the same problem.

The first failure mode is the starting point. A broad action space led the commercial agents to robust, familiar configurations. Their models used tens to hundreds of thousands of trainable parameters, whereas ATHENA and GRAFT+ATHENA used only a few thousand. Model size is not itself the cause, but the additional capacity did not reveal the decisive strategy. ATHENA's prompt-based scaffolding narrowed the action space enough to select SSBroyden~\cite{urban2024unveiling} and vRBA~\cite{toscano2026variational}, reaching $3.79\times10^{-8}$. GRAFT+ATHENA also selected SSBroyden, so the remaining improvement came from Reynolds-number continuation. Its prior aggregates methods across the retrieved neighborhood; the displayed rank-one member is inviscid Burgers, where ATHENA had previously discovered a longer continuation schedule~\cite{toscano2025athena}. GRAFT+ATHENA adapted that experience to three viscosity stages. Accordingly, its error remains near $10^{-1}$ while the intermediate equations are solved, then falls abruptly when the requested viscosity is reached near $140{,}000$ iterations, ultimately attaining $1.78\times10^{-9}$ (Fig.~\ref{fig:semantic_memory_learning}B).

The second failure mode is validation and reward calibration. The commercial agents performed meaningful checks: Fable used a subagent to reproduce the reference solution near machine precision, while Opus independently checked it through the Cole--Hopf solution, conservation, symmetry, and a held-out residual. These checks established a trustworthy reference and physically plausible solutions, but they did not establish either that optimization had stabilized or that the attained accuracy was competitive. Both commercial validation-error histories terminate without a clear plateau (Fig.~\ref{fig:semantic_memory_learning}B). This does not imply that extending the same L-BFGS runs would recover the GRAFT+ATHENA result; L-BFGS is a standard method and may simply approach its own method-dependent floor. The relevant failure is that termination of the inner optimizer was accepted without an independent convergence criterion. ATHENA and GRAFT+ATHENA instead separate optimizer termination from scientific stopping by running repeated SSBroyden cycles with vRBA resampling and monitoring convergence across cycles. The comparison standard created a second blind spot. Fable judged its $2.03\times10^{-4}$ error against the early PINN result of $6.7\times10^{-4}$~\cite{raissi2019physics}, despite subsequent methods reaching the $10^{-8}$ range (Table~\ref{tab:sota_comparison_graft}); Opus found that its final L-BFGS phase reduced the optimized residual while slightly worsening the global field error. The GRAFT-guided workflow instead derives expected accuracy and efficiency from the measured outcomes of related problems and combines them with explicit integrity gates. Persistent experience therefore changes both the methods considered and the criteria used to determine when a result is complete and competitive.

\subsection{Effect of GRAFT}
\label{sec:graft_effect_ablation}

\subparagraph*{User requests}
\label{sec:graft_effect_user_requests}

The three requests below were supplied unchanged to both systems, as the viscous Burgers request above was. The wording is reproduced from the original requests, with their Markdown formatting adapted to LaTeX.

\begin{quote}
\small
\emph{Poisson 2D (manufactured solution, Dirichlet)}

\emph{Equation.} Solve the steady 2D linear elliptic Poisson equation
\begin{equation*}
u_{xx}(x,y)+u_{yy}(x,y)+f(x,y)=0.
\end{equation*}

\emph{Domain.}
\begin{itemize}
\item Spatial: $(x,y)\in[0,1]^{2}$ (unit square).
\item Steady problem, no time dependence.
\end{itemize}

\emph{Boundary conditions.} Homogeneous Dirichlet on all four edges:
\begin{equation*}
u=0 \quad\text{on }\partial\Omega.
\end{equation*}

\emph{Forcing.}
\begin{equation*}
f(x,y)=32\pi^{2}\,\sin(4\pi x)\,\sin(4\pi y).
\end{equation*}
The forcing is chosen so that the problem admits the manufactured analytical solution
\begin{equation*}
u_{\text{an}}(x,y)=\sin(4\pi x)\,\sin(4\pi y).
\end{equation*}

\emph{What to find.} The scalar field $u(x,y)$ on $[0,1]^{2}$.

\emph{Reference data.} No external dataset. The reference solution is the analytical expression\linebreak $u_{\text{an}}(x,y)=\sin(4\pi x)\sin(4\pi y)$, evaluated in-line on a $100\times100$ uniform grid covering $[0,1]^{2}$ for error reporting.
\end{quote}

\begin{quote}
\small
\emph{2D Helmholtz equation}

\emph{Equation.}
\begin{equation*}
u_{xx}+u_{yy}+k^{2}u=g(x,y),
\end{equation*}
where the forcing term $g(x,y)$ is manufactured so the equation admits the analytical solution
\begin{equation*}
u_{\text{an}}(x,y)=\sin(\pi a_1 x)\,\sin(\pi a_2 y).
\end{equation*}
The forcing $g(x,y)$ is computed via automatic differentiation from $u_{\text{an}}$, that is, $g=\partial_{xx}u_{\text{an}}+\partial_{yy}u_{\text{an}}+k^{2}u_{\text{an}}$, and is not specified in closed form.

\emph{Domain.} $(x,y)\in[-1,1]\times[-1,1]$.

\emph{Time interval.} Not applicable, the problem is steady.

\emph{Initial condition.} Not applicable.

\emph{Boundary conditions.} Periodic in both $x$ and $y$:
\begin{itemize}
\item $u(-1,y)=u(1,y)$, $u_x(-1,y)=u_x(1,y)$.
\item $u(x,-1)=u(x,1)$, $u_y(x,-1)=u_y(x,1)$.
\end{itemize}

\emph{Parameters.} $a_1=1$, $a_2=4$, $k=1$.

\emph{Goals.}
\begin{itemize}
\item Recover $u(x,y)$ using a PINN.
\item Report the relative $L^{2}$ error against the analytical solution $u_{\text{an}}$.
\end{itemize}
\end{quote}

\begin{quote}
\small
We are solving a forward problem for the Korteweg--de Vries (KdV) equation:
\begin{equation*}
u_t+6uu_x+u_{xxx}=0.
\end{equation*}

\emph{Domain.}
\begin{itemize}
\item Spatial: $x\in[0,20]$.
\item Temporal: $t\in[0,5]$.
\end{itemize}

\emph{Boundary conditions.}
\begin{align*}
u(0,x)&=g_0(x) &&\text{(initial condition)},\\
u(t,0)&=g_1(t) &&\text{(BC at }x=0\text{)},\\
u(t,20)&=g_2(t) &&\text{(BC at }x=20\text{)},\\
u_x(t,20)&=g_3(t) &&\text{(Neumann BC at }x=20\text{)}.
\end{align*}

\emph{Analytical solution.} The solution is given by a two-soliton interaction formula,
\begin{equation*}
u(x,t)=\frac{2(c_1-c_2)\left[c_1\cosh^{2}\!\left(\tfrac{\sqrt{c_2}\,\zeta_2}{2}\right)+c_2\sinh^{2}\!\left(\tfrac{\sqrt{c_1}\,\zeta_1}{2}\right)\right]}
{\left[\left(\sqrt{c_1}-\sqrt{c_2}\right)\cosh\!\left(\tfrac{\sqrt{c_1}\,\zeta_1+\sqrt{c_2}\,\zeta_2}{2}\right)+\left(\sqrt{c_1}+\sqrt{c_2}\right)\cosh\!\left(\tfrac{\sqrt{c_1}\,\zeta_1-\sqrt{c_2}\,\zeta_2}{2}\right)\right]^{2}},
\end{equation*}
where $\zeta_1=x-c_1t-x_1$ and $\zeta_2=x-c_2t-x_2$.

\emph{Parameters.}
\begin{itemize}
\item $c_1=6.0$ (velocity of first soliton).
\item $c_2=2.0$ (velocity of second soliton).
\item $x_1=-2.0$ (initial position of first soliton).
\item $x_2=2.0$ (initial position of second soliton).
\end{itemize}

The boundary conditions $g_0$, $g_1$, $g_2$ and $g_3$ are derived by evaluating the analytical solution at the appropriate boundaries.
\end{quote}

\subparagraph*{Results}
\label{sec:graft_effect_results}

This ablation isolates what persistent memory adds to ATHENA's existing capacity for within-problem learning. ATHENA was designed as a contextual-bandit-inspired loop in which execution feedback conditions the next scientific action; its original evaluation reported submartingale-like reward dynamics and sublinear cumulative regret~\cite{toscano2025athena}. Here, we record both the reward assigned by each loop and a common external measure, the relative $L^2$ error against the same reference solution. The former shows how each system's search progresses, whereas the latter prevents faster acceptance from being mistaken for better performance. ATHENA's short-term history can improve a solution during one trial, but it is reset for the next problem. GRAFT adds the solved-case repository, allowing previous outcomes to condition the action prior and evaluation standards of later targets.

ATHENA behaves as intended: after unscored execution--repair attempts and occasional reversals, its reward generally rises while its external error falls (Fig.~\ref{fig:loop_efficiency}). Reaching acceptance nevertheless required $15$, $16$, $10$, and $9$ total attempts for viscous Burgers, Poisson, Helmholtz, and KdV, respectively. GRAFT+ATHENA accepted Burgers, Helmholtz, and KdV on the first attempt and Poisson on the third, while reaching a lower relative $L^2$ error in every case. GRAFT therefore does not replace ATHENA's iterative learning; it moves the loop closer to a successful part of the action space before the first execution.

The Poisson--Helmholtz sequence makes this transfer visible. Poisson was attempted without a closely matched solved record and required three GRAFT+ATHENA attempts. Once stored, it became the highest-ranked neighbor for Helmholtz and contributed its successful actions and measured outcome, together with the other retrieved records, to the aggregated prior. Helmholtz was then accepted in one attempt, whereas ATHENA without GRAFT required ten. ATHENA can learn an effective strategy when given enough iterations; GRAFT makes that learning cumulative, so a related problem does not pay the same search cost again.

\subsection{Apollo case}
\label{sec:apollo_ablation}

\subparagraph*{User request}
\label{sec:apollo_user_request}

The two Apollo tasks used the same request. The underlined text in parentheses was included only in the wake-requested version; omitting the two underlined parentheses gives the no-wake request. All remaining text was shared.

\begin{quote}
\small
\emph{Apollo Command Module Forebody, Hypersonic Inviscid Flow}

Integrate until $t=5$ and evaluate the final state.

\emph{Problem.} Simulate the hypersonic inviscid flow over the forebody (heatshield) of an Apollo-type Command Module reentry capsule at a representative trajectory point (wind-tunnel-class Mach number, perfect-gas air). The goal of this first pass is to ``replicate the flow---resolve the detached bow shock, the subsonic stagnation region, and the expansion around the shoulder \underline{(, and the onset of the downstream wake along the capsule symmetry axis)}.''
Validation against experimental or reference CFD data is deferred to a later iteration.

\emph{Governing equations.} Two-dimensional \emph{axisymmetric compressible Euler} equations for a perfect gas:
\begin{equation*}
\partial_t \mathbf{U}+\partial_z\mathbf{F}(\mathbf{U})+\partial_r\mathbf{G}(\mathbf{U})+\mathbf{S}(\mathbf{U},r)=0,
\end{equation*}
with conserved variables and fluxes
\begin{equation*}
\mathbf{U}=
\begin{pmatrix}\rho\\ \rho u_z\\ \rho u_r\\ \rho E\end{pmatrix},\quad
\mathbf{F}=
\begin{pmatrix}\rho u_z\\ \rho u_z^2+p\\ \rho u_z u_r\\ (\rho E+p)u_z\end{pmatrix},\quad
\mathbf{G}=
\begin{pmatrix}\rho u_r\\ \rho u_z u_r\\ \rho u_r^2+p\\ (\rho E+p)u_r\end{pmatrix},
\end{equation*}
and the axisymmetric source term
\begin{equation*}
\mathbf{S}=\frac{1}{r}
\begin{pmatrix}\rho u_r\\ \rho u_z u_r\\ \rho u_r^2\\ (\rho E+p)u_r\end{pmatrix},
\end{equation*}
closed by the perfect-gas equation of state
\begin{equation*}
p=(\gamma-1)\left(\rho E-\tfrac{1}{2}\rho(u_z^2+u_r^2)\right),\qquad \gamma=1.4.
\end{equation*}

\emph{Geometry---Apollo Command Module forebody.} Meridional cross-section of the heatshield, parameterized in $(z,r)$ with $z$ along the symmetry axis (flow direction) and $r$ the radial coordinate. All dimensions below are taken from the Apollo CM ``Symmetrical Smooth Heat Shield'' drawing in AEDC-TR-67-238 (Griffith \& Boylan 1968), Fig.~2.
\begin{itemize}
\item \emph{Heatshield:} spherical segment of nose radius $R_n=4.694\ \mathrm{m}$ (184.8 in), joined to the afterbody at the base radius $R_b=1.956\ \mathrm{m}$ (base diameter 154 in $=3.912$ m).
\item \emph{Shoulder fillet:} toroidal fillet of radius $r_c=0.196\ \mathrm{m}$ (7.7 in) joining the heatshield to the conical afterbody, i.e. a smooth rounded corner at the shoulder.
\item \emph{Conical afterbody:} $33^\circ$ half-angle cone (measured from the symmetry axis), aft-truncated at a radius of $\approx0.233\ \mathrm{m}$ (9.155 in).
\item \emph{Overall axial length:} the AEDC drawing's 142.57 in is the axial station of the \emph{theoretical sharp cone apex}---the virtual point where the $33^\circ$ afterbody cone would close on the axis if it were not truncated. It is a drafting datum, \emph{not} the length of the truncated body. The truncated body runs from the stagnation point to $z=3.266\ \mathrm{m}$ (128.58 in); the difference is exactly $r_\mathrm{trunc}/\tan 33^\circ=0.233/0.6494=0.359\ \mathrm{m}$, the cone tip removed by the truncation. Reconstructing the virtual apex from the radii above gives 142.687 in against the quoted 142.57 in, i.e. agreement to 0.08\% (drawing round-off).
\end{itemize}
The geometry is normalized so that the nose stagnation point is at $(z,r)=(0,0)$ and the flow direction is $+z$.

\textit{Exact normalized form.} In units of the base radius $R_b$ the shape constants are exactly rational, so building from these ratios rather than from the rounded metre values removes an $O(10^{-4})$ inconsistency ($4.694/1.956=2.39980$, not $2.4$):
\begin{equation*}
R_n=2.4R_b,\qquad r_c=0.1R_b,\qquad R_b=1.
\end{equation*}
The \emph{maximum body radius is attained on the fillet}, not at the sphere/fillet join, at the point where the surface normal is purely radial. Requiring that maximum to equal $R_b$ is what pins the whole shape together and fixes the sphere/fillet tangency angle in closed form:
\begin{equation*}
\sin\theta_t=\frac{R_b-r_c}{R_n-r_c}=\frac{0.9}{2.3}=\frac{9}{23}
\quad\Longrightarrow\quad \theta_t=23.0357^\circ.
\end{equation*}
The forebody is then three tangent-continuous segments, with the sphere centred at $(R_n,0)$ and the fillet centre at $C_f=(0.283399,0.9)$ exactly (its radius is $R_b-r_c=0.9$):
\begin{center}
\small
\setlength{\tabcolsep}{3pt}
\begin{tabularx}{\linewidth}{@{}>{\raggedright\arraybackslash}p{2.5cm}>{\raggedright\arraybackslash}p{3.0cm}>{\raggedright\arraybackslash}X>{\raggedright\arraybackslash}X@{}}
\toprule
segment & parameter & from & to \\
\midrule
spherical heatshield & $\theta\in[0^\circ,23.0357^\circ]$ & $(0,0)$ & $(0.191373,0.939130)$ \\
shoulder fillet & $\psi\in[23.0357^\circ,123^\circ]$ & $(0.191373,0.939130)$ & $(0.337863,0.983867)$ \\
conical afterbody & $\mathrm{d}r/\mathrm{d}z=-\tan33^\circ$ & $(0.337863,0.983867)$ & $(1.669801,0.118896)$ \\
\bottomrule
\end{tabularx}
\end{center}
Here $\theta$ is the polar angle about the sphere centre and $\psi$ is the outward-normal angle on the fillet, with surface point $C_f+r_c(-\cos\psi,\sin\psi)$; $\psi$ runs to $90^\circ+33^\circ=123^\circ$ at the cone tangency. The cone reaches the axis (the virtual apex above) at $z=1.852885R_b$.

\emph{Computational domain.} Half-plane $\{(z,r):r\geq0\}$ exterior to the capsule forebody. Boundaries: inflow, far-field, outflow, capsule surface, symmetry axis $r=0$.

\emph{Boundary conditions.}
\begin{itemize}
\item \emph{Capsule surface:} inviscid slip wall, $\mathbf{u}\cdot\hat{\mathbf{n}}=0$.
\item \emph{Symmetry axis} ($r=0$): axisymmetric symmetry condition, $u_r=0$, zero normal gradient of $\rho$, $u_z$, $p$.
\item \emph{Inflow/far-field (upstream and outer):} supersonic Dirichlet inflow set to the freestream state $(\rho_\infty,u_{z,\infty},u_{r,\infty},p_\infty)$ defined below.
\item \emph{Outflow (downstream):} supersonic extrapolation (zero-gradient)---the flow is supersonic in the aft expansion region for the parameters of interest.
\end{itemize}

\emph{Initial condition.} Uniform freestream everywhere in the domain (including inside the forebody bow-shock region---it will be set up transiently as the flow develops):
\begin{equation*}
\mathbf{U}(z,r,t=0)=\mathbf{U}_\infty,\qquad
\mathbf{u}_\infty=(u_{z,\infty},u_{r,\infty})=(u_\infty,0).
\end{equation*}

\emph{Parameters (non-dimensional freestream).}
\begin{itemize}
\item Freestream Mach number: $M_\infty=10.0$ (representative hypersonic wind-tunnel-class condition; the AEDC Tunnel L Apollo CM campaign covered $M=9.37$ and $M=10.15$, so $M=10$ sits inside that range without claiming a specific facility run).
\item Angle of attack: $\alpha=0^\circ$ in our convention (flow direction $+z$, capsule axis aligned with flow, heatshield facing the freestream---i.e. the axisymmetric limit). \emph{Note on convention:} the AEDC reference (Griffith \& Boylan 1968) uses $\alpha=180^\circ$ for this same heatshield-forward attitude (see their Fig.~2a). Our $\alpha=0^\circ$ is equivalent to their $\alpha=180^\circ$; the physics is identical, only the labeling differs. The actual AS-202 flight trimmed at $\alpha_T\approx162.5^\circ$ in AEDC convention (i.e. $\sim17.5^\circ$ off the symmetry axis, producing $L/D\approx0.3$). The symmetric case specified here is the idealized axisymmetric limit, not a simulation of the real AS-202 trajectory.
\item Specific heat ratio: $\gamma=1.4$ (perfect gas). Real-gas/equilibrium-air effects present at true reentry conditions ($M\gtrsim20$) are explicitly \emph{out of scope} for this first pass.
\item Non-dimensionalization: $\rho_\infty=1$, $p_\infty=1/\gamma$, so $a_\infty=1$ and $u_\infty=M_\infty=10.0$.
\end{itemize}
Lengths are non-dimensionalized by $R_b$ unless stated otherwise.

\emph{Time integration/final state.} Integrate in time and stop at $t=5$.

\emph{Deliverables (first pass---replication only).}
\begin{itemize}
\item Flow field $(\rho,u_z,u_r,p,M)$ at $t=5$ around the Apollo forebody.
\item Visualization of the bow shock, subsonic region behind the shock, and shoulder expansion \underline{(, and onset of the downstream wake)}.
\item Surface pressure distribution along the heatshield (as a function of arclength or polar angle from the stagnation point).
\item Bow-shock stand-off distance $\Delta$ along the stagnation streamline.
\end{itemize}

\emph{Sources (geometry).} Apollo Command Module dimensions used above were cross-checked against publicly available references:
\begin{itemize}
\item Apollo command and service module---Wikipedia (base diameter 3.91 m; $33^\circ$ afterbody cone). \url{https://en.wikipedia.org/wiki/Apollo_command_and_service_module}
\item {\sloppy NASA TN D-7564 and related NTRS entries on the Apollo CM geometry (heatshield spherical radius 469.4 cm/4.694 m). \url{https://ntrs.nasa.gov/api/citations/19740007423/downloads/19740007423.pdf}\par}
\item AEDC-TR-67-238, Griffith \& Boylan (1968), \textit{Postflight (AS-202) Apollo Command Module Aerodynamic Simulation Tests}---Fig.~2 ``Apollo CM Forces and Dimensions'' provides the symmetrical smooth heat-shield geometry ($R_n=184.8$ in, base diameter $154$ in, $33^\circ$ afterbody cone, shoulder fillet $7.7$ in, aft truncation $9.1$ in, overall axial length $142.57$ in). The Tunnel L campaign in the same report covered $M=9.37$ and $M=10.15$, the range our $M_\infty=10$ sits inside. \url{https://apps.dtic.mil/sti/tr/pdf/AD0666927.pdf}
\end{itemize}
The specific choice of computational domain extents, mesh topology, and limiter are \emph{implementer choices}---not problem-defining parameters---and are deliberately left unfixed here. All geometric dimensions above, by contrast, are fixed by the AEDC drawing and should be honored.
\end{quote}

\subparagraph*{Coupled numerical constraints}
\label{sec:apollo_dependencies}

The Apollo task couples four numerical requirements. The shock treatment must preserve positive density and pressure through the Mach-10 stagnation region and remain compatible with the mesh strategy; the axisymmetric reduction requires a custom geometric source term; and all implementation choices must preserve the fixed capsule geometry. These dependencies bind the complete method chain: a positivity treatment suitable for a stationary body-fitted mesh may be incompatible with runtime refinement, while altering the geometry or domain can produce a converged solution to a different physical problem. The resulting configurations are reported in Supplementary Table~\ref{tab:apollo_setup}.

\subparagraph*{Aerodynamic validation metrics}
\label{sec:aerodynamic_metrics}

For a closed body boundary $\Gamma_b$, the pressure-derived axial-force coefficient is defined by
\begin{equation}
-C_A(t)=\frac{1}{q_\infty A_{\mathrm{ref}}}
\int_{\Gamma_b}p(\mathbf{x},t)\,
\mathbf{n}_\Omega(\mathbf{x})\!\cdot\!\mathbf{e}_A\,\mathrm{d}S,
\qquad
q_\infty=\frac{1}{2}\rho_\infty U_\infty^2,
\label{eq:axial_coefficient_general}
\end{equation}
where $p$ is the surface pressure, $\mathbf{n}_\Omega$ is the unit normal directed outward from the computational fluid domain, $\mathbf{e}_A$ is the axial unit vector, and $\mathrm{d}S$ is the surface-area element. The freestream density and speed are $\rho_\infty$ and $U_\infty$, while $A_{\mathrm{ref}}$ is the chosen reference area. The leading minus sign follows the axial-loading convention used by the AEDC report. Eq.~\eqref{eq:axial_coefficient_general} uses the standard aerodynamic coefficient normalization~\cite{anderson1989hypersonic}. For an axisymmetric surface generated by a meridional curve with radial coordinate $r$ and arclength $s$, $\mathrm{d}S=2\pi r\,\mathrm{d}s$. The degree-four Lobatto--Legendre implementation therefore evaluates the integral at five surface nodes per boundary face as
\begin{equation}
-C_A(t)\approx\frac{1}{q_\infty A_{\mathrm{ref}}}
\sum_{f\subset\Gamma_b}\sum_{\ell=1}^{5}
p_{f\ell}(t)\,w_\ell\,2\pi r_{f\ell}\,\sigma_f
\left(\frac{\partial r_f}{\partial\xi}\right)_{\!\ell},
\qquad
\mathbf{w}=\left(\frac{1}{10},\frac{49}{90},\frac{32}{45},\frac{49}{90},\frac{1}{10}\right),
\label{eq:axial_coefficient_discrete}
\end{equation}
where $f$ indexes body-boundary faces, $\ell$ indexes their quadrature nodes, $\xi$ is the reference-face coordinate, $p_{f\ell}$ and $r_{f\ell}$ are the nodal pressure and radius, $w_\ell$ is the quadrature weight, and $\sigma_f$ supplies the outward-fluid orientation. For the Apollo calculation, $\rho_\infty=1$, $U_\infty=10$, $q_\infty=50$, and $A_{\mathrm{ref}}=\pi R_b^2$ with maximum base radius $R_b=1.956\ \mathrm{m}$. If $N_{4:5}$ denotes the number of stored snapshots in $t\in[4,5]$, the late-time mean and its difference from the AEDC value are
\begin{equation}
\overline{-C_A}=\frac{1}{N_{4:5}}\sum_{t_n\in[4,5]}-C_A(t_n)=1.462249,
\qquad
100\left(\frac{\overline{-C_A}}{1.44}-1\right)=1.55\%,
\label{eq:apollo_axial_coefficient}
\end{equation}
where $1.44$ is read from the high-Reynolds-number plateau of the AEDC correlation~\cite{griffith1968postflight}.

For a calorically perfect gas, the Rayleigh--Pitot stagnation pressure downstream of a normal shock is
\begin{equation}
p_{0,2}^{\mathrm{RP}}=p_\infty
\left[\frac{(\gamma+1)^2M_\infty^2}
{4\gamma M_\infty^2-2(\gamma-1)}\right]^{\!\gamma/(\gamma-1)}
\left[\frac{1-\gamma+2\gamma M_\infty^2}{\gamma+1}\right],
\label{eq:rayleigh_pitot}
\end{equation}
where $p_\infty$ is the freestream static pressure, $M_\infty$ is the freestream Mach number, $\gamma$ is the ratio of specific heats, and $p_{0,2}^{\mathrm{RP}}$ is the total pressure obtained after the normal shock and subsequent isentropic deceleration~\cite{ames1953equations}. In the discrete capsule solution, the stagnation pressure is estimated from the native pressure values at the apex,
\begin{equation}
\begin{aligned}
\mathcal{I}_{\mathrm{apex}}
&=\left\{(f,\ell):\sqrt{z_{f\ell}^{2}+r_{f\ell}^{2}}<10^{-10}\right\},
&p_{\mathrm{stag}}(t)
&=\frac{1}{|\mathcal{I}_{\mathrm{apex}}|}
\sum_{(f,\ell)\in\mathcal{I}_{\mathrm{apex}}}p_{f\ell}(t),\\
\overline{p}_{\mathrm{stag}}
&=\frac{1}{N_{4:5}}\sum_{t_n\in[4,5]}p_{\mathrm{stag}}(t_n)=92.726372,
&100\left(\frac{\overline{p}_{\mathrm{stag}}}{p_{0,2}^{\mathrm{RP}}}-1\right)
&=+0.46\%.
\end{aligned}
\label{eq:apollo_stagnation_pressure}
\end{equation}
Here $z_{f\ell}$ is the axial coordinate of a surface node and $|\mathcal{I}_{\mathrm{apex}}|$ is the number of coincident apex values. For $\gamma=1.4$, $M_\infty=10$, and $p_\infty=1/\gamma$, Eq.~\eqref{eq:rayleigh_pitot} gives $p_{0,2}^{\mathrm{RP}}=92.2978$.

Entropy supplies a further check that requires no external aerodynamic reference. In the nondimensionalization above, the specific entropy in units of $c_v$ is $s=\ln(p/\rho^\gamma)$ up to an additive constant, so the entropy rise relative to the freestream is $\Delta s=\ln(p/\rho^{\gamma})-\ln(p_\infty/\rho_\infty^{\gamma})$. Across a normal shock, the Rankine--Hugoniot jump conditions fix this rise exactly,
\begin{equation}
\Delta s_{\mathrm{NS}}
=\ln\!\left[\frac{2\gamma M_\infty^{2}-(\gamma-1)}{\gamma+1}\right]
-\gamma\ln\!\left[\frac{(\gamma+1)M_\infty^{2}}{(\gamma-1)M_\infty^{2}+2}\right],
\label{eq:entropy_floor}
\end{equation}
which gives $\Delta s_{\mathrm{NS}}=2.318$ for $\gamma=1.4$ and $M_\infty=10$. Fluid adjacent to the forebody wall has crossed the near-normal portion of the bow shock, the shoulder expansion is isentropic, and dissipation can only raise entropy along a streamline, so the wall layer must satisfy $\Delta s\ge\Delta s_{\mathrm{NS}}$. The wall-band diagnostic reported in Supplementary Table~\ref{tab:apollo_diagnostics} evaluates $\Delta s$ at $400$ stations along each run's own wetted wall, at wall-normal offsets $\delta\in[0.01,0.10]\,R_b$, and records the minimum over the band at each station; a station is counted as below the floor when that minimum falls under $0.98\,\Delta s_{\mathrm{NS}}$. The $2\%$ tolerance absorbs shock-capture smearing: the compliant runs dip at most $1.8\%$ below the floor, near the shoulder.

\subparagraph*{Results}

An autonomous numerical agent can produce a converged field and a technically competent report without making the assumptions that define the computation persistent or inspectable. We used the Mach-10 Apollo Command Module problem to test whether an explicit scientific substrate changes that outcome. The GRAFT-guided Numerical workflow augments the common agentic capabilities described above with explicit problem and action spaces, admissibility rules, persistent outcome memory, and stage-wise physicality gates. The comparison therefore isolates the effect of scientific structure rather than access to tools.

The ablation used two matched problem statements. Both fixed the Apollo forebody, the axisymmetric Euler equations at $M_\infty=10$, an impulsive start, and $t_{\mathrm{final}}=5$, and both declared the computational domain, mesh topology, and limiter to be implementation choices while requiring the stated geometry to be honored. The only task-level manipulation was whether the onset of the wake had to be represented. This produces two complementary comparisons: the commercial baselines and GRAFT can be compared within each task, while the two Fable runs show how the same autonomous baseline responds when one physical requirement is made explicit. Table~\ref{tab:apollo_setup} records the formulations that the five runs actually constructed.

The two commercial agentic baselines converge on the same method family across the reported Apollo configurations despite using different languages and grid sizes: a finite-volume MUSCL discretization, minmod limiting, and a uniform Cartesian immersed-boundary mesh. The implementations differ in their treatment of the geometry and computational boundaries. Without a wake requirement, Opus adds a sting to carry the body to the outflow, while Fable truncates the domain through the cone, downstream of the shoulder (Fig.~\ref{fig:apollo_ablation}). When the wake is requested, Fable instead allocates a larger domain and retains the truncated body. The two GRAFT runs select a different but internally consistent family: high-order DGSEM on mesh-conforming curved elements, with a body-fitted shell when the wake is excluded and an unstructured adaptively refined mesh when its onset must be represented.

The formal accuracy of the two families differs: the baseline MUSCL--minmod scheme is at most second order, dropping to first order where the limiter activates, whereas the GRAFT discretizations are fifth and sixth order ($p=4$ and $p=5$) in smooth regions. Higher order alone, however, does not establish that one solution is more trustworthy than another: distinct discretizations can solve the same problem correctly, and a dense mesh is not intrinsically inferior to a compact high-order one. We therefore recomputed the available scalar diagnostics under common definitions (see above) wherever an external reference exists. Table~\ref{tab:apollo_diagnostics} reports Rayleigh--Pitot and axial-coefficient errors against shared references, while retaining shock stand-off as a reported quantity rather than calling it an error: no exact stand-off reference is available for this geometry and condition, and the runs did not use a common shock-detection rule.

All five solvers complete and all five solutions converge. Every stagnation pressure lies within $0.9\%$ of the Rayleigh--Pitot value, the axial-coefficient errors span $-4.40\%$ to $+3.06\%$, and four of the five stand-off values fall between $0.3004$ and $0.3041$; the fifth, $0.385$, was reported by its own run as agreement with the correlation selected there. Fig.~\ref{fig:apollo_ablation} shows, for each run, the mesh and the density, Mach-number, and entropy fields.

Fig.~\ref{fig:apollo_ablation} exposes two distinct ways in which a converged answer can cease to represent the requested problem. In the Opus/no-wake run (first column), the agent extends the prescribed aft truncation as a sting to the outflow plane, describing it as a wind-tunnel-style measure to prevent inviscid base-flow unsteadiness from contaminating the forebody solution. The report does not present this choice as being forced by the numerical method, and both Fable runs retain the prescribed body on the same Cartesian immersed-boundary mesh, so the discretization did not require it. It is a geometry change, not merely a meshing choice. Its effect is visible in the aft flow, but it is weakly represented by the scalar checks: because the pressure-derived axial force is integrated against $dr$, the cylindrical sting contributes zero to $-C_A$. The loading can therefore remain plausible while the field belongs to a different vehicle.

The Fable/no-wake run (Fig.~\ref{fig:apollo_ablation}, second column) retains the specified geometry but sizes its domain, at the planning stage, from an expected answer: a stand-off near $0.4R_b$ taken from the Billig sphere correlation, with the outer boundary placed at $r=2R_b$ ``so the Dirichlet outer boundary never intersects the shock.'' Its checks then validate the plan against the same correlation: the measured stand-off, $\Delta/R_b=0.385$, is reported as agreeing with the Billig estimate of $0.355$, and the excess is attributed to the truncated shoulder. Both statements fail in the run's own solution: the bow shock reaches $r=1.998R_b$ against the boundary said never to intersect it, and the other four runs obtain stand-offs of $0.300$--$0.304$. The fields also reveal an unphysical wall layer (Fig.~\ref{fig:apollo_ablation}, Mach and entropy rows): a near-freestream streak hugs the body along the truncated cone, inside the shock layer where no such speed can exist, and the entropy row shows the same wall-adjacent layer far below the normal-shock floor, with $35.7\%$ of the sampled wall band below the entropy rise that fluid crossing the near-normal bow shock must acquire (Eq.~\eqref{eq:entropy_floor}). Inspection of the wall treatment also showed that mirror-point interpolation included solid-cell values still carrying the initial freestream state, mixing these values into the reflected wall states. The solver nevertheless decreases its residual by five orders of magnitude; convergence certifies that the discrete state has become stationary, not that it is physically reachable.

When the wake was explicitly requested, Fable constructed a solver that resolved the base flow and passed the entropy test. Its wall interpolation excluded solid-cell values, and its computational domain was larger. Across the two requests, the same agent therefore produced different implementations within the same finite-volume family, with different physical outcomes.

To test whether the Fable/no-wake failure documented above is representative rather than a selected trajectory, we reran the identical no-wake request four additional times with the same model in fresh sessions and audited all five trajectories with the same instruments; run 1 is the trajectory analyzed above. All five trajectories converge once more on essentially the same numerical method, finite-volume MUSCL with minmod limiting and an HLLE flux, three on uniform Cartesian immersed-boundary grids and two on body-fitted meshes extruded from the surface, and their Rayleigh--Pitot and axial-coefficient errors are all of the same order (Supplementary Table~\ref{tab:apollo_replicates}). Only two of the five trajectories, however, complete the task without a protocol failure: run 4 verifies shock containment directly, and run 5 catches the shock clamped at its lateral boundary in a coarse smoke test and widens the domain before production. Of the remaining three, run 1 never verifies its planned domain and ships the corrupted wall layer; run 3 alters the prescribed geometry with the same sting extension, and the same wind-tunnel justification, as Opus 4.8; and run 2 reports its outer boundary as verified to contain the shock while its stored field shows the shock crossing it, a violation confined to supersonic flow, so its deliverables remain clean.

The two GRAFT runs pair the higher-order discretization with scalar accuracy that matches or exceeds the baselines'. On the axial coefficient, both runs land closer to the AEDC reference than any baseline ($+1.55\%$ and $+1.73\%$ against $+2.62\%$, $+3.06\%$, and $-4.40\%$); their stagnation pressures agree with the Rayleigh--Pitot value to within $+0.46\%$ (wake) and $+0.08\%$ (no wake), using the common $t\in[4,5]$ window mean; and they do so with $50$k--$80$k solution nodes against the baselines' $128$k--$516$k cells while being coarser at the wall. They are also the only runs that enforce the stated geometry and every physicality condition as explicit gates across stages (entropy, positivity, shock clearance, and domain adequacy persist as gates across stages). The one feature the entropy row does show in the GRAFT columns, thin dark traces at the captured shock, is bounded Gibbs undershoot ($\Delta s$ to $-0.24$) confined to shock-straddling elements, the expected signature of a high-order discretization at a discontinuity; monotone schemes cannot undershoot by construction, which is why the entropy floor is evaluated on the wall band (Eq.~\eqref{eq:entropy_floor}) rather than on the field minimum. They are likewise the only ones that accompany their numbers with uncertainties: a four-run cross-check with a $\pm0.01$ tolerance on stand-off in one case, a six-iteration history with an explicit tolerance in the other, rather than a single declaration of agreement.

These choices are not improvised. The action space encodes methods and dependency rules distilled from the solvers' documentation, and the mesh and domain follow published hypersonic practice: domain containment from Billig's shock-shape correlations~\cite{anderson1989hypersonic}, margin and wall-spacing conventions from NASA production codes~\cite{mazaheri2010laura}, and forebody-truncation and wake-domain extents from Mars entry flight programs~\cite{dyakonov2012hypersonic,mitcheltree1995wake}. The correlations enter as containment bounds with measurement gates, not as answers: the wake run flags that the Billig fit is validated at $M=4$--$8$ and therefore measures the stand-off on the solution rather than trusting it. The GRAFT-guided workflow therefore does not eliminate numerical uncertainty; it makes the assumptions, admissibility conditions, and unresolved tolerances part of the scientific object that the agents carry from formulation through execution and validation.

\subsection{Action-space sampling sensitivity}
\label{sec:sensitivity_sampling}

To isolate the contribution of each layer of the action space, we remove the language-model proposer and compare four direct sampling procedures: (1) flattened uniform sampling, which pools all available actions without using the tree, dependencies, or GRAFT priors; (2) tree-structured uniform sampling, which traverses the active branches but does not enforce dependencies; (3) dependency-aware uniform sampling, which follows the GRAFT traversal and dependency rules with a uniform prior over admissible choices; and (4) GRAFT-prior sampling, which uses the same traversal and rules with problem-specific priors learned from nearby rewarded methods.

Each procedure produces the same fixed number of samples, and every sample is evaluated independently using the same topology audit and method validator; invalid samples are not replaced. Fig.~\ref{fig:sampling_sensitivity} reports the resulting component counts for the Apollo and spectral-PINN cases. Component counts include only methods that pass the admissibility validator.

The Apollo panel in Fig.~\ref{fig:sampling_sensitivity} shows why an explicit action space is necessary but not sufficient. Flattening the complete action space, or restricting sampling to the Trixi.jl tree while following only its branching structure, yields essentially no admissible methods, as expected for a problem whose solver choices are strongly coupled. Enforcing the dependencies makes admissible methods attainable, but only one sampled method combines the three requirements needed for this problem: the Euler equations, shock capturing, and a custom geometric source term. This shows that admissibility alone does not identify a suitable method. The GRAFT priors raise the count containing all three requirements to three in 100 and concentrate the samples on Euler and shock capturing, the two components supported by nearby solved problems. The custom source remains rare because it is absent from those neighbors. This problem-specific requirement is precisely where language-model judgment remains necessary: the proposer must infer that the axisymmetric formulation requires an axisymmetric source term even when the local repository assigns it little support.

The spectral-PINN panel asks a different question: whether the formulation derived through mathematical simplification by the Formalization team can be recovered by chance from the PIML action space. That derivation couples three structural choices: a modal sine representation, a time-only network input, and a pseudospectral Galerkin residual. These choices do not appear as a familiar combination in the local memory. The experiment already makes a favorable assumption by placing the nodes needed to describe the spectral PINN in the action space, although many of them were absent from the original PIML tree and were introduced only after the simplification had been derived. Without that expansion, the formulation lies outside the sampling support and cannot be recovered at all; even after the nodes are made available, none of the uniform procedures recovered their intersection in 1,000 samples, and the GRAFT prior recovered it only once (Fig.~\ref{fig:sampling_sensitivity}). Yet this single structural hit is not the derived formulation: it uses four networks where the validated method requires only one and, more importantly, replaces the hard initial and boundary constraints with a soft boundary constraint and no initial-condition constraint. The proposal may be executable, but it is unnecessarily inefficient and does not preserve the mathematical structure of the spectral PINN. The ablation therefore separates GRAFT's contribution, a structured and dependency-aware bias toward actions supported by previous solutions, from the mathematical judgment supplied by formalization and simplification, which remains necessary to construct a genuinely new formulation.

\subsection{Stability of the semantic space under action-space extension}
\label{sec:jaccard_invariance}

Incorporating the spectral-PINN formulation required extending the PIML
action space by 21 codenames in three subtrees: the network input type, the
solution representation, and the Galerkin residual form
(Supplementary Table~\ref{tab:spectral_pinn_extension}). A natural concern is that such an
extension rewrites the semantic space: adding nodes redistributes the
recursive layout, so the positions of pre-existing nodes change, and the
similarities accumulated over previously solved cases could drift with every
discovery. To verify that this does not happen, we recompute the method
fingerprints of all solved PIML repository cases on both taxonomies, the
pre-extension tree (reconstructed by removing the 21 added codenames, which
form three closed subtrees containing no pre-existing node) and the current
extended tree, and compare the two pairwise Jaccard similarity matrices.
Both computations use the deployed Stage 3 pipeline, i.e., the
node-set fingerprints of Methods.

Fig.~\ref{fig:jaccard_invariance} shows the result: the two matrices are
identical, with every entry agreeing to machine precision. This stability is
structural rather than accidental. Because the binning is injective at the
operating resolution, the Jaccard similarity depends only on which taxonomy
nodes two cases share, not on where the layout places them, so an additive
extension can alter no similarity between existing cases
(Proposition~\ref{prop:additive_invariance}). The action space can therefore grow with
each discovery while the semantic geometry of the solved-case memory, and
every outcome attached to it, remains exactly intact. We note that an
extension also enlarges what can be said about previously solved problems,
so an existing case can optionally be re-encoded in the richer vocabulary to
let the new decision axes draw on past experience. The clearest example is
the viscous Burgers case: it is the same benchmark for which the
spectral-PINN formulation was later derived, solved with the traditional
PINN approach (Fig.~\ref{fig:semantic_memory_learning}B and Table~\ref{tab:sota_comparison_graft}), so it benefits directly
from declaring its newly expressible input type and solution representation.
Even for this case, the one whose similarities shift the most, the largest
resulting change against any other case is $0.021$.

\clearpage

\section*{Supplementary Figures}

\begin{figure}[H]
\centering
\includegraphics[width=1\textwidth]{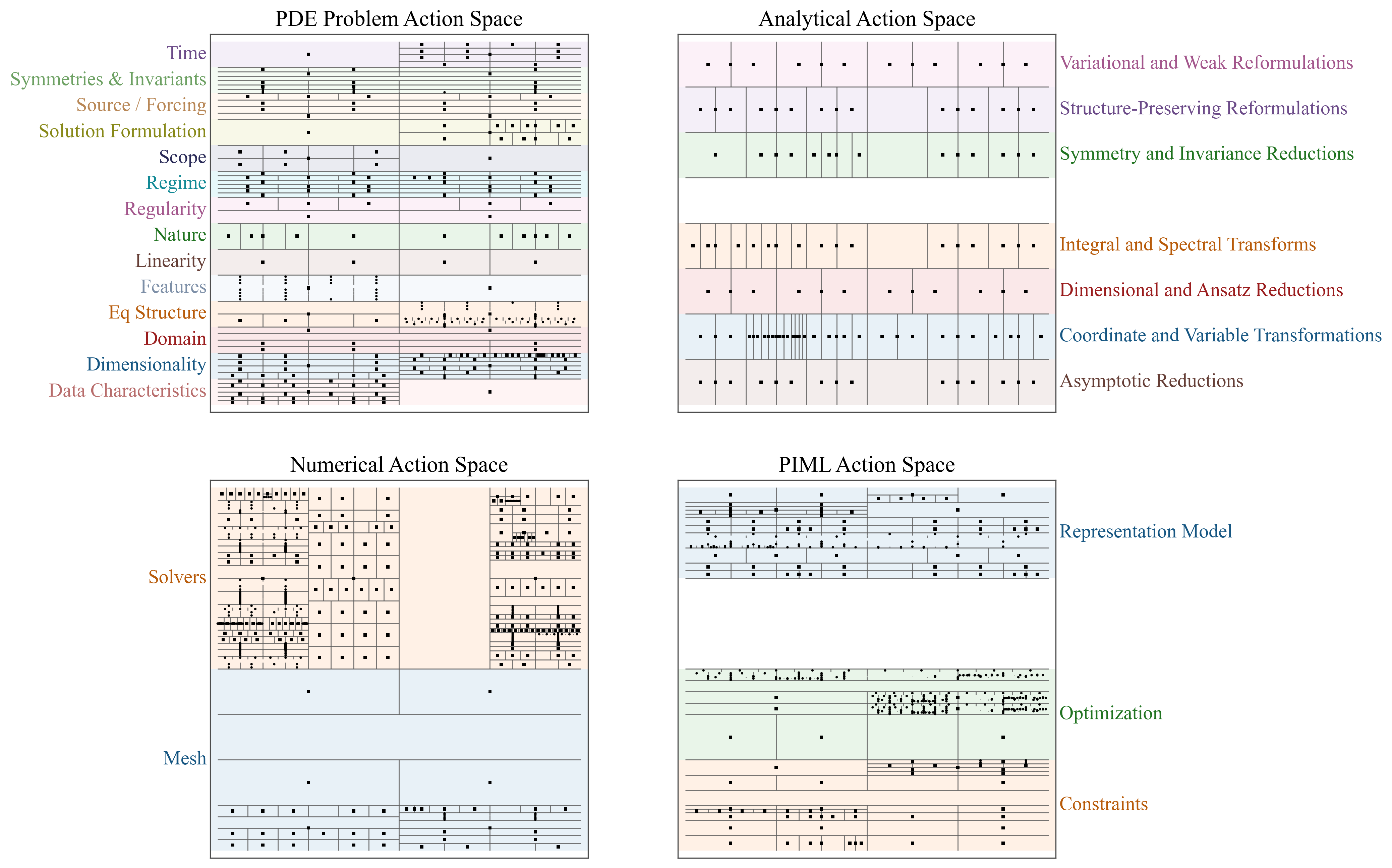}
\caption{
\textbf{The same semantic construction produces interpretable fingerprints across scientific-computing families.}
The four panels show the complete GRAFT spaces for PDE problem characteristics (upper left), analytical actions (upper right), numerical actions (lower left), and PIML actions (lower right). Each space follows the scientific vocabulary used to describe its own family, while GRAFT applies the same mathematical construction to all four. The partition-of-unity layout assigns each major semantic family a distinct region of the unit square and recursively partitions that region among increasingly specific characteristics or actions. Colored regions identify the major families, nested rectangles preserve their hierarchy, and black squares mark the available nodes. The position of an activated cell therefore identifies both the kind of scientific choice and its place in the hierarchy, making an individual fingerprint readable rather than merely unique. Each problem or method activates only the sparse cells along its selected paths, so fingerprints from different cases can be compared through the same geometric principle even though their domain-specific vocabularies differ.}
\label{fig:supp_full_space_pde_piml}
\end{figure}

\clearpage

\begin{figure}[H]
\centering
\includegraphics[width=1\textwidth]{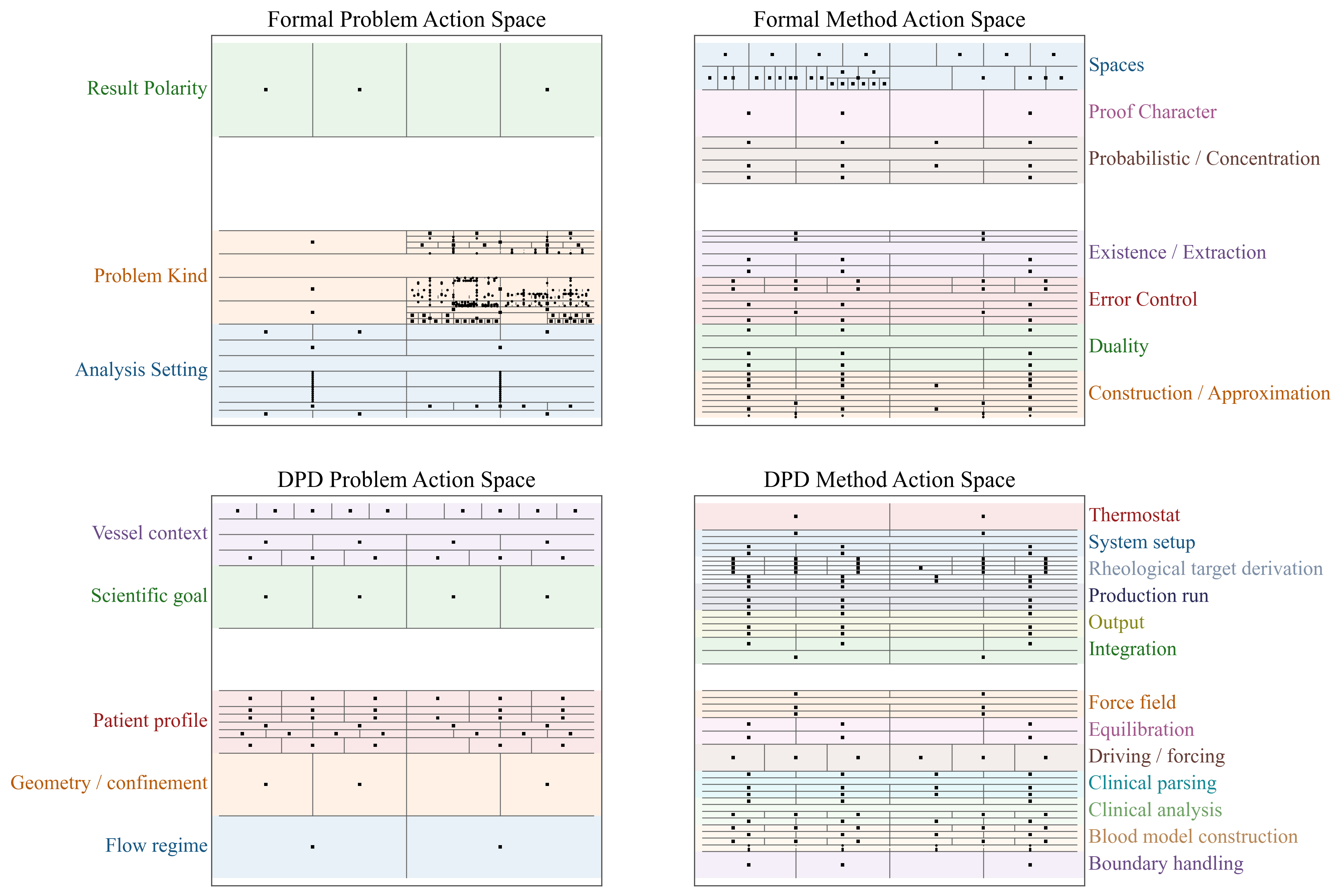}
\caption{
\textbf{The same semantic construction extends from formal reasoning to particle biophysics.}
The top row contains the complete Formal problem and method spaces. Formal problems are organized by result polarity, problem kind, and analysis setting, while their methods encode the mathematical spaces, proof character, construction route, and supporting arguments used to establish a result. The bottom row contains the corresponding DPD spaces, whose problem vocabulary describes vessel context, scientific goal, patient profile, geometry, and flow regime, and whose method vocabulary describes the choices needed to build, drive, run, and analyze a particle simulation. GRAFT preserves these domain-specific scientific vocabularies but projects both through the same partition-of-unity construction. Each major semantic family occupies a defined region of the unit square, and nested partitions locate increasingly specific characteristics or actions within that region; colored areas show the families, nested rectangles show their hierarchy, and black squares mark the available nodes. Consequently, the location of an activated cell remains scientifically interpretable in both domains, while the shared construction allows sparse fingerprints to be compared without forcing formal reasoning and DPD into the same vocabulary.}
\label{fig:supp_full_space_formal_dpd}
\end{figure}

\clearpage

\begin{figure}[H]
\centering
\includegraphics[width=1.0\textwidth]{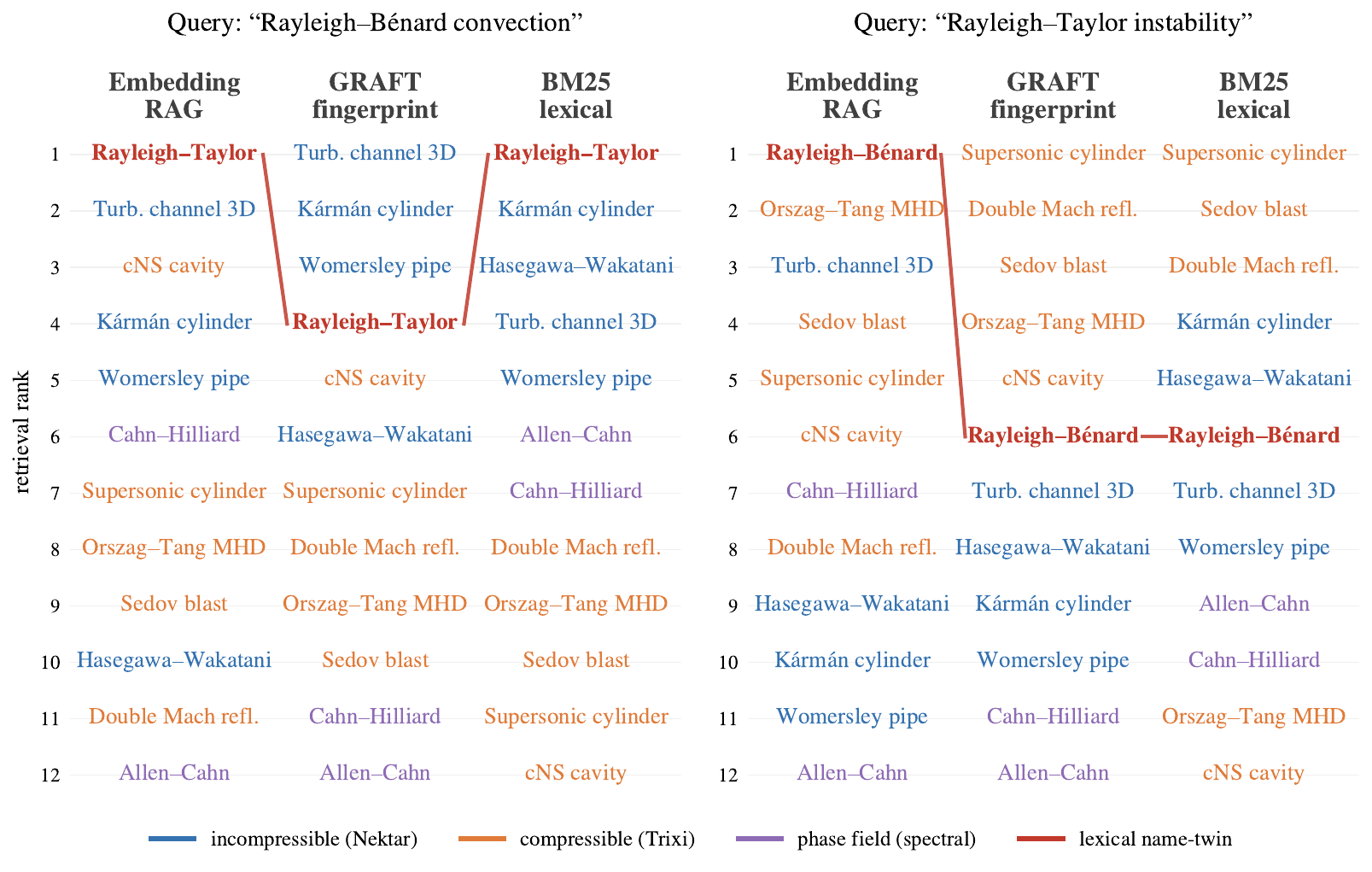}
\caption{%
\textbf{Scientific fingerprints recover solver-compatible neighbors when lexical similarity is misleading.}
Complete rankings of the 12 solved Numerical cases for the queries ``Rayleigh--B\'enard convection'' (left) and ``Rayleigh--Taylor instability'' (right), with the query problem's own source code held out. Each query is evaluated using embedding-based retrieval-augmented generation (RAG), GRAFT fingerprint Jaccard similarity, and BM25 lexical search. Colors identify the solver family of each retrieved case, and the red line traces the lexical name-twin across the three rankings. Embedding RAG places the name-twin first for both queries, and BM25 does so for Rayleigh--B\'enard; GRAFT instead ranks a solver-compatible case first in both directions. The comparison shows why textual resemblance alone is insufficient for transfer: the similarly named problems differ in the physical structure that determines which numerical methods are appropriate.}
\label{fig:retrieval_baselines}
\end{figure}

\clearpage

\begin{figure}[H]
\centering
\includegraphics[width=1\textwidth]{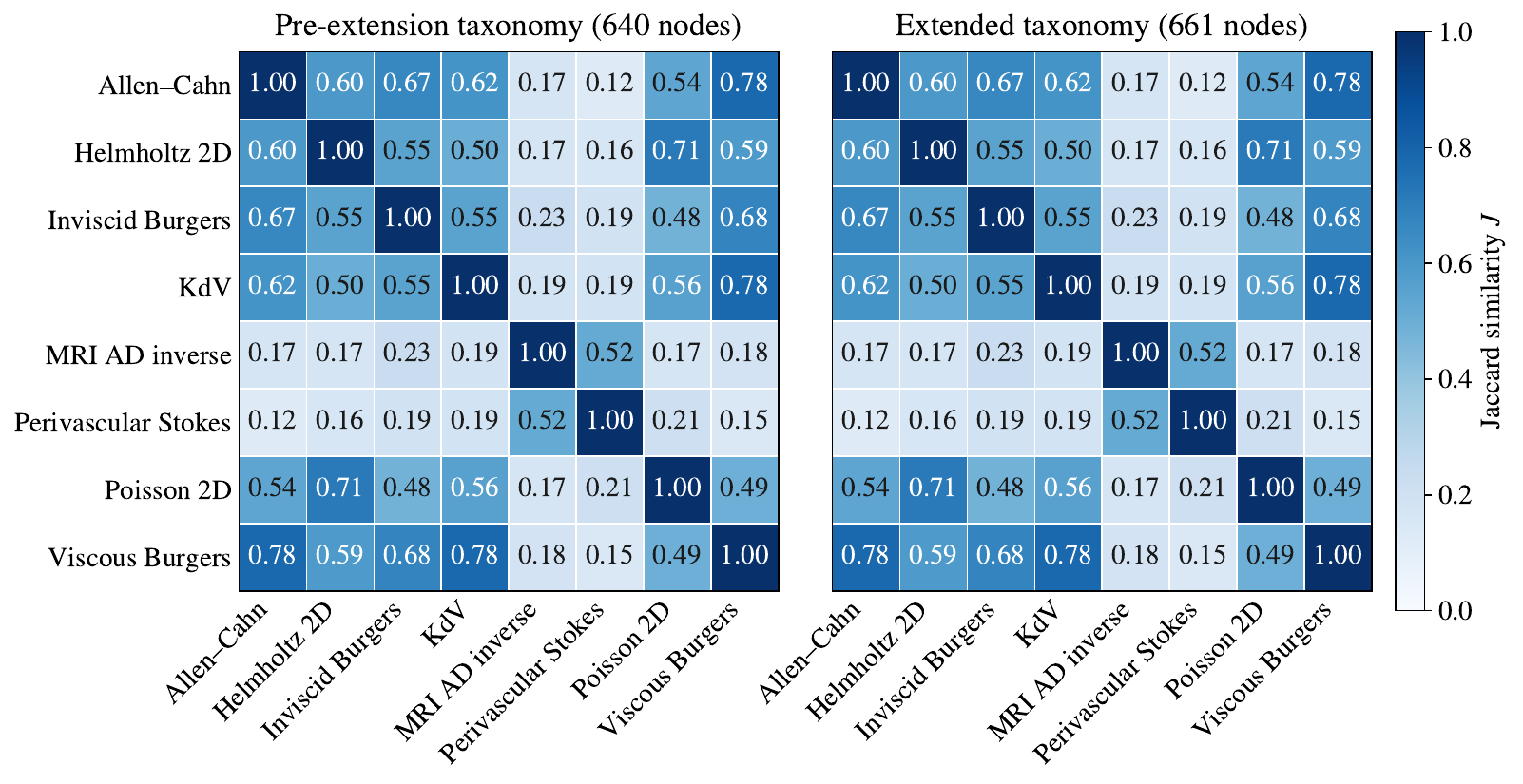}
\caption{%
\textbf{Pairwise similarities of unchanged encodings are invariant under additive action-space extension.}
Pairwise method-fingerprint Jaccard similarity between all solved PIML
repository cases, computed with the deployed Stage 3 fingerprint pipeline on
the PIML action space before (left, 640 nodes) and after (right, 661 nodes)
the 21-node extension introduced for the spectral-PINN formulation
(Supplementary Table~\ref{tab:spectral_pinn_extension}). The two matrices are identical to
machine precision: the similarity depends only on which nodes two cases
share, so extending the taxonomy changes node positions but no similarity
between existing cases.}
\label{fig:jaccard_invariance}
\end{figure}

\clearpage

\begin{figure}[H]
\centering
\includegraphics[width=0.9\textwidth]{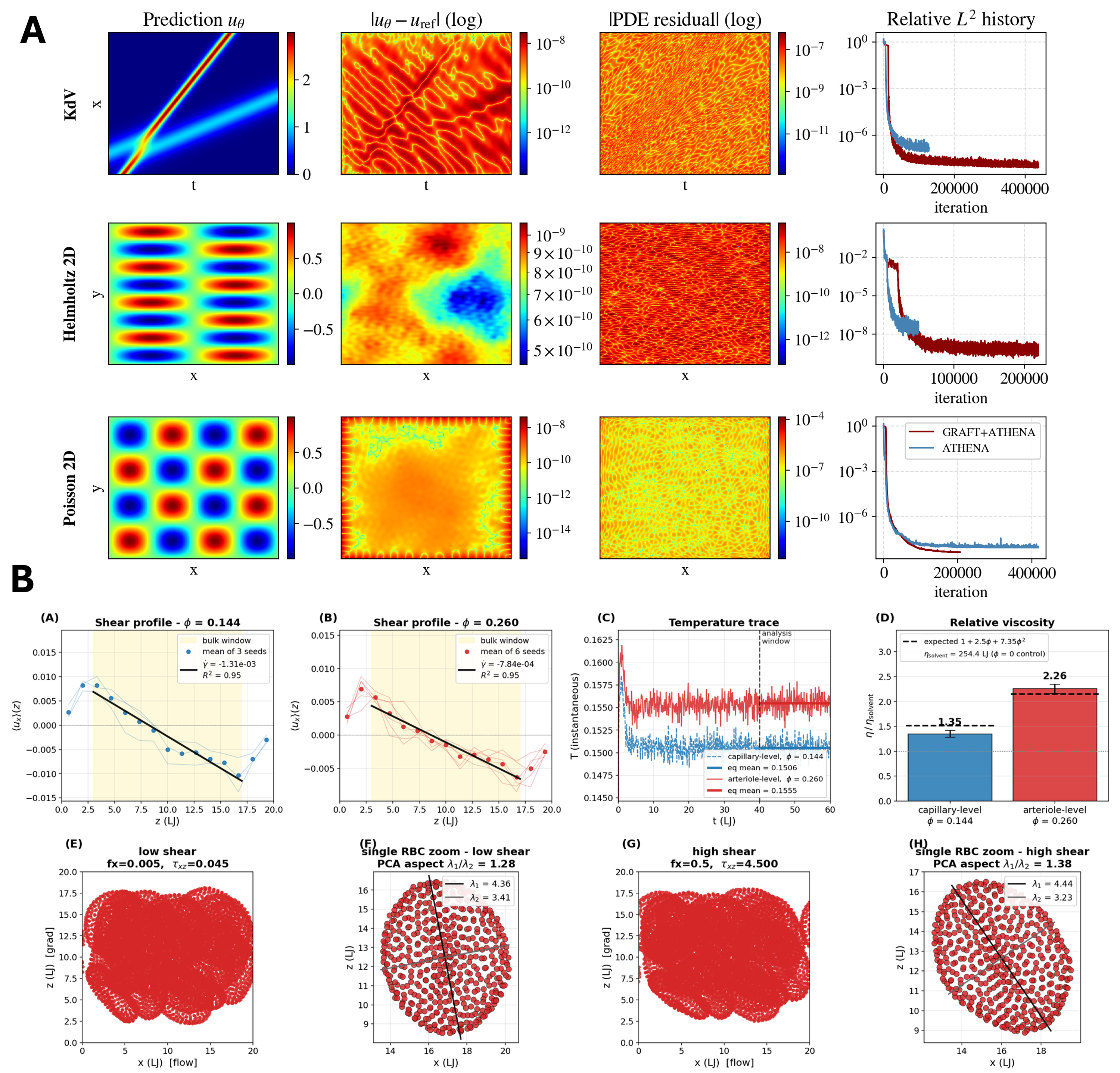}
\caption{
\textbf{GRAFT-informed methods reduce PIML errors and recover the expected hematocrit--viscosity relation.}
(\textbf{A}) Converged KdV, Helmholtz, and Poisson solutions. Each row shows the predicted field, absolute field error, PDE residual, and relative $L^2$ history. GRAFT+ATHENA (red) reaches lower error floors than ATHENA without GRAFT (blue), with near-machine-precision residuals, extending the Burgers result (Fig.~\ref{fig:semantic_memory_learning}B).
(\textbf{B}; subpanels A--D) RBC suspensions with identical membrane parameters and wall stress ($\tau_{xz}=0.45$), differing only in hematocrit: $\phi=0.144$ (blue) and $\phi=0.260$ (red). A and B show bulk velocity profiles averaged over the final $100{,}000$ steps and independent seeds to estimate shear rates. Faint lines denote individual seeds; shading marks the fitting window. C shows stable temperatures over the analysis interval. D gives viscosity relative to a cell-free control, standard errors across seeds, and expected concentration dependence (dashed lines). At matched wall stress, the denser suspension is more viscous. The assembled particle workflow therefore recovers the expected hematocrit dependence without tuning parameters to obtain this ordering. (Subpanels E--H) A separate low- and high-forcing sweep shows periodic-unwrapped RBCs and their principal axes, illustrating the membrane model's deformation mode.
}
\label{fig:priors_and_results}
\end{figure}

\clearpage

\begin{figure}[H]
\centering
\includegraphics[width=1\textwidth]{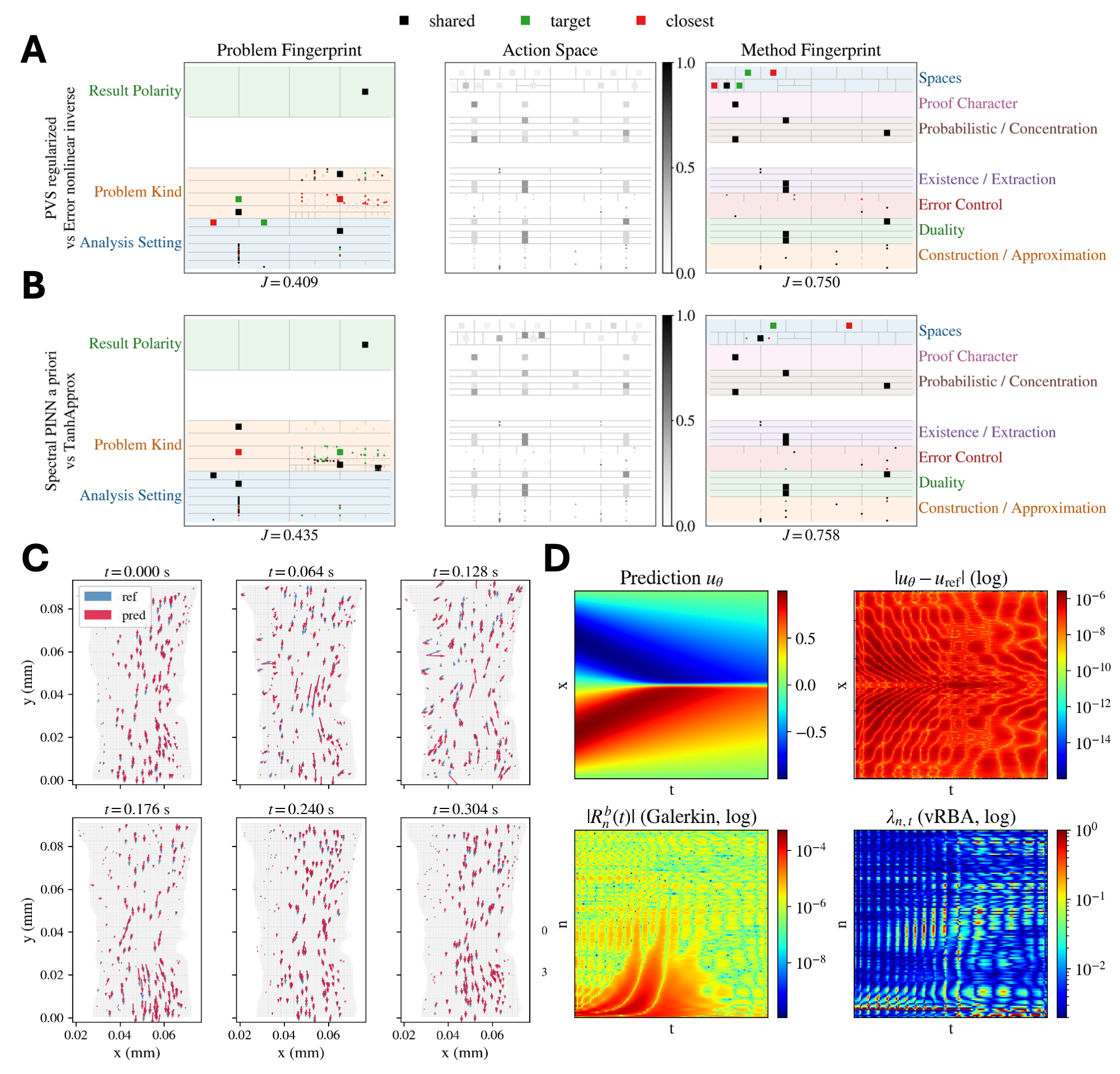}
\caption{
\textbf{Formal proof structures transfer across distinct problems, and the resulting methods pass targeted validation.}
(\textbf{A}) The regularized perivascular a posteriori theorem and its nearest stored nonlinear inverse-error theorem: problem fingerprints ($J=0.409$), the full-neighborhood action prior, and proof-method fingerprints ($J=0.750$). Greater method than problem similarity reveals transfer of an established inverse-error proof pattern to the perivascular guarantee.
(\textbf{B}) The spectral-PINN a priori theorem and nearest quantitative $\tanh$-approximation theorem similarly have problem similarity $J=0.435$ and method similarity $J=0.758$. Retrieved approximation structure supports the new a priori error bound. In A and B, black entries are shared, green are target-specific, and red are neighbor-specific.
(\textbf{C}) Predicted perivascular velocities (red) and held-out particle measurements (blue) at six times validate unseen motion within the measured plane.
(\textbf{D}) Dense-grid prediction, pointwise error, per-mode Galerkin residual, and converged vector residual-based-attention (vRBA) weights for the best-resolved spectral PINN. With the mode-count sweep (Fig.~\ref{fig:verified_discovery}F), these diagnostics show decreasing error with modal resolution without an optimization-imposed plateau.
}
\label{fig:verified_discovery_supp}
\end{figure}

\clearpage

\clearpage
\section*{Supplementary Tables}

\begin{table}[!htbp]
\centering
\caption{
\textbf{GRAFT factorization avoids the largest dense conditional expansions across four scientific method spaces.}
Part A reports the complete local decision vocabulary of the PIML, Numerical, DPD, and Formal spaces together with the learnable categorical entries and symbolic rules stored by GRAFT. Part B follows the most strongly conditioned chain in each family from its incoming dependencies $\mathrm{pa}_H(j)$, through their joint configurations, to the entries required by a dense conditional table. PIML gives the largest expansion: dense conditioning of a single chain requires $7.12\times10^{13}$ entries, whereas the complete factored PIML policy contains 512 learnable categorical entries and 17 explicit rules. Numerical shows the same combinatorial effect at smaller scale; DPD and Formal have shallower dependency frontiers and correspondingly modest dense tables. The comparison shows that the cost of dense conditioning rises with the number and cardinality of coupled choices, while GRAFT continues to store local decisions and explicit rules.}
\label{tab:policy_footprint_anatomy}
{\small\linespread{1}\selectfont
\setlength{\tabcolsep}{3.5pt}
\textbf{A. Complete method-space composition}\par\smallskip
\begin{tabularx}{\textwidth}{@{}>{\raggedright\arraybackslash}p{2.1cm}*{4}{>{\centering\arraybackslash}X}@{}}
\toprule
Method space & Chains $k$ & Learnable categorical entries $\sum_j|\mathcal{A}_j|$ & Mean choices per chain & Symbolic rules $|\mathcal{R}|$ \\
\midrule
PIML      & 117 & 512 & 4.38 & 17 \\
Numerical & 156 & 499 & 3.20 & 78 \\
DPD       & 37  & 133 & 3.59 & 6  \\
Formal    & 30  & 106 & 3.53 & 0  \\
\bottomrule
\end{tabularx}

\vspace{0.8em}
\textbf{B. Largest incoming-dependency expansion}\par\smallskip
\setlength{\tabcolsep}{3pt}
\begin{tabularx}{\textwidth}{@{}>{\raggedright\arraybackslash}p{1.8cm}>{\raggedright\arraybackslash}X>{\centering\arraybackslash}p{2.2cm}>{\raggedright\arraybackslash}X@{}}
\toprule
Method space & Target chain & \shortstack{Incoming\\dependencies\\$|\mathrm{pa}_H(j)|$} & Choices carried by incoming dependencies \\
\midrule
PIML & Number of networks (\texttt{NON}) & 16 & One 3-way chain and fifteen 7-way chains \\
Numerical & Nektar++ driver type (\texttt{ND4}) & 16 & Fourteen 2-way chains, one 3-way chain, and one 6-way chain \\
DPD & Driving/forcing (\texttt{DF}) & 1 & One 4-way chain \\
Formal & Operator domain space (\texttt{ODS}) & 1 & One 12-way chain \\
\bottomrule
\end{tabularx}

\vspace{0.6em}
\begin{tabular}{@{}>{\raggedright\arraybackslash}p{2.0cm}>{\centering\arraybackslash}p{4.6cm}>{\centering\arraybackslash}p{2.7cm}>{\centering\arraybackslash}p{3.1cm}@{}}
\toprule
Method space & \shortstack{Joint dependency configurations\\$\prod_{i\in\mathrm{pa}_H(j)}|\mathcal{A}_i|$} & Target choices $|\mathcal{A}_j|$ & Dense table entries \\
\midrule
PIML & $3\times7^{15}=1.42\times10^{13}$ & 5 & $7.12\times10^{13}$ \\
Numerical & $2^{14}\times3\times6=294{,}912$ & 7 & $2{,}064{,}384$ \\
DPD & $4$ & 6 & $24$ \\
Formal & $12$ & 6 & $72$ \\
\bottomrule
\end{tabular}
}
\end{table}

\begin{table}[!htbp]
\centering
\caption{\textbf{PIML seeds used to initialize the solved-case memory.} In-house denotes an implementation developed within the project rather than inherited from an external solver example.}
\label{tab:memory_seeds_piml}
{\footnotesize\linespread{1}\selectfont
\setlength{\tabcolsep}{3pt}
\begin{tabularx}{\textwidth}{@{}>{\raggedright\arraybackslash}p{2.15cm}>{\raggedright\arraybackslash}X>{\raggedright\arraybackslash}X>{\raggedright\arraybackslash}p{2.0cm}@{}}
\toprule
Seed & Problem & Successful method & Provenance \\
\midrule
Allen--Cahn & Periodic one-dimensional reaction--diffusion with weak diffusion and $x^2\cos(\pi x)$ initial data & Periodic-Fourier PINN with vRBA, Lion warm-up, backtracking SSBroyden, and diffusion continuation & vRBA study~\cite{toscano2026variational} \\
Inviscid Burgers & Periodic entropy solution including the shock formed from $u_0=-\sin(\pi x)$ & Periodic-wavelet PINN with a hard initial condition, entropy penalty, vRBA, and Adam-to-SSBroyden viscosity continuation & ATHENA study~\cite{toscano2025athena} \\
MRI advection--diffusion inverse problem & Recover divergence-free three-dimensional velocity and denoised tracer concentration from sparse, noisy MRI-like measurements & Six-subdomain cPINN with curl-constrained velocity, probabilistic concentration output, vRBA, SOAP, and staged physics--data training & In-house code \\
\bottomrule
\end{tabularx}
}
\end{table}

\begin{table}[!htbp]
\centering
\caption{\textbf{Numerical-method seeds used to initialize the solved-case memory (part I).} Fourier-spectral implementations are in-house; Nektar++- and Trixi.jl-based cases inherit their solver families from the corresponding example libraries.}
\label{tab:memory_seeds_numerical_a}
{\footnotesize\linespread{1}\selectfont
\setlength{\tabcolsep}{3pt}
\begin{tabularx}{\textwidth}{@{}>{\raggedright\arraybackslash}p{2.15cm}>{\raggedright\arraybackslash}X>{\raggedright\arraybackslash}X>{\raggedright\arraybackslash}p{2.0cm}@{}}
\toprule
Seed & Problem & Successful method & Provenance \\
\midrule
Allen--Cahn 1D & Periodic bistable reaction--diffusion with weak diffusion and interface formation from smooth initial data & Fourier pseudo-spectral SplitODE with fixed-step ETDRK4, separating stiff diffusion from the explicit cubic reaction & In-house Fourier-spectral code \\
Cahn--Hilliard 2D & Periodic spinodal decomposition from low-amplitude noise, monitored through mass conservation and free-energy decay & Dealiased Fourier pseudo-spectral SplitODE with ETDRK4, separating the stiff linear operator from the cubic term & ATHENA study~\cite{toscano2025athena} \\
K\'arm\'an cylinder wake & Incompressible flow past a cylinder at $Re_D=60$, resolving the periodic vortex street & Nektar++ velocity-correction solver on a curved continuous spectral/$hp$ mesh with IMEX stepping & Nektar++~\cite{cantwell2015nektar} \\
Double Mach reflection & Mach-10 inclined-shock reflection over a ramp, resolving Mach stems, slip lines, and the near-wall jet & Trixi DGSEM with hybrid DG--finite-volume shock capture, L\"ohner AMR, positivity preservation, and SSPRK integration & Trixi.jl~\cite{ranocha2021adaptive} \\
Hasegawa--Wakatani turbulence & Three-field periodic drift-wave turbulence at $\alpha=2$ and $\kappa=1$ from a Gaussian perturbation & Nektar++ discontinuous spectral/$hp$ solver with RK4 and internal Helmholtz--Poisson recovery of the potential & Nektar++-based plugin~\cite{cantwell2015nektar} \\
Orszag--Tang MHD vortex & Periodic ideal-MHD vortex developing interacting shocks and current sheets & Trixi DGSEM with hybrid shock capture, three-level AMR, low-storage RK4, and GLM divergence cleaning & Trixi.jl~\cite{ranocha2021adaptive} \\
Rayleigh--B\'enard convection & Two-dimensional incompressible Boussinesq convection at $Ra=10^4$ and $Pr=0.71$, heated from below & Nektar++ velocity-correction solver using continuous spectral/$hp$ quadrilaterals, IMEX2, and Boussinesq forcing & Nektar++~\cite{cantwell2015nektar} \\
\bottomrule
\end{tabularx}
}
\end{table}

\begin{table}[!htbp]
\centering
\caption{\textbf{Numerical-method seeds used to initialize the solved-case memory (part II).}}
\label{tab:memory_seeds_numerical_b}
{\footnotesize\linespread{1}\selectfont
\setlength{\tabcolsep}{3pt}
\begin{tabularx}{\textwidth}{@{}>{\raggedright\arraybackslash}p{2.15cm}>{\raggedright\arraybackslash}X>{\raggedright\arraybackslash}X>{\raggedright\arraybackslash}p{2.0cm}@{}}
\toprule
Seed & Problem & Successful method & Provenance \\
\midrule
Rayleigh--Taylor instability & Gravity-driven compressible Euler instability between heavy and light fluids, producing mushroom-shaped mixing structures & Trixi DGSEM with hybrid shock capture, L\"ohner AMR, gravitational forcing, and optimized explicit Runge--Kutta integration & ATHENA study~\cite{toscano2025athena} \\
Sedov blast & Two-dimensional inviscid cylindrical blast generated by a localized radial pressure pulse & Trixi DGSEM on a radially graded mesh with hybrid stabilization, AMR, and low-storage Runge--Kutta integration & Trixi.jl~\cite{ranocha2021adaptive} \\
Supersonic cylinder & Mach-3 Euler flow around a cylinder, including the detached bow shock and downstream wake & Trixi DGSEM on a body-fitted mesh with invariant-domain-preserving subcell limiting and SSPRK33 integration & Trixi.jl~\cite{ranocha2021adaptive} \\
Channel startup 3D-H1D & Low-Reynolds-number body-forced channel startup approaching laminar Poiseuille flow & Nektar++ velocity-correction solver with a spectral/$hp$ cross-section, dealiased homogeneous Fourier direction, and IMEX2 & Nektar++~\cite{cantwell2015nektar} \\
Womersley pipe flow & Three-dimensional pulsatile pipe flow driven by a ten-mode inflow at approximately $Re=10$ and $\alpha=4$ & Nektar++ velocity-correction solver on a curved-prism continuous spectral/$hp$ mesh with IMEX1 & Nektar++~\cite{cantwell2015nektar} \\
Compressible cavity & Low-Mach compressible Navier--Stokes lid-driven cavity at approximately $Re=1000$ & Trixi DGSEM with hybrid stabilization, L\"ohner AMR, and CFL-controlled low-storage Runge--Kutta integration & Trixi.jl~\cite{ranocha2021adaptive} \\
\bottomrule
\end{tabularx}
}
\end{table}

\begin{table}[!htbp]
\centering
\caption{\textbf{DPD seeds used to initialize the solved-case memory.} All four workflows were implemented in-house using LAMMPS~\cite{LAMMPS}.}
\label{tab:memory_seeds_dpd}
{\footnotesize\linespread{1}\selectfont
\setlength{\tabcolsep}{3pt}
\begin{tabularx}{\textwidth}{@{}>{\raggedright\arraybackslash}p{2.15cm}>{\raggedright\arraybackslash}X>{\raggedright\arraybackslash}X>{\raggedright\arraybackslash}p{2.0cm}@{}}
\toprule
Seed & Problem & Successful method & Provenance \\
\midrule
Planar Poiseuille flow & Steady Newtonian channel flow under uniform body force, with viscosity inferred from the parabolic profile & DPD with frozen-particle walls, mVV integration, native thermostatting, and parabolic-profile postprocessing & In-house LAMMPS \\
RBC shear smoke test & Eight healthy RBCs in a wall-bounded slot under opposing wall forces, testing short-run thermal stability & DPD with spring-network RBC membranes, wall forcing, mVV integration, native thermostatting, and thermo-trace postprocessing & In-house LAMMPS \\
Reverse Poiseuille flow & Fully periodic Newtonian flow with opposite forcing in the two halves of the domain & Single-species DPD without walls or RBCs, using mVV integration and separate parabolic fits for both flow arcs & In-house LAMMPS \\
RBC suspension viscosity & Twenty deformable healthy RBCs under opposing-wall shear, with effective viscosity estimated from the bulk slope & DPD with spring-network membranes, mVV integration, native thermostatting, and a linear bulk-velocity fit & In-house LAMMPS \\
\bottomrule
\end{tabularx}
}
\end{table}

\begin{table}[!htbp]
\centering
\caption{\textbf{Formal seeds used to initialize the solved-case memory (part I).} Each record encodes the cited theorem and its machine-checked Lean realization.}
\label{tab:memory_seeds_formal_a}
{\footnotesize\linespread{1}\selectfont
\setlength{\tabcolsep}{3pt}
\begin{tabularx}{\textwidth}{@{}>{\raggedright\arraybackslash}p{2.15cm}>{\raggedright\arraybackslash}X>{\raggedright\arraybackslash}X>{\raggedright\arraybackslash}p{2.0cm}@{}}
\toprule
Seed & Theorem problem & Encoded proof strategy & Source \\
\midrule
CNO universality & Quantitative $L^2/L^\infty$ universality for translation-equivariant CNOs on band-limited Sobolev inputs & Fourier truncation and sinc sampling reduce the operator to a finite stencil realized by a CNN & \cite{raonic2023convolutional} \\
Chen--Chen operator theorem & Uniform approximation of continuous nonlinear operators on compact function families by finite-sensor branch--trunk expansions & Bochner--Riesz reduction followed by partition-of-unity sensor interpolation and Arzel\`a--Ascoli compactness & \cite{chen1995universal} \\
DeepONet universality & Extend Chen--Chen from shallow networks to universal branch and trunk architecture families & Approximate every factor in the Chen--Chen decomposition and control the product-substitution error & \cite{lu2021learning} \\
DeepONet sensor approximation & Bound finite-sensor approximation of a Lipschitz operator value over a compact input family & Approximate the decoder on the compact sensor image and combine architecture and Lipschitz reconstruction errors & \cite{lu2021learning} \\
Ridge-PINN consistency & Almost-sure risk consistency and asymptotic unbiasedness of regularized PINNs as sampling and width increase & Ridge confinement, concentration with Borel--Cantelli, and $C^K$ density connect sampling and approximation limits & \cite{doumeche2023convergence} \\
Unregularized PINN overfitting & Construct PINNs whose sampled loss vanishes while population residual risk diverges & Nested tanh saturation functions with shrinking transition windows controlled by Cauchy--Schwarz & \cite{doumeche2023convergence} \\
Regularized PINN well-posedness & Existence, uniqueness, and strong $H^m$ convergence for the regularized variational problem & Lax--Milgram and the direct method, with Rellich compactness and weak lower semicontinuity & \cite{doumeche2023convergence} \\
Nonlinear forward error & Forward PINN error from training residual, quadrature error, and nonlinear stability & A $p$-root quadrature bridge followed by Lipschitz stability & \cite{mishra2023estimates} \\
Nonlinear inverse error & Inverse PINN error from PDE and data losses, quadrature, conditional stability, and noise & Raise quadrature estimates to the stability exponents and substitute them into the inverse inequality & \cite{mishra2022estimates} \\
FNO universality & Universal approximation of continuous operators between periodic Sobolev spaces by Fourier neural operators & Truncate coefficient space and Fourier modes, apply finite-dimensional universality, and lift back to sequence space & \cite{kovachki2021universal} \\
\bottomrule
\end{tabularx}
}
\end{table}

\begin{table}[!htbp]
\centering
\caption{\textbf{Formal seeds used to initialize the solved-case memory (part II).}}
\label{tab:memory_seeds_formal_b}
{\footnotesize\linespread{1}\selectfont
\setlength{\tabcolsep}{3pt}
\begin{tabularx}{\textwidth}{@{}>{\raggedright\arraybackslash}p{2.15cm}>{\raggedright\arraybackslash}X>{\raggedright\arraybackslash}X>{\raggedright\arraybackslash}p{2.0cm}@{}}
\toprule
Seed & Theorem problem & Encoded proof strategy & Source \\
\midrule
KKAN universality & Basis-agnostic universality when the inner and outer univariate approximation classes are dense & Apply exact Kolmogorov superposition, approximate each block, and sum the componentwise errors & \cite{toscano2025kkans} \\
Kolmogorov superposition & Exact five-term representation of two-variable continuous functions using universal monotone inner functions & Nested coverings and separated labels, followed by uniform limits and an outer contraction & \cite{Kolmogorov1957} \\
K\r{u}rkov\'a density theorem & Uniform density of two-hidden-layer sigmoidal networks for continuous multivariate functions & Combine the general-dimensional Kolmogorov representation with sigmoidal staircase approximation & \cite{Kuurkova1992kolmogorov} \\
A posteriori PINN error & Computable least-squares error bound from achieved loss and quadrature under coercivity & Hilbert-space energy identity and coercivity convert discrete loss into solution error & \cite{zeinhofer2024unified} \\
A priori PINN error & C\'ea-type bound combining best ansatz approximation with optimization and quadrature gaps & Coercivity, boundedness, energy identities, and a discrete-to-continuous objective-gap estimate & \cite{zeinhofer2024unified} \\
Fixed-width realization obstruction & Convexity of a fixed-width network realization set forces a polynomial activation & Parameter-capacity bounds, translation--dilation invariance, and finite-dimensional exponential-polynomial structure & \cite{petersen2021topological} \\
Greedy convergence & $O(1/n)$ convergence of the relaxed greedy algorithm over a variation-norm dictionary ball & Derive one-step descent and solve the resulting quadratic recurrence & \cite{Siegel2023112084} \\
Greedy risk decomposition & Population excess risk as approximation, uniform generalization, and empirical optimization terms & Telescope the risks and bound them using empirical optimality and the uniform population--empirical gap & \cite{Siegel2023112084} \\
Quantitative tanh approximation & Explicit two-hidden-layer approximation in $W^{k,\infty}$, including width and weight bounds & Tanh Taylor emulation, a tanh partition of unity, and Bramble--Hilbert estimates & \cite{de2021approximation} \\
Sharp shallow-network UAT & Single-hidden-layer universality holds exactly for nonpolynomial activations & Dual annihilators and ridge superposition prove sufficiency; polynomial finite-dimensionality proves necessity & \cite{Leshno1993} \\
\bottomrule
\end{tabularx}
}
\end{table}

\begin{table}[!htbp]
\centering
\caption{\textbf{Formulations and computational setups produced for the Apollo ablation.} Runs are arranged as rows so that the wake requirement, numerical method, mesh, and treatment of the geometry and domain can be read together. IB denotes immersed boundary; EC, entropy-conservative; H--G, Hennemann--Gassner; and AMR, adaptive mesh refinement.}
\label{tab:apollo_setup}
\setlength{\tabcolsep}{2.5pt}
\begin{tabularx}{\textwidth}{@{}>{\raggedright\arraybackslash}p{1.8cm}>{\centering\arraybackslash}p{0.75cm}*{3}{>{\raggedright\arraybackslash}X}@{}}
\toprule
System & Wake & Spatial method and shock treatment & Mesh and size & Geometry and domain treatment \\
\midrule
Opus 4.8 & No & FV MUSCL; minmod; HLLE & Cartesian IB; 516,096 cells & Sting added to the outflow; fixed body altered \\
Fable 5 & No & FV MUSCL; minmod; HLLE & Cartesian IB; 128,000 cells & Domain ends mid-cone; fixed-freestream boundary adjacent to shock-disturbed flow \\
GRAFT & No & DGSEM $p=5$; EC volume; IDP and Zhang--Shu & Body-fitted shell; 1,386 elements; 49,896 nodes & Fixed forebody retained; bow shock contained \\
Fable 5 & Yes & FV MUSCL; minmod; HLLC with HLLE guard & Cartesian IB; 440,000 cells & Fixed geometry retained; wake resolved \\
GRAFT & Yes & DGSEM $p=4$; EC volume; H--G/FV subcells and Zhang--Shu & Unstructured AMR; 3,198 elements; 79,950 nodes & Fixed geometry retained; wake onset resolved and exit gated \\
\bottomrule
\end{tabularx}
\end{table}

\begin{table}[!htbp]
\centering
\caption{\textbf{Human intervention across all evaluated problems (part I): PIML and Numerical.} Stage ``conceptualization'' covers formulation and method selection before execution; ``final evaluation'' covers review of the completed result.}
\label{tab:human_intervention}
{\footnotesize\linespread{1}\selectfont
\setlength{\tabcolsep}{3pt}
\begin{tabularx}{\textwidth}{@{}>{\raggedright\arraybackslash}p{2.1cm}>{\raggedright\arraybackslash}p{1.25cm}>{\raggedright\arraybackslash}X>{\raggedright\arraybackslash}p{1.5cm}>{\raggedright\arraybackslash}X@{}}
\toprule
Problem & Family & Description & Intervention & Stage and nature \\
\midrule
Viscous Burgers & PIML & Canonical forward benchmark at low viscosity, solved with the
  inherited Reynolds-number continuation & None & \\
KdV & PIML & Canonical forward dispersive benchmark against the analytical solution
  & None & \\
Helmholtz 2D & PIML & Canonical forward benchmark against the analytical solution
  & None & \\
Poisson 2D & PIML & Canonical forward benchmark against the analytical solution
  & None & \\
Perivascular inverse problem & PIML & Inverse reconstruction of velocity and pressure in the
  perivascular space from sparse in vivo two-photon data & Yes & Conceptualization:
  interactive formalization; the agents identified the deficit and proposed the regularizer,
  the user supplied missing context and selected among the proposed options. \\
Spectral PINN (periodic Burgers) & PIML & Periodic viscous Burgers under the agent-derived
  pseudospectral Fourier--Galerkin reformulation & Yes & Conceptualization: the agents
  proposed the Fourier reduction; the user redirected the ranking and posed the modified
  setting. \\
\midrule
Apollo Mach 10 (wake) & Numerical & Axisymmetric Mach-10 inviscid flow over the Apollo
  forebody with the downstream wake required & None & \\
Apollo Mach 10 (no wake) & Numerical & Same forebody problem without the wake requirement
  & None & \\
\bottomrule
\end{tabularx}}
\end{table}

\begin{table}[!htbp]
\centering
\caption{\textbf{Human intervention across all evaluated problems (part II): DPD and Formal.} Stage definitions as in part I.}
\label{tab:human_intervention_b}
{\footnotesize\linespread{1}\selectfont
\setlength{\tabcolsep}{3pt}
\begin{tabularx}{\textwidth}{@{}>{\raggedright\arraybackslash}p{2.1cm}>{\raggedright\arraybackslash}p{1.25cm}>{\raggedright\arraybackslash}X>{\raggedright\arraybackslash}p{1.5cm}>{\raggedright\arraybackslash}X@{}}
\toprule
Problem & Family & Description & Intervention & Stage and nature \\
\midrule
Planar Poiseuille flow & DPD & Steady Newtonian channel flow under uniform body force;
  viscosity from the parabolic profile & None & \\
RBC shear smoke test & DPD & Eight healthy RBCs in a wall-bounded slot under opposing wall
  forces; short-run thermal stability & None & \\
Reverse Poiseuille flow & DPD & Fully periodic Newtonian flow with opposite forcing in the
  two domain halves & None & \\
RBC suspension viscosity & DPD & Twenty deformable healthy RBCs under opposing-wall shear;
  effective viscosity from the bulk slope & None & \\
Low-hematocrit suspension & DPD & Deformable RBC suspension at capillary-level
  hematocrit, $\phi=0.144$, under opposing-wall shear at $\tau_{xz}=0.45$; relative
  viscosity from the bulk slope & None & \\
High-hematocrit suspension & DPD & The same suspension at arteriole-level hematocrit,
  $\phi=0.260$, at the same wall stress and the same membrane parameters; hematocrit is
  the only difference & None & \\
\midrule
cKAN universality & Formal & Machine-checked universal-approximation theorem for deep
  Chebyshev Kolmogorov--Arnold networks & None & \\
LNO universality & Formal & Machine-checked universal-approximation theorem for the Laplace
  neural operator & Yes & Final evaluation: human review found the verified statement weaker
  than requested and directed its strengthening. \\
Perivascular a posteriori certificate & Formal & Machine-checked a posteriori error
  certificate for the regularized perivascular reconstruction & None & \\
Spectral a priori theorem & Formal & Machine-checked a priori exponential error bound for
  the spectral-PINN architecture & None & \\
\bottomrule
\end{tabularx}}
\end{table}

\begin{table}[!htbp]
\centering
\caption{\textbf{Sensitivity of GRAFT retrieval to the level weight $\omega$
across the four solved-case repositories.}
Family means from leave-one-out retrieval under each weight schedule.
Numerical uses the 13-query set of Table~\ref{tab:topk_benchmark}; the
other families use all solved cases with accepted method selections.
Only the problem-fingerprint weights used for ranking are varied, while
the accepted-method overlap used for evaluation remains fixed under
uniform weighting. The oracle ranks candidates directly by this overlap
and is therefore independent of the retrieval weight. Fingerprint depths
reach 5 in the Numerical and PIML trees, 4 in DPD, and 13 in Formal.}
\label{tab:weight_sensitivity}
\begin{tabular}{@{}llccc@{}}
\toprule
Repository & Level weight & Top-1 overlap & Coverage@3 & NDCG@3 \\
\midrule
Numerical (13 queries) & Oracle (ceiling)         & 0.852 & 0.955 & 1.000 \\
                       & $\omega(d)=1$ (deployed) & 0.742 & 0.940 & 0.916 \\
                       & $\omega(d)=1/(1+d)$      & 0.742 & 0.945 & 0.903 \\
                       & $\omega(d)=2^{-d}$       & 0.743 & 0.945 & 0.898 \\
                       & $\omega(d)=10^{-d}$      & 0.804 & 0.949 & 0.916 \\
\midrule
PIML (8 queries)       & Oracle (ceiling)         & 0.821 & 0.944 & 1.000 \\
                       & $\omega(d)=1$ (deployed) & 0.762 & 0.940 & 0.970 \\
                       & $\omega(d)=1/(1+d)$      & 0.762 & 0.939 & 0.968 \\
                       & $\omega(d)=2^{-d}$       & 0.762 & 0.936 & 0.963 \\
                       & $\omega(d)=10^{-d}$      & 0.802 & 0.940 & 0.984 \\
\midrule
DPD (19 queries)       & Oracle (ceiling)         & 0.950 & 0.988 & 1.000 \\
                       & $\omega(d)=1$ (deployed) & 0.924 & 0.987 & 0.968 \\
                       & $\omega(d)=1/(1+d)$      & 0.918 & 0.987 & 0.966 \\
                       & $\omega(d)=2^{-d}$       & 0.921 & 0.985 & 0.965 \\
                       & $\omega(d)=10^{-d}$      & 0.921 & 0.985 & 0.964 \\
\midrule
Formal (25 queries)    & Oracle (ceiling)         & 0.898 & 0.968 & 1.000 \\
                       & $\omega(d)=1$ (deployed) & 0.859 & 0.957 & 0.950 \\
                       & $\omega(d)=1/(1+d)$      & 0.858 & 0.955 & 0.950 \\
                       & $\omega(d)=2^{-d}$       & 0.852 & 0.955 & 0.953 \\
                       & $\omega(d)=10^{-d}$      & 0.854 & 0.954 & 0.949 \\
\bottomrule
\end{tabular}
\end{table}

\begin{table}[!htbp]
\centering
\caption{\small\linespread{1}\selectfont
\textbf{Models for every agent role in the GRAFT-guided workflows.}
Stage numbers refer to Supplementary Table~\ref{tab:graft_athena_family_stages}.}
\label{tab:model_identifiers}
\small
\setlength{\tabcolsep}{4pt}
\renewcommand{\arraystretch}{0.96}
\begin{tabularx}{\textwidth}{@{}>{\raggedright\arraybackslash}p{2.7cm}>{\raggedright\arraybackslash}p{2.3cm}>{\raggedright\arraybackslash}X@{}}
\toprule
Model & Access & Agent roles (stages) \\
\midrule
Claude Opus 4.8 & Claude Code (headless) & Formalization dialog, proposal ranker, and derivation auditor (1); mesh planning, construction, and gate (5--6); code planning and implementation (7--8); postprocessing (10); advisor (13); Formal plan critic, section-verification gate, and final auditor (7--9) \\
GPT 5.6 Sol & Codex (headless) & Analytical simplification proposer and derivation reviewer (1); postprocess planning (7); diagnosis (12); Formal reduction ranking, proof planner, proof sectioning, Lean verifier, render, and method extraction (4--13) \\
GPT 5.4 & OpenAI API & Encoding selector (2); method proposer (4) \\
Gemini 3.1 Pro Preview & Google API & Encoding critic (2); method critic (4); advisor-report parser (14) \\
Gemini 2.5 Pro & Google API & PIML training-log score parser (11) \\
\bottomrule
\end{tabularx}
\end{table}

\begin{table}[!htbp]
\centering
\caption{
\textbf{Family-specific realization of the GRAFT-guided stage spine.}
An omitted stage is a deliberate no-op: its function is either unnecessary for that family or absorbed by another stage.}
\label{tab:graft_athena_family_stages}
{\small\linespread{1}\selectfont
\setlength{\tabcolsep}{3pt}
\begin{tabularx}{\textwidth}{@{}>{\raggedright\arraybackslash}p{1.1cm}*{4}{>{\raggedright\arraybackslash}X}@{}}
\toprule
Stages & PIML & Numerical & DPD & Formal \\
\midrule
1--4 & Formalize and encode a PDE problem; retrieve three neighbors; refine a PIML method with a proposer--critic pair. & Formalize and encode a PDE problem; retrieve three neighbors; refine a numerical method with a proposer--critic pair. & Formalize and encode a particle-flow problem; retrieve three neighbors; refine a DPD method with a proposer--critic pair. & Formalize and encode a theorem; retrieve five proved neighbors; rank reduction and proof-transfer routes. \\
5--6 & Omitted. & Plan, construct, and validate the computational mesh. & Omitted; particle geometry is part of the simulation state. & Omitted. \\
7--10 & Plan the trainer and postprocessor; implement the Python/JAX state; train it; compute predictions, residuals, errors, and flatness. & Plan the solver and postprocessor; implement the numerical state; execute it under monitoring; compute fields, errors, and conservation diagnostics. & Plan the LAMMPS simulation and analysis; construct and test the input state; execute the particle system; extract thermo, rheological, and clinical observations. & Plan and section the proof; implement it in Lean; compile and audit it; render the accepted machine-checked proof. \\
11--13 & Score accuracy, diagnose training and physical consistency, and issue method- or code-level advice. & Score entropy stability, diagnose numerical behavior, and issue mesh- or code-level advice. & Score simulation health, diagnose physical and rheological behavior, and issue geometry- or code-level advice. & Stages 11--12 are omitted; Stage 13 extracts the audited proof method and reward. \\
14 & Store an accepted run or create a revised model iteration. & Store an accepted run or return to the mesh or solver. & Store an accepted run or return to the geometry or simulation state. & Store and fingerprint an accepted proof or return to proof planning or Lean implementation. \\
\bottomrule
\end{tabularx}
}
\end{table}

\begin{table}[!htbp]
\centering
\caption{\small\linespread{1}\selectfont
\textbf{Inventory of the complete coded GRAFT problem and action spaces.}
A branching node is a decision family with at least one child; depth is measured from the root at depth zero.}
\label{tab:graft_space_inventory}
\small
\setlength{\tabcolsep}{4.5pt}
\renewcommand{\arraystretch}{0.96}
\begin{tabular}{@{}p{0.35\textwidth}llrrrr@{}}
\toprule
Space & Role & Root & Nodes & Families & Leaves & Depth \\
\midrule
PDE problem space & Problem & \textnormal{(\texttt{PRB})} & 428 & 119 & 309 & 6 \\
DPD problem space & Problem & \textnormal{(\texttt{DP})} & 78 & 17 & 61 & 4 \\
Formal mathematical problem space & Problem & \textnormal{(\texttt{MP})} & 459 & 139 & 320 & 13 \\
Analytical-method action space & Action & \textnormal{(\texttt{AM})} & 125 & 45 & 80 & 4 \\
Numerical-method action space & Action & \textnormal{(\texttt{NMR})} & 693 & 204 & 489 & 8 \\
Physics-informed machine-learning action space & Action & \textnormal{(\texttt{PIML})} & 661 & 173 & 488 & 10 \\
DPD action space & Action & \textnormal{(\texttt{DPD})} & 182 & 49 & 133 & 4 \\
Formal-reasoning action space & Action & \textnormal{(\texttt{FRM})} & 148 & 43 & 105 & 5 \\
\midrule
\textbf{Total} & & & \textbf{2774} & \textbf{789} & \textbf{1985} & \\
\bottomrule
\end{tabular}
\end{table}

\begin{table}[ht]
\centering
\small
\begin{tabular}{rrrr}
\toprule
$N$ & Relative $L^2$ & $L^\infty$ & $L^1$ \\
\midrule
$16$  & $4.758\times10^{-2}$ & $1.385\times10^{-1}$ & $1.905\times10^{-2}$ \\
$32$  & $6.774\times10^{-3}$ & $2.926\times10^{-2}$ & $2.337\times10^{-3}$ \\
$64$  & $1.899\times10^{-4}$ & $1.089\times10^{-3}$ & $5.317\times10^{-5}$ \\
$128$ & $6.796\times10^{-7}$ & $2.633\times10^{-6}$ & $2.721\times10^{-7}$ \\
\bottomrule
\end{tabular}
\caption{\textbf{Spectral-PINN mode-count sweep.} Errors are evaluated on the same dense reference grid. The exponential fit in Fig.~\ref{fig:verified_discovery}F uses all four runs.}
\label{tab:spectral_mode_sweep}
\end{table}

\begin{table}[!htbp]
\centering
\caption{\textbf{Extreme-case queries: all retrieval arms compared on each arm's best and worst query.}
Best and worst queries are selected per arm by NDCG@3 over the 13
leave-one-out queries of Table~\ref{tab:topk_benchmark}; the six selections
coincide on three distinct queries, shown as blocks with the oracle ceiling
on that query. On the text baselines' best query (Allen--Cahn 1D), all arms
are equivalent at the ceiling. The worst cases separate them: BM25 collapses
on a query GRAFT retrieves at the ceiling, and on the shared hardest query
GRAFT's first pick misses while its top three still recover 97\% of the
accepted method nodes, whereas both text baselines recover 8\% against a
ceiling of 96\%. GRAFT's Coverage@3 can marginally exceed the oracle's
because the oracle ranks by per-candidate overlap, not by the union of the
top three.}
\label{tab:topk_extremes}
\begin{tabular}{@{}lccc@{}}
\toprule
Arm & Top-1 overlap & Coverage@3 & NDCG@3 \\
\midrule
\multicolumn{4}{@{}l}{\textit{K\'arm\'an cylinder wake: GRAFT best, BM25 worst}} \\
Oracle (ceiling)    & 0.950 & 1.000 & 1.000 \\
GRAFT fingerprints  & 0.950 & 1.000 & 1.000 \\
Embedding RAG       & 0.950 & 0.996 & 0.783 \\
BM25 lexical search & 0.051 & 0.068 & 0.039 \\
\midrule
\multicolumn{4}{@{}l}{\textit{Allen--Cahn 1D: embedding RAG and BM25 best}} \\
Oracle (ceiling)    & 0.842 & 0.914 & 1.000 \\
GRAFT fingerprints  & 0.842 & 0.914 & 0.975 \\
Embedding RAG       & 0.842 & 0.914 & 0.995 \\
BM25 lexical search & 0.842 & 0.914 & 0.985 \\
\midrule
\multicolumn{4}{@{}l}{\textit{Hasegawa--Wakatani turbulence: GRAFT and embedding RAG worst}} \\
Oracle (ceiling)    & 0.887 & 0.961 & 1.000 \\
GRAFT fingerprints  & 0.058 & 0.970 & 0.563 \\
Embedding RAG       & 0.055 & 0.077 & 0.055 \\
BM25 lexical search & 0.055 & 0.077 & 0.048 \\
\bottomrule
\end{tabular}
\end{table}

\begin{table}[!htbp]
\centering
\caption{\textbf{The five replicated Fable no-wake trajectories.} All runs select finite-volume MUSCL with minmod limiting and an HLLE flux. Rayleigh--Pitot and axial-coefficient errors and the entropy wall band use the definitions of Supplementary Table~\ref{tab:apollo_diagnostics}; the run-1 row is the trajectory audited above. Stand-off is each run's own reported value. IB denotes immersed boundary.}
\label{tab:apollo_replicates}
\setlength{\tabcolsep}{2pt}
{\small\linespread{1}\selectfont
\begin{tabularx}{\textwidth}{@{}>{\raggedright\arraybackslash}p{0.55cm}>{\raggedright\arraybackslash}p{2.0cm}>{\raggedright\arraybackslash}X>{\centering\arraybackslash}p{1.3cm}>{\centering\arraybackslash}p{1.05cm}>{\centering\arraybackslash}p{1.25cm}>{\centering\arraybackslash}p{1.4cm}@{}}
\toprule
Run & Mesh and size & Domain and geometry outcome & Rayleigh--Pitot error & $-C_A$ error & Reported $\Delta/R_b$ & Wall band below floor \\
\midrule
1 & Cartesian IB; 128,000 cells & Domain ends mid-cone; fixed-freestream boundary adjacent to disturbed flow; unphysical wall layer & $-0.55\%$ & $-4.40\%$ & $0.385$ & $35.7\%$ \\
2 & Body-fitted, extruded; 42,624 cells & Geometry retained; containment claimed but the shock crosses the outer boundary (benign) & $-0.25\%$ & $+2.20\%$ & $0.2985$ & $0.0\%$ \\
3 & Cartesian IB; 307,200 cells & Sting added, as in Opus 4.8; shock crosses the lateral boundary (benign) & $-0.25\%$ & $+2.46\%$ & $0.3004$ & $0.0\%$ \\
4 & Body-fitted, extruded; 26,544 cells & Geometry retained; containment verified & $-0.61\%$ & $+1.79\%$ & $0.299$ & $0.0\%$ \\
5 & Cartesian IB; 206,336 cells & Geometry retained; smoke test caught the clamped shock; domain widened & $-0.25\%$ & $+2.57\%$ & $0.30$ & $0.0\%$ \\
\bottomrule
\end{tabularx}
}
\end{table}

\begin{table}[!htbp]
\centering
\caption{\textbf{Validation and physical diagnostics for the five Apollo solutions.} Rayleigh--Pitot errors use the exact stagnation-pressure relation (Eq.~\eqref{eq:rayleigh_pitot}), with the GRAFT stagnation pressure reported as the common $t\in[4,5]$ window mean defined in Eq.~\eqref{eq:apollo_stagnation_pressure}, and axial-coefficient errors use the common AEDC reference $-C_A=1.44$ (Eq.~\eqref{eq:axial_coefficient_general}). Shock stand-off is reported without an error label because no common exact reference is available. The entropy fraction is measured in a $0.01$--$0.10R_b$ wall-normal band against the normal-shock floor $\Delta s=2.318$ (Eq.~\eqref{eq:entropy_floor}). The uncertainty column refers to the stand-off value.}
\label{tab:apollo_diagnostics}
\setlength{\tabcolsep}{2.5pt}
\begin{tabularx}{\textwidth}{@{}>{\raggedright\arraybackslash}p{1.65cm}>{\centering\arraybackslash}p{0.72cm}>{\centering\arraybackslash}p{1.4cm}>{\centering\arraybackslash}p{1.2cm}>{\centering\arraybackslash}p{1.3cm}>{\centering\arraybackslash}p{1.65cm}>{\raggedright\arraybackslash}X@{}}
\toprule
System & Wake & Rayleigh--Pitot error & $-C_A$ error & Reported $\Delta/R_b$ & Wall band below entropy floor & Uncertainty or replication \\
\midrule
Opus 4.8 & No & $-0.19\%$ & $+2.62\%$ & $0.3023$ & $0.0\%$ & Not stated \\
Fable 5 & No & $-0.55\%$ & $-4.40\%$ & $0.385$ & $35.7\%$ & Not stated \\
GRAFT & No & $+0.08\%$ & $+1.73\%$ & $0.3004$ & $0.0\%$ & Four-run cross-check; $\pm0.01$ \\
Fable 5 & Yes & $-0.90\%$ & $+3.06\%$ & $0.302$ & $0.0\%$ & Not stated \\
GRAFT & Yes & $+0.46\%$ & $+1.55\%$ & $0.3041$ & $0.0\%$ & Six-iteration history and tolerance \\
\bottomrule
\end{tabularx}
\end{table}

\begin{table}[!htbp]
\centering
\caption{\textbf{The 21 nodes added to the PIML action space for the
spectral-PINN formulation.} Each subtree attaches to the pre-existing branch
named in its header row. Nodes selected by the spectral-PINN method are
marked (sPINN).}
\label{tab:spectral_pinn_extension}
{\footnotesize\linespread{1}\selectfont
\setlength{\tabcolsep}{3pt}
\begin{tabularx}{\textwidth}{@{}>{\raggedright\arraybackslash}p{3.8cm}>{\raggedright\arraybackslash}X@{}}
\toprule
Added node & Role \\
\midrule
\multicolumn{2}{@{}>{\raggedright\arraybackslash}p{\textwidth}@{}}{\emph{Network Input Type}, a new decision under
\emph{Main Model}: which independent variables the network receives} \\
\quad Spatiotemporal Input & Network receives $(\mathbf{x}, t)$; the
standard PINN input. \\
\quad Spatial-only Input & Network receives spatial coordinates only. \\
\quad Time-only Input (sPINN) & Network receives $t$ alone; the spatial
dependence is carried entirely by the representation basis, and spatial
derivatives act analytically on the basis rather than by automatic
differentiation. \\
\quad Custom Input Type & Free slot for a user-defined input map. \\
\midrule
\multicolumn{2}{@{}>{\raggedright\arraybackslash}p{\textwidth}@{}}{\emph{Solution Representation}, a new decision under
\emph{Representation Model}: how the network output realizes the solution
field} \\
\quad Direct Field & The network output is the solution field itself; the
standard PINN representation. \\
\quad Operator Learning & The network parameterizes a solution operator
rather than a single field. \\
\quad Modal Expansion (sPINN) & The field is a truncated basis expansion
whose coefficients the network outputs. \\
\quad\quad Modal Mode Set & Which modes enter the expansion: \\
\quad\quad\quad Full Mode Set & All modes up to the truncation. \\
\quad\quad\quad Sine Only (sPINN) & Parity-odd (sine) subspace, admissible
when the initial condition is odd and the dynamics preserve parity. \\
\quad\quad\quad Cosine Only & Parity-even (cosine) subspace. \\
\quad\quad\quad Custom Mode Set & User-selected modes. \\
\hangindent=1.6em\hangafter=1 \quad Custom Solution Representation & Free
slot for a user-defined representation. \\
\midrule
\multicolumn{2}{@{}>{\raggedright\arraybackslash}p{\textwidth}@{}}{\emph{Galerkin form}, a new decision under
\emph{Equation Form}: a weighted-residual statement of the PDE loss} \\
\quad Galerkin Variant & Which projection defines the residual: \\
\quad\quad Pure Galerkin & Test space equals the trial basis. \\
\quad\quad Petrov--Galerkin & Test space differs from the trial basis. \\
\hangindent=2.1em\hangafter=1 \quad\quad Pseudospectral Galerkin (sPINN) &
Galerkin projection evaluated pseudospectrally: nonlinear terms on a
collocation grid, derivatives in mode space. \\
\quad\quad Custom Galerkin & Free slot for a user-defined weighted-residual
form. \\
\bottomrule
\end{tabularx}
}
\end{table}

\clearpage

\subsection{Supplementary Video 1. Recursive partition-of-unity projection of the PDE problem tree.}
The animation constructs the PDE problem space while mapping its hierarchy onto the partition plane. Mutually exclusive ``subdivides in'' choices recursively partition the $x$ direction, whereas jointly applicable ``characterized by'' attributes partition the $y$ direction; tree depth is retained as the third coordinate. The highlighted path follows one encoded problem through successive partitions. The animation illustrates how the recursive, nonoverlapping construction assigns distinct tree nodes distinct planar positions; formal injectivity is established in Proposition~\ref{prop:phi_injective}.

\end{document}